\documentclass{article}

\PassOptionsToPackage{numbers}{natbib}

\usepackage[final]{neurips_2021}

\usepackage[utf8]{inputenc} 
\usepackage[T1]{fontenc}    
\usepackage{hyperref}       
\usepackage{url}            
\usepackage{booktabs}       
\usepackage{amsfonts}       
\usepackage{nicefrac}       
\usepackage{microtype}      
\usepackage{xcolor}         
\usepackage{algorithm}
\usepackage{bbm}
\usepackage{algpseudocode}
\usepackage{comment}
\usepackage{xspace}
\usepackage{graphicx}
\usepackage{sidecap}
\usepackage{subcaption}
\usepackage{amsmath}
\usepackage{amsthm}
\usepackage[colorinlistoftodos]{todonotes}
\usepackage{enumitem}
\usepackage{adjustbox}
\usepackage{wrapfig}
\usepackage{tabularx}
\usepackage{amssymb}
\usepackage{thmtools, thm-restate}
\usepackage{pifont}

\newcommand{\cmark}{\ding{51}}%
\newcommand{\xmark}{\ding{55}}%

\newtheorem{defn}{Definition}
\numberwithin{defn}{section}
\newtheorem{lemma}[defn]{Lemma}
\newtheorem{theorem}[defn]{Theorem}

\newenvironment{subproof}[1][\proofname]{%
  \begin{proof}[#1]%
}{%
  \end{proof}%
}

\DeclareMathOperator*{\argmin}{arg\,min}
\DeclareMathOperator*{\sign}{sgn}

\newcommand{\smallparagraph}[1]{\noindent\textbf{#1}\quad}

\newcommand{\EE}{\mathbb{E}}
\newcommand{\PP}{\mathbb{P}}
\newcommand{\II}{\mathbf{1}}
\newcommand{\indep}{\perp \!\!\! \perp}

\newcommand{\Att}{A}
\newcommand{\att}{a}
\newcommand{\attspace}{\mathcal{A}}
\newcommand{\sample}{\mathcal{S}}
\newcommand{\hypset}{\mathbb{H}}
\newcommand{\hypcls}{\mathcal{H}}
\newcommand{\hyp}{h}
\newcommand{\Obs}{X}
\newcommand{\obs}{x}
\newcommand{\obsspace}{\mathcal{X}}
\newcommand{\Lbl}{Y}
\newcommand{\lbl}{y}
\newcommand{\Dec}{\hat{\Lbl}}
\newcommand{\dec}{\hat{\lbl}}
\newcommand{\dist}{\mathcal{D}}
\newcommand{\loss}{\ell}
\newcommand{\costmat}{C}
\newcommand{\cost}{c}
\newcommand{\risk}{\mathcal{R}}
\newcommand{\post}{q}
\newcommand{\opt}{\text{opt}}
\newcommand{\MD}{\text{MD}}

\makeatletter
\newcommand\thefontsize{The current font size is: \f@size pt}
\makeatother

\makeatletter
\def\th@plain{%
  \thm@notefont{}
  \itshape 
}
\def\th@definition{%
  \thm@notefont{}
  \normalfont 
}
\makeatother

\title{Uncertain Decisions Facilitate Better \\ Preference Learning}

\author{%
  Cassidy Laidlaw \\
  University of California, Berkeley \\
  \texttt{cassidy\_laidlaw@cs.berkeley.edu} \\
  \And
  Stuart Russell \\
  University of California, Berkeley \\
  \texttt{russell@cs.berkeley.edu} \\
}

\begin{document}

\maketitle

\begin{abstract}
Existing observational approaches for learning human preferences, such as inverse reinforcement learning, usually make strong assumptions about the observability of the human's environment.
However, in reality, people make many important decisions under uncertainty.
To better understand preference learning in these cases, we study the setting of inverse decision theory (IDT), a previously proposed framework where a human is observed making non-sequential binary decisions under uncertainty. In IDT, the human's preferences are conveyed through their loss function, which expresses a tradeoff between different types of mistakes.
We give the first statistical analysis of IDT, providing conditions necessary to identify these preferences and characterizing the sample complexity---the number of decisions that must be observed to learn the tradeoff the human is making to a desired precision.
Interestingly, we show that it is actually easier to identify preferences when the decision problem is more uncertain. Furthermore, uncertain decision problems allow us to relax the unrealistic assumption that the human is an optimal decision maker but still identify their exact preferences; we give sample complexities in this suboptimal case as well. Our analysis contradicts the intuition that partial observability should make preference learning more difficult. It also provides a first step towards understanding and improving preference learning methods for uncertain and suboptimal humans.

\end{abstract}

\section{Introduction}





The problem of inferring human preferences has been studied for decades in fields such as inverse reinforcement learning (IRL), preference elicitation, and active learning. However, there are still several shortcomings in existing methods for preference learning. Active learning methods require query access to a human; this is infeasible in many purely observational settings and may lead to inaccuracies due to the description-experience gap \citep{hertwig_descriptionexperience_2009}. IRL is an alternative preference learning tool which requires only observations of human behavior. However, IRL suffers from underspecification, i.e. preferences are not precisely identifiable from observed behavior \citep{ng_algorithms_2000}. Furthermore, nearly all IRL methods require that the observed human is optimal or noisily optimal at optimizing for their preferences. However, humans are often \emph{systematically} suboptimal decision makers \cite{mullainathan_machine_2019}, and accounting for this makes IRL even \emph{more} underspecified, since it is hard to tell suboptimal behavior for one set of preferences apart from optimal behavior for another set of preferences \citep{armstrong_occams_2018}.


IRL and preference learning from observational data are generally applied in situations where a human is acting under no uncertainty. Given the underspecification challenge, one might expect that adding in the possibility of uncertainty in decision making (known as partial observability) would only make preference learning more challenging.
Indeed, \citet{choi_inverse_2011} and \citet{chinaei_inverse_2012}, who worked to apply IRL to partially observable Markov decision processes (POMDPs, where agents act under uncertainty), remarked that the underspecification of IRL combined with the intractability of POMDPs made for a very difficult task.

\begin{figure}
    \centering
    \setlength{\extrarowheight}{6pt}
    \begin{tabular}{p{0.03\linewidth} p{0.44\linewidth} p{0.44\linewidth}}
        & \multicolumn{1}{c}{\bf Decisions without uncertainty} & \multicolumn{1}{c}{\bf Decisions under uncertainty} \\
        (a) & Should I quarantine a traveler with a 100\% accurate negative test for a dangerous disease? &
        Should I quarantine a traveler with some symptoms of a dangerous disease but no test results? \\
        (b) & Should a person with irrefutable evidence of and confession to a crime be convicted? &
        Should a person with circumstantial evidence of a crime be convicted? \\
    \end{tabular}
    \caption{One of our key findings is that \emph{decisions made under uncertainty can reveal more preferences than clear decisions.} Here we give examples of decisions made with and without uncertainty.
    (a) In the case without uncertainty, nobody would choose to quarantine the traveler, so we cannot distinguish between different people's preferences. However, in the case \emph{with} uncertainty, people might decide differently whether to quarantine the traveler depending on their preferences on the tradeoff between individual freedom and public health. This allows us to identify those preferences by observing decisions. (b) Similarly, observing decisions on whether to convict a person under uncertainty reveals preferences about the tradeoff between convicting innocent people and allowing criminals to go free.
    }
    \label{fig:decision_uncertainty}
\end{figure}

In this work, we find that, surprisingly, observing humans making decisions under uncertainty actually makes preference learning \emph{easier} (see Figure \ref{fig:decision_uncertainty}). To show this, 
we analyze a simple setting, where a human decision maker observes some information and must make a binary choice. This is somewhat analogous to supervised learning, where a decision rule is chosen to minimize some loss function
over a data distribution. In our formulation, the goal is to learn the human decision maker's loss function by observing their decisions. Often, in supervised learning, the loss function is simply the 0-1 loss. However, humans may incorporate many other factors into their implicit ``loss functions''; they may weight different types of mistakes unequally or incorporate fairness constraints, for instance. One might call this setting ``inverse supervised learning,'' but it is better described as inverse decision theory (IDT) \citep{swartz_inverse_2006,davies_inverse_2005}, since the objective is to reverse-engineer only the human's decision rule and not any learning process used to arrive at it. IDT can be shown to be a special case of partially observable IRL (see Appendix \ref{sec:pomdp}) but its restricted assumptions allow more analysis than would be possible for IRL in arbitrary POMDPs. However, we believe that the insights we gain from studying IDT should be applicable to POMDPs and uncertain decision making settings in general. We introduce a formal description of IDT in Section \ref{sec:problem}.

While we hope to provide insight into general reward learning, IDT is also a useful tool in its own right; even in this binary, non-sequential setting, human decisions can reveal important preferences. For example, during a deadly disease outbreak, a government might pass a law to quarantine individuals with a chance of being sick. The decision rule the government uses to choose who to quarantine depends on the relative costs of failing to quarantine a sick person versus accidentally quarantining an uninfected one. In this way, even human decisions where there is a ``right'' answer are revealing if they are made under uncertainty. This example could distinguish a preference for saving lives versus one for guaranteeing freedom of movement. These preferences on the tradeoff between costs of mistakes are expressed through the loss function that the decision maker optimizes.

In our main results on IDT in Section \ref{sec:sample_complexity}, we find that the identifiability of a human's loss function is dependent on whether the decision we observe them making involves uncertainty. If the human faces sufficient uncertainty, we give tight sample complexity bounds on the number of decisions we must observe to identify their loss function, and thus preferences, to any desired precision (Theorem \ref{thm:opt_dec}). On the other hand, if there is no uncertainty---i.e., the correct decision is always obvious---then we show that there is no way to identify the loss function (Theorem \ref{thm:optimal_lower} and Corollary \ref{corollary:certain_lower}).
Technically, we show that learning the loss function is equivalent to identifying a threshold function over the space of posterior probabilities for which decision is correct given an observation (Figure \ref{fig:posterior_viz}). This threshold can be determined to precision $\epsilon$ in $\Theta(1 / ( p_\cost \epsilon))$ samples, where $p_\cost$ is the probability density of posterior probabilities around the threshold. In the case where there is no uncertainty in the decision problem, $p_\cost = 0$ and we demonstrate that the loss function cannot be identified.

These results apply to optimal human decision makers---that is, those who completely minimize their expected loss. When a decision rule or policy is suboptimal, in general their loss function cannot be learned \citep{armstrong_occams_2018,shah_feasibility_2019}. However, we show that decisions made under uncertainty are also helpful in this case; under certain models of suboptimality, we can still \emph{exactly} recover the human's loss function.

We present two such models of suboptimality (see Figure \ref{fig:optimality_cases}). In both, we assume that the decision maker is restricting themselves to choosing a decision rule $\hyp$ in some hypothesis class $\hypcls$, which may not include the optimal decision rule. This framework is similar to that of agnostic supervised learning \citep{haussler_decision_1992,kearns_toward_1994}, but solves the inverse problem of determining the loss function given a hypothesis class and decision samples. If the restricted hypothesis class $\hypcls$ is known, we show that the loss function can be learned similarly to the optimal case (Theorem \ref{thm:subopt_known}). Our analysis makes a novel connection between Bayesian posterior probabilities and binary hypothesis classes.
However, assuming that $\hypcls$ is known is a strong assumption; for instance, we might suspect that a decision maker is ignoring some data features but we may not know exactly which features. We formalize this case by assuming that the decision maker could be considering the optimal decision rule in any of a number of hypothesis classes in some family $\hypset$. This case is more challenging because we may need to identify which hypothesis class the human is using in order to identify their loss function. We show that, assuming a smoothness condition on $\hypset$, we can still obtain the decision maker's loss function (Theorem \ref{thm:subopt_unknown}).

We conclude with a discussion of our results and their implications in Section \ref{sec:discussion}. We extend IDT to more complex loss functions that can depend on certain attributes of the data in addition to the chosen decision; we show that this extension can be used to test for the fairness of a decision rule under certain criteria which were previously difficult to measure. We also compare the implications of IDT for preference learning in uncertain versus clear decision problems. Our work shows that uncertainty is \emph{helpful} for preference learning and suggests how to exploit this fact.

\section{Related Work}

Our work builds upon that of \citet{davies_inverse_2005} and \citet{swartz_inverse_2006}, who first introduced inverse decision theory. They describe how to apply IDT to settings in which a doctor makes treatment decisions based on a few binary test outcomes, but provide no statistical analysis. In contrast, we explore when IDT can be expected to succeed in more general cases and how many observed decisions are necessary to infer the loss function. We also analyze cases where the decision maker is suboptimal for their loss function, which are not considered by Davies or Swartz et al.

Inverse reinforcement learning (IRL) \cite{ng_algorithms_2000,abbeel_apprenticeship_2004,ramachandran_bayesian_2007,ziebart_maximum_2008,fu_learning_2017}, also known as inverse optimal control, aims to infer the reward function for an agent acting in a Markov decision process (MDP). Our formulation of IDT can be considered as a special case of IRL in a partially observable MDP (POMDP) with two states and two actions (see Appendix \ref{sec:pomdp}). Some prior work explored IRL in POMDPs \citep{choi_inverse_2011,chinaei_inverse_2012} by reducing the POMDP to a belief-state MDP and applying standard IRL algorithms.
Our main purpose is not to present improvements to IRL algorithms; rather, we give an analysis of the difference between observable and partially observable settings for preference learning. We begin with the restricted setting of IDT but hope to extend to sequential decision making in the future. We also consider cases where the human decision maker is suboptimal, which previous work did not explore.


Performance metric elicitation (ME) aims to learn a loss function (aka performance metric) by querying a human \citep{hiranandani_performance_2019, hiranandani_multiclass_2019, hiranandani_fair_2020}. ME and other active learning approaches \citep{biyik_batch_2018,mindermann_active_2019,biyik_asking_2019,bhatia_agnostic_2021} require the ability to actively ask a user for their preference among different loss or reward functions. 
In contrast, IDT aims to learn the loss function purely by observing a decision maker. Active learning is valuable for some applications, but there are many cases where it is infeasible. Observed decisions are often easier to obtain than expert feedback. Also, active learning may suffer from the description-experience gap \citep{hertwig_descriptionexperience_2009}; that is, it may be difficult to evaluate in the abstract the comparisons that these methods give as queries to the user, leading to biased results. In contrast, observing human decision making ``in the wild'' with IDT could lead to a more accurate understanding of human preferences.

Preference and risk elicitation aim to identify people's preferences between different uncertain or certain choices. A common tool is to ask a person to choose between a lottery (i.e., uncertain payoff) and a guaranteed payoff, or between two lotteries, varying parameters and observing the resulting choices \citep{cohen_experimental_1987,holt_risk_2002,csermely_how_2016}. In our analysis of IDT, decision making under uncertainty can be cast as a natural series of choices between lotteries. If we observe enough different lotteries, the decision maker's preferences can be identified. On the other hand, if there is no uncertainty, then we only observe choices between guaranteed payoffs and there is little information to characterize preferences.


\section{Problem Formulation}
\label{sec:problem}

\begin{figure}
    \centering
    \input{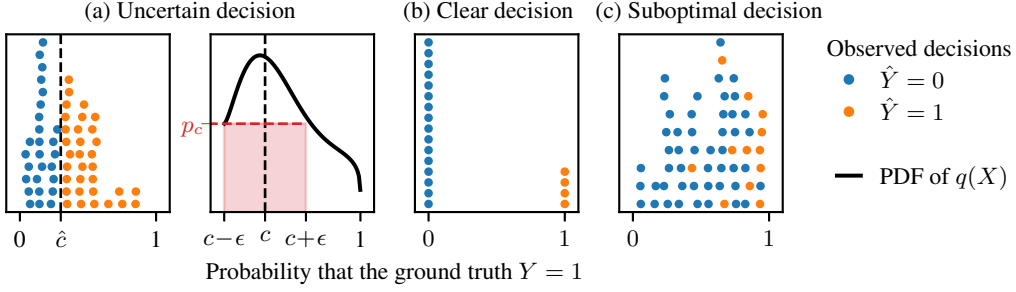}
    \caption{A visualization of three settings for inverse decision theory (IDT), which aims to estimate $\cost$, the parameter of a decision maker's loss function, given observed decisions $\dec_1, \hdots, \dec_m \in \{0, 1\}$. Here, each decision $\dec_i$ is plotted against the probability $\post(\obs_i) = \PP(\Lbl = 1 \mid \Obs = \obs_i)$ that the ground truth (correct) decision $\Lbl$ is $1$ given the decision maker's observation $\obs_i$. Lemma \ref{lemma:bayesopt} shows that an optimal decision rule assigns $\dec_i = \II\{\post(\obs_i) \geq c\}$. (a) For uncertain decision problems, IDT can estimate $\cost$ as the threshold of posterior probabilities $\post(\obs_i)$ where the decision switches from 0 to 1 (Section \ref{sec:optimal}). If the distribution of $\post(\Obs)$ has probability density at least $p_\cost$ on $[c - \epsilon, c + \epsilon]$, Theorem \ref{thm:opt_dec} shows we can learn $\cost$ to precision $\epsilon$ with $m \geq O(1/(p_\cost \epsilon))$ samples. (b) When there is no uncertainty in the decision problem, IDT cannot characterize the loss parameter $\cost$ because the threshold between positive and negative decisions could be anywhere between 0 and 1 (Section \ref{sec:lower_bounds}). (c) A suboptimal human decision maker does not use an optimal decision rule for any loss parameter $\cost$, but we can often still estimate their preferences (Sections \ref{sec:suboptimal_known} and \ref{sec:suboptimal_unknown}).}
    \label{fig:posterior_viz}
\end{figure}

We formalize inverse decision theory using decision theory and statistical learning theory. Let $\dist$ be a distribution over observations $\Obs \in \obsspace$ and ground truth decisions $\Lbl \in \{0, 1\}$. We consider an agent that receives an observation $\Obs$ and must make a binary decision $\Dec \in \{0, 1\}$. While many decision problems include more than two choices, we consider the binary case to simplify analysis. However, the results are applicable to decisions with larger numbers of choices; assuming irrelevance from independent alternatives (i.e. the independence axiom \citep{luce_choice_1977}), a decision among many choices can be reduced to binary choices between pairs of them. We generally assume that $\dist$ is fixed and known to both the decision maker and the IDT algorithm. Unless otherwise stated, all expectations and probabilities on $\Obs$ and $\Lbl$ are with respect to the distribution $\dist$.

We furthermore assume that the agent has chosen a decision rule (or hypothesis) $\hyp: \obsspace \to \{0, 1\}$ from some hypothesis class $\hypcls$ that minimizes a loss function which depends only on the decision $\Dec = \hyp(\Obs)$ that was made and the correct decision $\Lbl$:
\begin{equation*}
    \hyp \in \argmin_{\hyp \in \hypcls} \; \mathbb{E}_{(\Obs, \Lbl) \sim \dist} \left[ \loss(\hyp(\Obs), \Lbl) \right].
\end{equation*}
In general, the loss function $\ell$ might depend on the observation $\Obs$ as well; we explore this extension in the context of fair decision making in Section \ref{sec:fairness}. Assuming the formulation above, since $\Lbl, \Dec \in \{0, 1\}$ we can write the loss function $\ell$ as a matrix $\costmat \in \mathbb{R}^{2 \times 2}$ such that $\loss(\dec, \lbl) = \costmat_{\dec \lbl}$. We denote by $\risk_\costmat(\hyp) = \mathbb{E}_{(\Obs, \Lbl) \sim \dist} \left[ \loss(\hyp(\Obs), \Lbl) \right]$ the expected loss or ``risk'' of the hypothesis $\hyp$ with cost matrix $\costmat$. This cost matrix has four entries, but the following lemma shows that it effectively has only one degree of freedom.

\begin{restatable}[Equivalence of cost matrices]{lemma}{lemmacostmat}
\label{lemma:costmat}
Any cost matrix $\costmat = (\begin{smallmatrix} \costmat_{0 0} & \costmat_{0 1} \\ \costmat_{1 0} & \costmat_{1 1} \end{smallmatrix})$ is equivalent to a cost matrix $\costmat' = (\begin{smallmatrix} 0 & 1 - \cost \\ \cost & 0 \end{smallmatrix})$ where $\cost = \frac{\costmat_{1 0} - \costmat_{0 0}}{\costmat_{1 0} + \costmat_{0 1} - \costmat_{0 0} - \costmat_{1 1}}$ as long as $\costmat_{1 0} + \costmat_{0 1} - \costmat_{0 0} - \costmat_{1 1} \neq 0$. That is, there are constants $a, b \in \mathbb{R}$ such that $\risk_\costmat(\hyp) = a \risk_{\costmat'}(\hyp) + b$ for all $\hyp$.
\end{restatable}

See Appendix \ref{proof:costmat} for this and other proofs. Based on Lemma \ref{lemma:costmat}, from now on, we assume the cost matrix only has one parameter $\cost$, which is the cost of a false positive; $1 - \cost$ is the cost of a false negative. Intuitively, high values of $\cost$ indicate a preference for erring towards the decision $\Dec = 0$ under uncertainty while low values indicate a preference for erring towards the decision $\Dec = 1$. Finally, we assume that making the correct decision is always better than making an incorrect decision, i.e. $\costmat_{0 0} < \costmat_{1 0}$ and $\costmat_{1 1} < \costmat_{0 1}$. This implies that $0 < \cost < 1$.

We write $\loss_c$ and $\risk_c$ to denote the loss and risk functions using this loss parameter $\cost$. Thus, we can formally define a binary decision problem:

\begin{defn}[Decision problem]
\label{defn:decision_problem}
A (binary) \emph{decision problem} is a pair $(\dist, \cost)$, where $\dist$ is a distribution over pairs of observations and correct decisions $(\Obs, \Lbl) \in \obsspace \times \{0, 1\}$ and $\cost \in (0, 1)$ is the loss parameter. The decision maker aims to choose a decision rule $\hyp: \obsspace \to \{0, 1\}$ that minimizes the risk $\risk_\cost(\hyp) = \EE_{(\Obs, \Lbl) \sim \dist}[\loss_\cost(\hyp(\Obs), \Lbl)]$.
\end{defn}

As a running example, we consider the decision problem where an emergency room (ER) doctor needs to decide whether to treat a patient for a heart attack. In this case, the observation $\Obs$ might consist of the patient's medical records and test results; the correct decision is $\Lbl = 1$ if the patient is having a heart attack and $\Lbl = 0$ otherwise; and the made decision is $\Dec = 1$ if the doctor treats the patient and $\Dec = 0$ if not. In this case, a higher value of $\cost$ indicates that the doctor places higher cost on accidentally treating a patient not having a heart attack, while a lower value of $\cost$ indicates the doctor places higher cost on accidentally failing to treat a patient with a heart attack.

In \textbf{inverse decision theory (IDT)}, our goal is to determine the loss function the agent is optimizing, which here is equivalent to the parameter $\cost$. We assume access to the true distribution $\dist$ of observations and labels and also a finite sample of observations and decisions $\sample = \{(\obs_1, \dec_1), \hdots, (\obs_m, \dec_m)\}$ where $\obs_i \sim \dist$ i.i.d. and the decisions are made according to the decision rule, i.e. $\dec_i = \hyp(\obs_i)$.

Some of our main results concern the effects on IDT of whether or not a decision is made under uncertainty. We now formally characterize such decision problems.

\begin{defn}[Decision problems with and without uncertainty]
\label{defn:uncertainty}
A decision problem $(\dist, \cost)$ has \emph{no uncertainty} if $\PP_{(\Obs, \Lbl) \sim \dist}(\Lbl = 1 \mid \Obs) \in \{0, 1\}$ almost surely. The decision problem has uncertainty otherwise.
\end{defn}

That is, if it is always the case that, after observing and conditioning on $\Obs$, either $\Lbl = 1$ with 100\% probability or $\Lbl = 0$ with 100\% probability, then the decision problem has no uncertainty.

\section{Identifiability and Sample Complexity}
\label{sec:sample_complexity}

\begin{figure}
    \centering
    \input{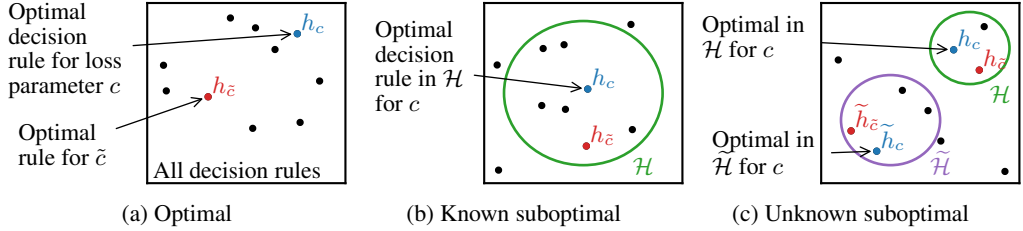}
    \caption{We analyze IDT for optimal decision makers and two cases of suboptimal decision makers. (a) In the optimal case (Section \ref{sec:optimal}), the decision maker chooses the optimal decision rule $\hyp$ for their loss parameter $\cost$ from all possible rules. (b) In the known suboptimal case (Section \ref{sec:suboptimal_known}), the decision maker chooses from a restricted hypothesis class $\hypcls$ which may not contain the overall best decision rule. (c) In the unknown suboptimal case (Section \ref{sec:suboptimal_unknown}), the decision maker chooses any of several hypothesis classes $\hypcls \in \hypset$ and then uses the optimal rule within that class, which may not be the optimal rule amongst all classes. This case is more difficult than (b) because we often need to identify the hypothesis class $\hypcls$ in addition to the loss parameter $\cost$.}
    \label{fig:optimality_cases}
    \vspace{-6pt}
\end{figure}

We aim to answer two questions about IDT. First, under what assumptions is the loss function identifiable? Second, if the loss function is identifiable, how large must the sample $\sample$ be to estimate $\cost$ to some precision with high probability? We adopt a framework similar to that of probably approximately correct (PAC) learning \citep{valiant_theory_1984}, and aim to calculate a $\hat\cost$ such that with probability at least $1 - \delta$ with respect to the sample of observed decisions, $| \hat\cost - \cost | \leq \epsilon$. While PAC learning typically focuses on test or prediction error, we instead focus on the estimation error for $\cost$. This has multiple advantages. First, it allows for better understanding and prediction of human behavior across distribution shift or in unseen environments \cite{gleave_quantifying_2021}. Second, there are cases where we care about the precise tradeoff the decision maker is optimizing for; for instance, in the ER doctor example, there are guidelines on the tradeoff between different types of treatment errors and we may want to determine if doctors' behavior aligns with these guidelines \cite{mullainathan_machine_2019}. Third, if the decision maker is suboptimal for their loss function (explored in Sections \ref{sec:suboptimal_known} and \ref{sec:suboptimal_unknown}), we may not want to simply replicate the suboptimal decisions, but find a better decision rule according to the loss function.

We consider three settings where we would like to estimate $\cost$, illustrated in Figure \ref{fig:optimality_cases}. First, we assume that the decision maker is perfectly optimal for their loss function. This is similar to the framework of \citet{swartz_inverse_2006}. However, moving beyond their analysis, we present properties necessary for identifiability and sample complexity rates. Second, we relax the assumption that the decision maker is optimal, and instead assume that they only consider a restricted set of hypotheses $\hypcls$ which is known to us. Finally, we remove the assumption that we know the hypothesis class that the decision maker is considering. Instead, we consider a family of hypothesis classes; the decision maker could choose the optimal decision rule \emph{within any class}, which is not necessarily the optimal decision rule \emph{across all classes}.

\subsection{Optimal decision maker}
\label{sec:optimal}

First, we assume that the decision maker is optimal. In this case, the form of the optimal decision rule is simply the Bayes classifier \citep{devroye_probabilistic_2013}.

\begin{restatable}[Bayes optimal decision rule]{lemma}{lemmabayesopt}
\label{lemma:bayesopt}
An optimal decision rule $\hyp$ for a decision problem $(\dist, \cost)$ is given by $\hyp(\obs) = \II \{ \post(\obs) \geq \cost \}$ where $\post(\obs) = \PP_{(\Obs, \Lbl) \sim \dist}(\Lbl = 1 \mid \Obs = \obs)$ is the posterior probability of class 1 given the observation $\obs$.
\end{restatable}

That is, any optimal decision rule corresponds to a threshold function on the posterior probability $\post(\obs)$, where the threshold is at the loss parameter $\cost$. Thus, the strategy for estimating $\cost$ from a sample of observations and decisions is simple. For each observation $\obs_i$, we calculate $\post(\obs_i)$. Then, we choose any $\hat\cost$ such that $\post(\obs_i) \geq \hat\cost \Leftrightarrow \dec_i = 1$; that is, $\hat\cost$ is consistent with the observed data. From statistical learning theory, we know that a threshold function can be PAC learned in $O(\log (1 / \delta) / \epsilon)$ samples. However, such learning only guarantees low \textit{prediction error} of the learned hypothesis. We need stronger conditions to ensure that $\hat\cost$ is close to the true loss function parameter $\cost$. The following theorem states conditions which allow estimation of $\cost$ to arbitrary precision.

\begin{restatable}[IDT for optimal decision maker]{theorem}{thmoptdec}
\label{thm:opt_dec}
Let $\epsilon > 0$ and $\delta > 0$. Say that there exists $p_\cost > 0$ such that $\PP(\post(\Obs) \in (\cost, \cost + \epsilon]) \geq p_\cost \epsilon$ and $\PP(\post(\Obs) \in [\cost - \epsilon, \cost)) \geq p_\cost \epsilon$.
Let $\hat\cost$ be chosen to be consistent with the observed decisions as stated above, i.e. $\post(\obs_i) \geq \hat\cost \Leftrightarrow \dec_i = 1$. Then
$| \hat\cost - \cost | \leq \epsilon$ with probability at least $1 - \delta$ as long as the number of samples $m \geq \frac{\log(2 / \delta)}{p_\cost \epsilon}$.
\end{restatable}

The parameter $p_\cost$ can be interpreted as the approximate probability density of $\post(\Obs)$ around the threshold $\cost$. For instance, the requirements of Theorem \ref{thm:opt_dec} are satisfied if the random variable $\post(\Obs)$ has a probability density of at least $p_\cost$ on the interval $[\cost - \rho, \cost + \rho]$ for some $\rho \geq \epsilon$; the requirements of Theorem \ref{thm:opt_dec} are more general to allow for cases when $\post(\Obs)$ does not have a density. The lower the density $p_\cost$, and thus the probability of observing decisions close to the threshold $\cost$, the more difficult inference becomes. Because of this, Theorem \ref{thm:opt_dec} \emph{requires} that the decision problem has uncertainty. If the decision problem has no uncertainty according to Definition \ref{defn:uncertainty}, then $\post(\Obs) \in \{0, 1\}$ always, i.e. the distribution of posterior probabilities has mass only at 0 and 1. In this case, $p_\cost = 0$ for small enough $\epsilon$ and Theorem \ref{thm:opt_dec} cannot be applied.
In fact, as we show in Section \ref{sec:lower_bounds}, it is impossible to tell what the true loss parameter $\cost$ when the decision problem lacks uncertainty. Figure \ref{fig:posterior_viz}(a-b) illustrates these results.


\subsection{Suboptimal decision maker with known hypothesis class}
\label{sec:suboptimal_known}

Next, we consider cases where the decision maker may not be optimal with respect to their loss function. Our model of suboptimality is that the agent only considers decision rules within some hypothesis class $\hypcls$, which may not include the optimal decision rule. This formulation is similar to that of agnostic PAC learning \citep{haussler_decision_1992,kearns_toward_1994}. It can also be considered a case of a restricted ``choice set'' as defined in the preference learning literature \citep{jeon_reward-rational_2020,freedman_choice_2021}. It can encompass many types of irrationality or suboptimality. For instance, one could assume that the decision maker is ignoring some of the features in $\obs$; then $\hypcls$ would consist of only decision rules depending on the remaining features. In the ER doctor example, we might assume that $\hypcls$ consists of decision rules using only the patient's blood pressure and heart rate; this models a suboptimal doctor who is unable to use more data to make a treatment decision.

While there are many possible models of suboptimality, this one has distinct advantages for preference learning with IDT. One alternative model is that the decision maker has small excess risk, i.e. $\risk_\cost(\hyp) \leq \risk_\cost(\hyp^*) + \Delta$ for some small $\Delta$ where $\hyp^*$ is the optimal decision rule. However, this definition precludes identifiability even in the infinite sample limit (see Appendix \ref{sec:alternative_subopt}).
Another form of suboptimality could be that the decision maker chooses a decision rule to minimize a surrogate loss rather than the true loss. However, we show in Appendix \ref{sec:surrogate} that for reasonable surrogate losses this is no different from minimizing the true loss.
A final alternative model of suboptimality is that the human is noisily optimal; this assumption underlies models like Boltzmann rationality or the Shephard-Luce choice rule \citep{shepard_stimulus_1957,luce_choice_1977,baker_goal_2007,ziebart_maximum_2008}. However, these models assume stochastic decision making and also cannot handle \emph{systematically} suboptimal humans.

In this section we begin by assuming that the restricted hypothesis class $\hypcls$ is known; this requires some novel analysis but the resulting identifiability conditions and sample complexity are very similar to the optimal case in Section \ref{sec:optimal}. In the next section, we consider cases where we are unsure about which restricted hypothesis class the decision maker is considering.

\begin{defn}
\label{defn:monotone_hypcls}
A hypothesis class $\hypcls$ is \emph{monotone} if for any $\hyp, \hyp' \in \hypcls$, either $\hyp(\obs) \geq \hyp'(\obs) \; \forall \obs \in \obsspace$ or $\hyp(\obs) \leq \hyp'(\obs) \; \forall \obs \in \obsspace$.
\end{defn}
\begin{defn}
The \emph{optimal subset} of a hypothesis class $\hypcls$ for a distribution $\dist$ is defined as
$$\opt_\dist(\hypcls) = \left\{\hyp \in \hypcls \;\middle|\; \exists \cost \; \text{ such that } \; \hyp \in {\textstyle \argmin_{\hyp \in \hypcls} \risk_\cost(\hyp)} \right\}$$
\end{defn}

In this section, we consider hypothesis classes whose optimal subsets are monotone. That is, changing the parameter $\cost$ has to either flip the optimal decision rule's output for some observations from 0 to 1, or flip some decisions from 1 to 0. It cannot both change some decisions from 0 to 1 and some from 1 to 0. This assumption is mainly technical; many interesting hypothesis classes naturally have monotone optimal subsets. Any hypothesis class formed by thresholding a function is monotone, i.e $\hypcls = \{\hyp(\obs) = \II\{f(\obs) \geq b\} \mid b \in \mathbb{R} \}$. Also, the set of decision rules based on a particular subset of the observed features satisfies this criterion, since optimal decision rules in this set are thresholds on the posterior probability that $\Lbl = 1$ given the subset of features.

For hypothesis classes with monotone optimal subsets, we can prove properties that allow for similar analysis to that we introduced in Section \ref{sec:optimal}. Let $\hyp_\cost$ denote a decision rule which is optimal for loss parameter $\cost$ in hypothesis class $\hypcls$. That is, $\hyp_\cost \in \argmin_{\hyp \in \hypcls} \risk_\cost(\hyp)$. A key lemma allows us to define a value similar to the posterior probability we used for analyzing the optimal decision maker.

\begin{restatable}[Induced posterior probability]{lemma}{lemmainducedpostprob}
\label{lemma:induced_postprob}
Let $\opt_\dist(\hypcls)$ be monotone and define
\begin{equation*}
    \overline{\post}_\hypcls(\obs) \triangleq \sup \Big( \{ \cost \in [0, 1] \mid \hyp_\cost(\obs) = 1 \} \cup \{0\} \Big)
    \quad \text{and} \quad
    \underline{\post}_\hypcls(\obs) \triangleq \inf \Big( \{ \cost \in [0, 1] \mid \hyp_\cost(\obs) = 0 \} \cup \{1\} \Big).
\end{equation*}
Then for all $\obs \in \obsspace$, $\overline{\post}_\hypcls(\obs) = \underline{\post}_\hypcls(\obs)$. Define the \emph{induced posterior probability} of $\hypcls$ as $\post_\hypcls(\obs) \triangleq \overline{\post}_\hypcls(\obs) = \underline{\post}_\hypcls(\obs)$.
\end{restatable}
\begin{restatable}{corollary}{corollarybayesoptsuboptknown}
\label{corollary:bayesopt_subopt_known}
Let $\hyp_\cost$ be any optimal decision rule in $\hypcls$ for loss parameter $\cost$. Then for any $\obs \in \obsspace$, $\hyp_\cost(\obs) = 1$ if $\post_\hypcls(\obs) > \cost$ and $\hyp_\cost(\obs) = 0$ if $\post_\hypcls(\obs) < \cost$.
\end{restatable}

Using Lemma \ref{lemma:induced_postprob}, the problem of IDT again reduces to learning a threshold; this time, any optimal classifier in $\hypcls$ is a threshold function on the \emph{induced} posterior probability $\post_\hypcls(\Obs)$, as shown in Corollary \ref{corollary:bayesopt_subopt_known}. Thus, to estimate $\hat\cost$, we calculate an induced posterior probability $\post_\hypcls(\obs_i)$ for each observation $\obs_i$ and choose any estimate $\hat\cost$ such that $\post_\hypcls(\obs_i) \geq \hat\cost \Leftrightarrow \dec_i = 1$. This allows us to state a theorem equivalent to Theorem \ref{thm:opt_dec} for the suboptimal case.

\begin{restatable}[Known suboptimal decision maker]{theorem}{theoremsuboptknown}
\label{thm:subopt_known}
Let $\epsilon > 0$ and $\delta > 0$, and let $\opt_\dist(\hypcls)$ be monotone. Say that there exists $p_\cost > 0$ such that $\PP(\post_\hypcls(\Obs) \in (\cost, \cost + \epsilon]) \geq p_\cost \epsilon$ and $\PP(\post_\hypcls(\Obs) \in [\cost - \epsilon, \cost)) \geq p_\cost \epsilon$.
Let $\hat\cost$ be chosen to be consistent with the observed decisions, i.e. $\post_\hypcls(\obs_i) \geq \hat\cost \Leftrightarrow \dec_i = 1$. Then
$| \hat\cost - \cost | \leq \epsilon$ with probability at least $1 - \delta$ as long as the number of samples $m \geq \frac{\log(2 / \delta)}{p_\cost \epsilon}$.
\end{restatable}


\subsection{Suboptimal decision maker with unknown hypothesis class}
\label{sec:suboptimal_unknown}

We now analyze the case when the decision maker is suboptimal but we are not sure in what manner. We model this by considering a family of hypothesis classes $\hypset$. We assume that the decision maker considers one of these hypothesis classes $\hypcls \in \hypset$ and then chooses a rule $\hyp \in \argmin_{\hyp \in \hypcls} \mathcal{R}_\cost(\hyp)$. This case is more challenging because we may need to identify $\hypcls$ to identify $\cost$.

One natural family $\hypset$ consists of hypothesis classes which depend only on some subset of the features:
\begin{equation}
    \label{eq:feat_subset_fam}
    \hypset_\text{feat} \triangleq \left\{ \hypcls_S \mid S \subseteq \{1, \hdots, n\} \right\} \quad \text{where} \quad \hypcls_S \triangleq \left\{ \hyp(\obs) = f(\obs_{S}) \mid f: \mathbb{R}^{|S|} \to \{0, 1\} \right\}
\end{equation}
where $\obs_S$ denotes only the coordinates of $\obs$ which are in the set $S$. This models a situation where we believe the decision maker may be ignoring some features, but we are not sure which features are being ignored.
Another possibility for $\hypset$ is thresholded linear combinations of the features in $\obs$, i.e.
\begin{equation*}
    \hypset_\text{linear} \triangleq \left\{ \hypcls_w \mid w \in \mathbb{R}^n \right\}
    \quad \text{where} \quad
    \hypcls_w \triangleq \left\{ \hyp(\obs) = \II\{ w^\top \obs \geq b \} \mid b \in \mathbb{R} \right\}.
\end{equation*}
In this case, we assume that the decision maker chooses some weights $w$ for the features arbitrarily but then thresholds the combination optimally. This could model the decision maker under- or over-weighting certain features, or also ignoring some (if $w_j = 0$ for some $j$).

In the high pressure and hectic environment of the ER example, we might assume that the doctor is using only a few pieces of data to decide whether to treat a patient. Here, $\hypset_\text{feat}$ would consist of a hypothesis class with decision rules that depend only on blood pressure and heart rate, a hypothesis class with decision rules that rely on these and also on an ECG, and so on. The difficulty of this setting compared to that of Section \ref{sec:suboptimal_known} is that the doctor could be using an optimal decision rule within any of these hypothesis classes. Thus, we may need to identify what data the doctor is using in their decision rule in order to identify their loss parameter $\cost$.

Estimating the loss parameter $\cost$ in the unknown hypothesis class case requires an additional assumption on the family of hypothesis classes $\hypset$, in addition to the monotonicity assumption from Section \ref{sec:suboptimal_known}.

\begin{defn}
\label{def:min_disagree}
Consider a family of hypothesis classes $\hypset$. Let $\hyp \in \hypcls \in \hypset$ and $\tilde{\hypcls} \in \hypset$. Then the \emph{minimum disagreement} between $\hyp$ and $\tilde{\hypcls}$ is defined as
$\MD(\hyp, \tilde{\hypcls}) \triangleq \inf_{\tilde\hyp \in \tilde{\hypcls}} \PP\big(\tilde\hyp(\Obs) \neq \hyp(\Obs)\big)$.
\end{defn}
\begin{defn}
A family of hypothesis classes $\hypset$ and hypothesis $\hyp_\cost \in \hypcls \in \hypset$ such that $\hyp_\cost \in \argmin_{\hyp \in \hypcls} \mathcal{R}_\cost(\hyp)$ is \emph{$\alpha$-MD-smooth} if $\opt_\dist(\tilde\hypcls)$ is monotone for every $\tilde\hypcls \in \hypset$ and
\begin{equation*}
    \forall \tilde{\hypcls} \in \hypset \quad \forall \cost' \in (0, 1) \qquad \MD(\hyp_{\cost'}, \opt_\dist(\tilde{\hypcls})) \leq (1 + \alpha |\cost' - \cost|) \MD(\hyp_\cost, \opt_\dist(\tilde{\hypcls})).
\end{equation*}
\end{defn}

While MD-smoothness is not particularly intuitive at first, it is necessary in some cases to ensure identifiability of the loss parameter $\cost$.
We present a case in Appendix \ref{sec:md_smooth_counterexample} where a lack of MD-smoothness precludes identifiability.



\begin{restatable}[Unknown suboptimal decision maker]{theorem}{thmsuboptunknown}
\label{thm:subopt_unknown}
Let $\epsilon > 0$ and $\delta$ > 0. Suppose we observe decisions from a decision rule $\hyp_\cost$ which is optimal for loss parameter $\cost$ in hypothesis class $\hypcls \in \hypset$. Let $\hyp_\cost$ and $\hypset$ be $\alpha$-MD-smooth. Furthermore, assume that there exists $p_\cost > 0$ such that for any $\rho \leq \epsilon$, $\PP(\post_\hypcls(\Obs) \in (\cost, \cost + \rho)) \geq p_\cost \rho$ and $\PP(\post_\hypcls(\Obs) \in (\cost - \rho, \cost)) \geq p_\cost \rho$. Let $d \geq \text{VCdim}\left(\cup_{\hypcls \in \hypset} \hypcls \right)$ be an upper bound on the VC-dimension of the union of all the hypothesis classes in $\hypset$.

Let $\hat\hyp_{\hat\cost} \in \argmin_{\hat\hyp \in \hat{\hypcls}} \risk_{\hat\cost}(\hat\hyp)$ be chosen to be consistent with the observed decisions, i.e. $\hat\hyp_{\hat\cost}(x_i) = \dec_i$ for $i = 1, \hdots, m$. Then $| \hat\cost - \cost | \leq \epsilon$ with probability at least $1 - \delta$ as long as the number of samples $m \geq \tilde{O}\left[ \left(\frac{\alpha}{\epsilon} + \frac{1}{\epsilon^2} \right) \left(\frac{d + \log(1 / \delta)}{p_\cost}\right) \right]$.
\end{restatable}

Theorem \ref{thm:subopt_unknown} requires more decision samples to guarantee low estimation error $|\hat\cost - \cost|$. Unlike Theorems \ref{thm:opt_dec} and \ref{thm:subopt_known}, the number of samples needed grow with the square of the desired precision $1/\epsilon^2$. There is also a dependence on the VC-dimension of the hypothesis classes $\hypcls \in \hypset$, since we are not sure which one the decision maker is considering.

Since our results in this section are highly general, it may be difficult to see how they apply to concrete cases. In Appendix \ref{sec:feat_hypset}, we explore the specific case of IDT in the unknown hypothesis class setting for $\hypset_\text{feat}$ as defined in (\ref{eq:feat_subset_fam}). We give sufficient conditions for MD-smoothness to hold and show that the sample complexity grows only logaramithically with $n$, the dimension of the observation space $\obsspace$, if the decision maker is relying on a sparse set of features.

\subsection{Lower bounds}
\label{sec:lower_bounds}

Is there any algorithm which can always determine the loss parameter $\cost$ to precision $\epsilon$ with high probability using fewer samples than required by Theorems \ref{thm:opt_dec} and \ref{thm:subopt_known}? We show that the answer is no: our previously given sample complexity rates are minimax optimal up to constant factors. We formalize this by considering any generic IDT algorithm, which we represent as a function $\hat\cost: (\obsspace \times \{0, 1\})^m \to (0, 1)$. The algorithm maps the sample of observations and decisions $\sample$ to an estimated loss parameter $\hat\cost(\sample)$. The algorithm also takes as input the distribution $\dist$ and in the suboptimal cases the hypothesis class $\hypcls$ or family of hypothesis classes $\hypset$, but we leave this dependence implicit in our notation. First, we consider the optimal (Theorem \ref{thm:opt_dec}) and known suboptimal (Theorem \ref{thm:subopt_known}) cases; since these are nearly identical, we focus on the optimal case.

\begin{restatable}[Lower bound for optimal decision maker]{theorem}{theoremoptimallower}
\label{thm:optimal_lower}
Fix $0 < \epsilon < \nicefrac{1}{4}$, $0 < \delta \leq \nicefrac{1}{2}$, and $0 < p_\cost \leq \nicefrac{1}{8 \epsilon}$. Then for any IDT algorithm $\hat{\cost}(\cdot)$, there exists a decision problem $(\dist, \cost)$ satisfying the conditions of Theorem \ref{thm:subopt_known} such that $m < \frac{\log(\nicefrac{1}{2 \delta})}{8 p_\cost \epsilon}$ implies that $\PP(| \hat\cost(\sample) - \cost | \geq \epsilon) > \delta$.
\end{restatable}

\begin{restatable}[Lack of uncertainty precludes identifiability]{corollary}{corollarycertainlower}
\label{corollary:certain_lower}
Fix $0 < \epsilon < 1/4$ and suppose a decision problem $(\dist, \cost)$ has no uncertainty. Then for any IDT algorithm $\hat\cost(\cdot)$, there is a loss parameter $\cost$ and hypothesis class $\hypcls$ such that for any sample size $m$, $\PP(| \hat\cost(\sample) - \cost | \geq \epsilon) \geq 1/2$.
\end{restatable}
Corollary \ref{corollary:certain_lower} shows that a lack of uncertainty in the decision problem means that no algorithm can learn the loss parameter $\cost$ to a non-trivial precision with high probability. Thus, uncertainty is \emph{required} for IDT to learn the loss parameter $\cost$. Since $\cost$ represents the preferences of the decision maker, \emph{decisions made under certainty do not reveal precise preference information}.
In Appendix \ref{sec:subopt_unknown_additional}, we explore lower bounds for the unknown suboptimal case (Section \ref{sec:suboptimal_unknown} and Theorem \ref{thm:subopt_unknown}).

\section{Discussion}
\label{sec:discussion}

Now that we have thoroughly analyzed IDT, we explore its applications, implications, and limitations.

\subsection{IDT for fine-grained loss functions with applications to fairness}
\label{sec:fairness}
First, we discuss an extension of IDT to loss functions which depend not only on the chosen decision $\Dec = \hyp(\Obs)$ and the ground truth $\Lbl$, but on the observation $\Obs$ as well. In particular, we extend the formulation of IDT from Section \ref{sec:problem} to include loss functions which depend on the observations via a ``sensitive attribute'' $\Att \in \attspace$. We denote the value of the sensitive attribute for an observation $\obs$ by $\att(\obs)$. We again assume that the decision maker chooses the optimal decision rule for this extended loss function:
\begin{equation}
    \label{eq:att_risk}
    \textstyle \hyp \in \argmin_\hyp \EE_{(\Obs, \Lbl) \sim \dist}[\ell(\hyp(\Obs), \Lbl, \att(\Obs))].
\end{equation}
This optimal decision rule $\hyp \in \hypcls$ is equivalent to a set of decision rules for every value of $\Att$, each of which is chosen to minimize the conditional risk for observations with that attribute value:
\begin{equation*}
    \hyp(\obs) = \hyp_{\att(\obs)}(\obs)
    \quad \text{where} \quad
    \hyp_\att \in \argmin_\hyp \EE_{(\Obs, \Lbl) \sim \dist}[\ell(\hyp(\Obs), \Lbl, \att) \mid \att(\Obs) = \att].
\end{equation*}
In this formulation, each attribute-specific decision rule $\hyp_\att$ minimizes an expected loss which only depends on the made and correct decisions $\hyp(\Obs)$ and $\Lbl$ over a conditional distribution.
Thus, we can split a sample of decisions into samples for each value of the sensitive attribute and perform IDT separately. This will result in a loss parameter estimate $\hat\cost_\att$ for each value of $\att$.

Once we have estimated loss parameters for each value of $\Att$, we may ask if the decision maker is applying the same loss function across all such values, i.e. if $\cost_\att = \cost_{\att'}$ for any $\att, \att' \in \attspace$. If the loss function is not identical for all values of $\Att$, i.e. if $\cost_\att \neq \cost_{\att'}$, then one might conclude that the decision maker is unfair or discriminatory against observations with certain values of $\Att$. For instance, in the ER example, we might be concerned if the doctor is using different loss functions for patients with and without insurance. Concepts like these have received extensive treatment in the machine learning fairness literature, which studies criteria for when a decision rule can be considered ``fair.'' One such fairness criterion is that of group calibration, also known as sufficiency \citep{kleinberg_inherent_2016,liu_implicit_2019,barocas_fairness_2019}:

\begin{restatable}{defn}{defnsufficiency}
A decision rule $\hyp: \obsspace \to \{0, 1\}$ for a distribution $(\Obs, \Lbl) \sim \dist$ satisfies the \emph{group calibration/sufficiency} fairness criterion if there is a function $r: \obsspace \to \mathbb{R}$ and threshold $t \in \mathbb{R}$ such that $\hyp(\obs) = \II\{r(\obs) \geq t\}$ and $r$ satisfies $\Lbl \indep \Att \mid r(\Obs)$.
\end{restatable}
Testing for group calibration is known to be difficult because of the problem of infra-marginality \citep{simoiu_problem_2017}. While complex Bayesian models have previously been used to perform a ``threshold test'' for group calibration, we can use IDT to directly test this criterion in an observed decision maker:
\begin{restatable}[Equal loss parameters imply group calibration]{lemma}{lemmafairness}
\label{lemma:fairness}
Let $\hyp$ be chosen as in (\ref{eq:att_risk}) where $\ell(\dec, \lbl, \att) = \cost_\att$ if $\dec = 1$ and $\lbl = 0$, $\ell(\dec, \lbl, \att) = 1 - \cost_\att$ if $\dec = 0$ and $\lbl = 1$, and $\ell(\dec, \lbl, \att) = 0$ otherwise.
Then $\hyp$ satisfies group calibration (sufficiency) if $\cost_\att = \cost_{\att'}$ for every $\att, \att' \in \attspace$.

Conversely, if there exist $\att, \att' \in \attspace$ such that $\cost_\att \neq \cost_\att'$ and $\PP(\post(\Obs) \in (\cost_\att, \cost_{\att'})) > 0$, then $\hyp$ does not satisfy group calibration.
\end{restatable}
If we can estimate $\cost_\att$ for a decision rule $\hyp$ for each $\att \in \attspace$, then Lemma \ref{lemma:fairness} allows us to immediately determine if $\hyp$ satisfies sufficiency.
The minimax guarantees on the accuracy of IDT may make this approach more attractive than the Bayesian threshold test in many scenarios.


\subsection{Suboptimal decision making with and without uncertainty}
\label{sec:uncertain_certain}
We have so far compared the effect of decisions made with and without uncertainty on the \emph{identifiability} of preferences; here, we argue that uncertainty also allows for much more expressive models of \emph{suboptimality} in decision making.
In decisions made with certainty, suboptimality can generally only take two forms: either the decision maker is noisy and sometimes randomly makes incorrect decisions, or the decision maker is systematically suboptimal and always makes the wrong decision. Neither seems realistic in the ER doctor example: we would not expect to the doctor to randomly choose not to treat some patients who are clearly having heart attacks, and certainly not expect them to \emph{never} treat patients having heart attacks. In contrast, the models of suboptimality we have presented for uncertain decisions allow for much more rich and realistic forms of suboptimal decision making, like ignoring certain data or over-/under-weighting evidence. We expect that there are similarly more rich forms of suboptimality for uncertain sequential decision problems.

\subsection{Limitations and future work}
\label{sec:limitations}
While this study sheds significant light on preference learning for uncertain humans, there are some limitations that may be addressed by future work. First, while we assume the data distribution $\dist$ of observations $\Obs$ and ground truth decisions $\Lbl$ is known, this is rarely satisfied in practice. However, statistics is replete with methods for estimating properties of a data distribution given samples from it. Such methods are  beyond the scope of this work, which focuses on the less-studied problem of inferring a decision maker's loss function. Our work also lacks computational analysis of algorithms for performing IDT. However, such algorithms are likely straightforward; we decide to focus on the statistical properties of IDT, which are more relevant for preference learning in general.
Finally, we assume in this work that the decision maker is maximizing expected utility (EU), or equivalently minimizing expected loss. In reality, human decision making may not agree with EU theory; alternative models of decision making under uncertainty such as prospect theory are discussed in the behavioral economics literature \citep{kahneman_prospect_1979}. Some work has applied these models to statistical learning \citep{leqi_human-aligned_2019}, but we leave their implications for IDT to future work.


\section{Conclusion and Societal Impact}
\label{sec:conclusion}
We have presented an analysis of preference learning for uncertain humans through the setting of inverse decision theory. Our principle findings are that decisions made under uncertainty can reveal more preference information than obvious ones; and, that uncertainty can alleviate underspecification in preference learning, even in the case of suboptimal decision making. We hope that this and other work on preference learning will lead to AI systems which better understand human preferences and can thus better fulfill them. However, improved understanding of humans could also be applied by malicious actors to manipulate people or invade their privacy. Additionally, building AI systems which learn from human decisions could reproduce racism, sexism, and other harmful biases which are widespread in human decision-making.
Despite these concerns, understanding human preferences is important for the long-term positive impact of AI systems. Our work shows that uncertain decisions can be a valuable source of such preference information.

\begin{ack}
We would like to thank Kush Bhatia for valuable discussions, Meena Jagadeesan, Sam Toyer, and Alex Turner for feedback on drafts, and the NeurIPS reviewers for helping us improve the clarity of the paper. This research was supported by the Open Philanthropy Foundation. Cassidy Laidlaw is also supported by a National Defense Science and Engineering Graduate (NDSEG) Fellowship.
\end{ack}

\small

\bibliographystyle{unsrtnat}
\bibliography{paper}

\begin{thebibliography}{46}
\providecommand{\natexlab}[1]{#1}
\providecommand{\url}[1]{\texttt{#1}}
\expandafter\ifx\csname urlstyle\endcsname\relax
  \providecommand{\doi}[1]{doi: #1}\else
  \providecommand{\doi}{doi: \begingroup \urlstyle{rm}\Url}\fi

\bibitem[Hertwig and Erev(2009)]{hertwig_descriptionexperience_2009}
Ralph Hertwig and Ido Erev.
\newblock The {Description}–{Experience} {Gap} in {Risky} {Choice}.
\newblock \emph{Trends in Cognitive Sciences}, 13\penalty0 (12):\penalty0
  517--523, December 2009.
\newblock ISSN 1364-6613.
\newblock \doi{10.1016/j.tics.2009.09.004}.
\newblock URL
  \url{https://www.sciencedirect.com/science/article/pii/S1364661309002125}.

\bibitem[Ng and Russell(2000)]{ng_algorithms_2000}
Andrew~Y. Ng and Stuart~J. Russell.
\newblock Algorithms for {Inverse} {Reinforcement} {Learning}.
\newblock In \emph{{ICML}}, volume~1, page~2, 2000.

\bibitem[Mullainathan and Obermeyer(2019)]{mullainathan_machine_2019}
Sendhil Mullainathan and Ziad Obermeyer.
\newblock A {Machine} {Learning} {Approach} to {Low}-{Value} {Health} {Care}:
  {Wasted} {Tests}, {Missed} {Heart} {Attacks} and {Mis}-predictions.
\newblock Technical report, National Bureau of Economic Research, 2019.

\bibitem[Armstrong and Mindermann(2018)]{armstrong_occams_2018}
Stuart Armstrong and Sören Mindermann.
\newblock Occam's {Razor} is {Insufficient} to {Infer} the {Preferences} of
  {Irrational} {Agents}.
\newblock \emph{Advances in Neural Information Processing Systems}, 31, 2018.
\newblock URL
  \url{https://proceedings.neurips.cc/paper/2018/hash/d89a66c7c80a29b1bdbab0f2a1a94af8-Abstract.html}.

\bibitem[Choi and Kim(2011)]{choi_inverse_2011}
Jaedeug Choi and Kee-Eung Kim.
\newblock Inverse {Reinforcement} {Learning} in {Partially} {Observable}
  {Environments}.
\newblock \emph{Journal of Machine Learning Research}, 12\penalty0
  (21):\penalty0 691--730, 2011.
\newblock ISSN 1533-7928.
\newblock URL \url{http://jmlr.org/papers/v12/choi11a.html}.

\bibitem[Chinaei and Chaib-Draa(2012)]{chinaei_inverse_2012}
Hamid~R. Chinaei and Brahim Chaib-Draa.
\newblock An {Inverse} {Reinforcement} {Learning} {Algorithm} for {Partially}
  {Observable} {Domains} with {Application} on {Healthcare} {Dialogue}
  {Management}.
\newblock volume~1, pages 144--149, December 2012.
\newblock \doi{10.1109/ICMLA.2012.31}.

\bibitem[Swartz et~al.(2006)Swartz, Cox, Cantor, Davies, and
  Follen]{swartz_inverse_2006}
Richard~J Swartz, Dennis~D Cox, Scott~B Cantor, Kalatu Davies, and Michele
  Follen.
\newblock Inverse {Decision} {Theory}.
\newblock \emph{Journal of the American Statistical Association}, 101\penalty0
  (473):\penalty0 1--8, March 2006.
\newblock ISSN 0162-1459.
\newblock \doi{10.1198/016214505000000998}.
\newblock URL
  \url{https://amstat.tandfonline.com/doi/abs/10.1198/016214505000000998}.
\newblock Publisher: Taylor \& Francis.

\bibitem[Davies(2005)]{davies_inverse_2005}
Kalatu Davies.
\newblock \emph{Inverse {Decision} {Theory} with {Medical} {Applications}}.
\newblock PhD thesis, Rice University, Houston, Texas, May 2005.
\newblock URL \url{https://scholarship.rice.edu/handle/1911/18756}.

\bibitem[Shah et~al.(2019)Shah, Gundotra, Abbeel, and
  Dragan]{shah_feasibility_2019}
Rohin Shah, Noah Gundotra, Pieter Abbeel, and Anca Dragan.
\newblock On the {Feasibility} of {Learning}, {Rather} than {Assuming}, {Human}
  {Biases} for {Reward} {Inference}.
\newblock In \emph{International {Conference} on {Machine} {Learning}}, pages
  5670--5679. PMLR, 2019.

\bibitem[Haussler(1992)]{haussler_decision_1992}
David Haussler.
\newblock Decision {Theoretic} {Generalizations} of the {PAC} {Model} for
  {Neural} {Net} and {Other} {Learning} {Applications}.
\newblock \emph{Information and Computation}, 100\penalty0 (1):\penalty0
  78--150, September 1992.
\newblock ISSN 0890-5401.
\newblock \doi{10.1016/0890-5401(92)90010-D}.
\newblock URL
  \url{https://www.sciencedirect.com/science/article/pii/089054019290010D}.

\bibitem[Kearns et~al.(1994)Kearns, Schapire, and Sellie]{kearns_toward_1994}
Michael~J. Kearns, Robert~E. Schapire, and Linda~M. Sellie.
\newblock Toward {Efficient} {Agnostic} {Learning}.
\newblock \emph{Machine Learning}, 17\penalty0 (2):\penalty0 115--141, November
  1994.
\newblock ISSN 1573-0565.
\newblock \doi{10.1007/BF00993468}.
\newblock URL \url{https://doi.org/10.1007/BF00993468}.

\bibitem[Abbeel and Ng(2004)]{abbeel_apprenticeship_2004}
Pieter Abbeel and Andrew~Y. Ng.
\newblock Apprenticeship {Learning} via {Inverse} {Reinforcement} {Learning}.
\newblock In \emph{Proceedings of the twenty-first international conference on
  {Machine} learning}, page~1, 2004.

\bibitem[Ramachandran and Amir(2007)]{ramachandran_bayesian_2007}
Deepak Ramachandran and Eyal Amir.
\newblock Bayesian {Inverse} {Reinforcement} {Learning}.
\newblock In \emph{{IJCAI}}, volume~7, pages 2586--2591, 2007.

\bibitem[Ziebart et~al.(2008)Ziebart, Maas, Bagnell, and
  Dey]{ziebart_maximum_2008}
Brian~D. Ziebart, Andrew~L. Maas, J.~Andrew Bagnell, and Anind~K. Dey.
\newblock Maximum {Entropy} {Inverse} {Reinforcement} {Learning}.
\newblock In \emph{Aaai}, volume~8, pages 1433--1438. Chicago, IL, USA, 2008.

\bibitem[Fu et~al.(2017)Fu, Luo, and Levine]{fu_learning_2017}
Justin Fu, Katie Luo, and Sergey Levine.
\newblock Learning {Robust} {Rewards} with {Adversarial} {Inverse}
  {Reinforcement} {Learning}.
\newblock \emph{arXiv preprint arXiv:1710.11248}, 2017.

\bibitem[Hiranandani et~al.(2019{\natexlab{a}})Hiranandani, Boodaghians, Mehta,
  and Koyejo]{hiranandani_performance_2019}
Gaurush Hiranandani, Shant Boodaghians, Ruta Mehta, and Oluwasanmi Koyejo.
\newblock Performance {Metric} {Elicitation} from {Pairwise} {Classifier}
  {Comparisons}.
\newblock \emph{arXiv:1806.01827 [cs, stat]}, January 2019{\natexlab{a}}.
\newblock URL \url{http://arxiv.org/abs/1806.01827}.
\newblock arXiv: 1806.01827.

\bibitem[Hiranandani et~al.(2019{\natexlab{b}})Hiranandani, Boodaghians, Mehta,
  and Koyejo]{hiranandani_multiclass_2019}
Gaurush Hiranandani, Shant Boodaghians, Ruta Mehta, and Oluwasanmi~O Koyejo.
\newblock Multiclass {Performance} {Metric} {Elicitation}.
\newblock In H.~Wallach, H.~Larochelle, A.~Beygelzimer,
  F.~d{\textbackslash}textquotesingle Alché-Buc, E.~Fox, and R.~Garnett,
  editors, \emph{Advances in {Neural} {Information} {Processing} {Systems} 32},
  pages 9356--9365. Curran Associates, Inc., 2019{\natexlab{b}}.
\newblock URL
  \url{http://papers.nips.cc/paper/9133-multiclass-performance-metric-elicitation.pdf}.

\bibitem[Hiranandani et~al.(2020)Hiranandani, Narasimhan, and
  Koyejo]{hiranandani_fair_2020}
Gaurush Hiranandani, Harikrishna Narasimhan, and Oluwasanmi Koyejo.
\newblock Fair {Performance} {Metric} {Elicitation}.
\newblock \emph{arXiv:2006.12732 [cs, stat]}, November 2020.
\newblock URL \url{http://arxiv.org/abs/2006.12732}.
\newblock arXiv: 2006.12732.

\bibitem[Biyik and Sadigh(2018)]{biyik_batch_2018}
Erdem Biyik and Dorsa Sadigh.
\newblock Batch {Active} {Preference}-{Based} {Learning} of {Reward}
  {Functions}.
\newblock In \emph{Proceedings of {The} 2nd {Conference} on {Robot}
  {Learning}}, pages 519--528. PMLR, October 2018.
\newblock URL \url{https://proceedings.mlr.press/v87/biyik18a.html}.
\newblock ISSN: 2640-3498.

\bibitem[Mindermann et~al.(2019)Mindermann, Shah, Gleave, and
  Hadfield-Menell]{mindermann_active_2019}
Sören Mindermann, Rohin Shah, Adam Gleave, and Dylan Hadfield-Menell.
\newblock Active {Inverse} {Reward} {Design}.
\newblock \emph{arXiv:1809.03060 [cs, stat]}, November 2019.
\newblock URL \url{http://arxiv.org/abs/1809.03060}.
\newblock arXiv: 1809.03060.

\bibitem[Bıyık et~al.(2019)Bıyık, Palan, Landolfi, Losey, and
  Sadigh]{biyik_asking_2019}
Erdem Bıyık, Malayandi Palan, Nicholas~C. Landolfi, Dylan~P. Losey, and Dorsa
  Sadigh.
\newblock Asking {Easy} {Questions}: {A} {User}-{Friendly} {Approach} to
  {Active} {Reward} {Learning}.
\newblock \emph{arXiv:1910.04365 [cs]}, October 2019.
\newblock URL \url{http://arxiv.org/abs/1910.04365}.
\newblock arXiv: 1910.04365.

\bibitem[Bhatia et~al.(2021)Bhatia, Bartlett, Dragan, and
  Steinhardt]{bhatia_agnostic_2021}
Kush Bhatia, Peter~L. Bartlett, Anca~D. Dragan, and Jacob Steinhardt.
\newblock Agnostic {Learning} with {Unknown} {Utilities}.
\newblock \emph{arXiv:2104.08482 [cs, stat]}, April 2021.
\newblock URL \url{http://arxiv.org/abs/2104.08482}.
\newblock arXiv: 2104.08482.

\bibitem[Cohen et~al.(1987)Cohen, Jaffray, and Said]{cohen_experimental_1987}
Michele Cohen, Jean-Yves Jaffray, and Tanios Said.
\newblock Experimental {Comparison} of {Individual} {Behavior} {Under} {Risk}
  and {Under} {Uncertainty} for {Gains} and for {Losses}.
\newblock \emph{Organizational Behavior and Human Decision Processes},
  39\penalty0 (1):\penalty0 1--22, February 1987.
\newblock ISSN 0749-5978.
\newblock \doi{10.1016/0749-5978(87)90043-4}.
\newblock URL
  \url{https://www.sciencedirect.com/science/article/pii/0749597887900434}.

\bibitem[Holt and Laury(2002)]{holt_risk_2002}
Charles~A. Holt and Susan~K. Laury.
\newblock Risk {Aversion} and {Incentive} {Effects}.
\newblock \emph{The American Economic Review}, 92\penalty0 (5):\penalty0
  1644--1655, 2002.
\newblock ISSN 0002-8282.
\newblock URL \url{https://www.jstor.org/stable/3083270}.
\newblock Publisher: American Economic Association.

\bibitem[Csermely and Rabas(2016)]{csermely_how_2016}
Tamás Csermely and Alexander Rabas.
\newblock How to {Reveal} {People}'s {Preferences}: {Comparing} {Time}
  {Consistency} and {Predictive} {Power} of {Multiple} {Price} {List} {Risk}
  {Elicitation} {Methods}.
\newblock \emph{Journal of Risk and Uncertainty}, 53\penalty0 (2):\penalty0
  107--136, 2016.
\newblock ISSN 0895-5646.
\newblock \doi{10.1007/s11166-016-9247-6}.

\bibitem[Luce(1977)]{luce_choice_1977}
R.~Duncan Luce.
\newblock The {Choice} {Axiom} {After} {Twenty} {Years}.
\newblock \emph{Journal of Mathematical Psychology}, 15\penalty0 (3):\penalty0
  215--233, June 1977.
\newblock ISSN 0022-2496.
\newblock \doi{10.1016/0022-2496(77)90032-3}.
\newblock URL
  \url{https://www.sciencedirect.com/science/article/pii/0022249677900323}.

\bibitem[Valiant(1984)]{valiant_theory_1984}
Leslie~G. Valiant.
\newblock A {Theory} of the {Learnable}.
\newblock \emph{Communications of the ACM}, 27\penalty0 (11):\penalty0
  1134--1142, 1984.
\newblock Publisher: ACM New York, NY, USA.

\bibitem[Gleave et~al.(2021)Gleave, Dennis, Legg, Russell, and
  Leike]{gleave_quantifying_2021}
Adam Gleave, Michael Dennis, Shane Legg, Stuart Russell, and Jan Leike.
\newblock Quantifying {Differences} in {Reward} {Functions}.
\newblock In \emph{International {Conference} on {Learning} {Representations}},
  2021.
\newblock URL \url{https://openreview.net/forum?id=LwEQnp6CYev}.

\bibitem[Devroye et~al.(2013)Devroye, Györfi, and
  Lugosi]{devroye_probabilistic_2013}
Luc Devroye, László Györfi, and Gábor Lugosi.
\newblock \emph{A {Probabilistic} {Theory} of {Pattern} {Recognition}},
  volume~31.
\newblock Springer Science \& Business Media, 2013.

\bibitem[Jeon et~al.(2020)Jeon, Milli, and Dragan]{jeon_reward-rational_2020}
Hong~Jun Jeon, Smitha Milli, and Anca~D. Dragan.
\newblock Reward-{Rational} ({Implicit}) {Choice}: {A} {Unifying} {Formalism}
  for {Reward} {Learning}.
\newblock \emph{arXiv:2002.04833 [cs]}, December 2020.
\newblock URL \url{http://arxiv.org/abs/2002.04833}.
\newblock arXiv: 2002.04833.

\bibitem[Freedman et~al.(2021)Freedman, Shah, and Dragan]{freedman_choice_2021}
Rachel Freedman, Rohin Shah, and Anca Dragan.
\newblock Choice {Set} {Misspecification} in {Reward} {Inference}.
\newblock \emph{arXiv:2101.07691 [cs]}, January 2021.
\newblock URL \url{http://arxiv.org/abs/2101.07691}.
\newblock arXiv: 2101.07691.

\bibitem[Shepard(1957)]{shepard_stimulus_1957}
Roger~N. Shepard.
\newblock Stimulus and {Response} {Generalization}: {A} {Stochastic} {Model}
  {Relating} {Generalization} to {Distance} in {Psychological} {Space}.
\newblock \emph{Psychometrika}, 22\penalty0 (4):\penalty0 325--345, December
  1957.
\newblock ISSN 1860-0980.
\newblock \doi{10.1007/BF02288967}.
\newblock URL \url{https://doi.org/10.1007/BF02288967}.

\bibitem[Baker et~al.(2007)Baker, Tenenbaum, and Saxe]{baker_goal_2007}
Chris~L. Baker, Joshua~B. Tenenbaum, and Rebecca~R. Saxe.
\newblock Goal {Inference} as {Inverse} {Planning}.
\newblock In \emph{Proceedings of the {Annual} {Meeting} of the {Cognitive}
  {Science} {Society}}, volume~29, 2007.
\newblock Issue: 29.

\bibitem[Kleinberg et~al.(2016)Kleinberg, Mullainathan, and
  Raghavan]{kleinberg_inherent_2016}
Jon Kleinberg, Sendhil Mullainathan, and Manish Raghavan.
\newblock Inherent {Trade}-{Offs} in the {Fair} {Determination} of {Risk}
  {Scores}.
\newblock \emph{arXiv:1609.05807 [cs, stat]}, November 2016.
\newblock URL \url{http://arxiv.org/abs/1609.05807}.
\newblock arXiv: 1609.05807.

\bibitem[Liu et~al.(2019)Liu, Simchowitz, and Hardt]{liu_implicit_2019}
Lydia~T. Liu, Max Simchowitz, and Moritz Hardt.
\newblock The {Implicit} {Fairness} {Criterion} of {Unconstrained} {Learning}.
\newblock In \emph{International {Conference} on {Machine} {Learning}}, pages
  4051--4060. PMLR, May 2019.
\newblock URL \url{http://proceedings.mlr.press/v97/liu19f.html}.
\newblock ISSN: 2640-3498.

\bibitem[Barocas et~al.(2019)Barocas, Hardt, and
  Narayanan]{barocas_fairness_2019}
Solon Barocas, Moritz Hardt, and Arvind Narayanan.
\newblock \emph{Fairness and {Machine} {Learning}}.
\newblock fairmlbook.org, 2019.

\bibitem[Simoiu et~al.(2017)Simoiu, Corbett-Davies, and
  Goel]{simoiu_problem_2017}
Camelia Simoiu, Sam Corbett-Davies, and Sharad Goel.
\newblock The {Problem} of {Infra}-{Marginality} in {Outcome} {Tests} for
  {Discrimination}.
\newblock \emph{The Annals of Applied Statistics}, 11\penalty0 (3), September
  2017.
\newblock ISSN 1932-6157.
\newblock \doi{10.1214/17-AOAS1058}.
\newblock URL
  \url{https://projecteuclid.org/journals/annals-of-applied-statistics/volume-11/issue-3/The-problem-of-infra-marginality-in-outcome-tests-for-discrimination/10.1214/17-AOAS1058.full}.

\bibitem[Kahneman and Tversky(1979)]{kahneman_prospect_1979}
Daniel Kahneman and Amos Tversky.
\newblock Prospect {Theory}: {An} {Analysis} of {Decision} under {Risk}.
\newblock \emph{Econometrica}, 47\penalty0 (2):\penalty0 263--291, 1979.
\newblock ISSN 0012-9682.
\newblock \doi{10.2307/1914185}.
\newblock URL \url{https://www.jstor.org/stable/1914185}.
\newblock Publisher: [Wiley, Econometric Society].

\bibitem[Leqi et~al.(2019)Leqi, Prasad, and Ravikumar]{leqi_human-aligned_2019}
Liu Leqi, Adarsh Prasad, and Pradeep~K Ravikumar.
\newblock On {Human}-{Aligned} {Risk} {Minimization}.
\newblock In H.~Wallach, H.~Larochelle, A.~Beygelzimer,
  F.~d{\textbackslash}textquotesingle Alché-Buc, E.~Fox, and R.~Garnett,
  editors, \emph{Advances in {Neural} {Information} {Processing} {Systems} 32},
  pages 15055--15064. Curran Associates, Inc., 2019.
\newblock URL
  \url{http://papers.nips.cc/paper/9642-on-human-aligned-risk-minimization.pdf}.

\bibitem[Vapnik(2006)]{vapnik_estimation_2006}
V.~Vapnik.
\newblock \emph{Estimation of {Dependences} {Based} on {Empirical} {Data}}.
\newblock Information {Science} and {Statistics}. Springer-Verlag, New York,
  2006.
\newblock ISBN 978-0-387-30865-4.
\newblock \doi{10.1007/0-387-34239-7}.
\newblock URL \url{https://www.springer.com/gp/book/9780387308654}.

\bibitem[Blumer et~al.(1989)Blumer, Ehrenfeucht, Haussler, and
  Warmuth]{blumer_learnability_1989}
Anselm Blumer, A.~Ehrenfeucht, David Haussler, and Manfred~K. Warmuth.
\newblock Learnability and the {Vapnik}-{Chervonenkis} {Dimension}.
\newblock \emph{Journal of the ACM}, 36\penalty0 (4):\penalty0 929--965,
  October 1989.
\newblock ISSN 0004-5411, 1557-735X.
\newblock \doi{10.1145/76359.76371}.
\newblock URL \url{https://dl.acm.org/doi/10.1145/76359.76371}.

\bibitem[Matoušek and Vondrák(2008)]{matousek_probablistic_2008}
Jiří Matoušek and Jan Vondrák.
\newblock The {Probablistic} {Method}.
\newblock Lecture {Notes}, Charles University, Prague, Czech Republic, March
  2008.

\bibitem[Angluin and Valiant(1979)]{angluin_fast_1979}
D.~Angluin and L.~G. Valiant.
\newblock Fast {Probabilistic} {Algorithms} for {Hamiltonian} {Circuits} and
  {Matchings}.
\newblock \emph{Journal of Computer and System Sciences}, 18\penalty0
  (2):\penalty0 155--193, April 1979.
\newblock ISSN 0022-0000.
\newblock \doi{10.1016/0022-0000(79)90045-X}.
\newblock URL
  \url{https://www.sciencedirect.com/science/article/pii/002200007990045X}.

\bibitem[Ehrenfeucht et~al.(1989)Ehrenfeucht, Haussler, Kearns, and
  Valiant]{ehrenfeucht_general_1989}
Andrzej Ehrenfeucht, David Haussler, Michael Kearns, and Leslie Valiant.
\newblock A {General} {Lower} {Bound} on the {Number} of {Examples} {Needed}
  for {Learning}.
\newblock \emph{Information and Computation}, 82\penalty0 (3):\penalty0
  247--261, 1989.
\newblock Publisher: Elsevier.

\bibitem[Vapnik(1991)]{vapnik_principles_1991}
V.~Vapnik.
\newblock Principles of {Risk} {Minimization} for {Learning} {Theory}.
\newblock \emph{Advances in Neural Information Processing Systems}, 4, 1991.
\newblock URL
  \url{https://proceedings.neurips.cc/paper/1991/hash/ff4d5fbbafdf976cfdc032e3bde78de5-Abstract.html}.

\bibitem[Rosasco et~al.(2004)Rosasco, De~Vito, Caponnetto, Piana, and
  Verri]{rosasco_are_2004}
Lorenzo Rosasco, Ernesto De~Vito, Andrea Caponnetto, Michele Piana, and
  Alessandro Verri.
\newblock Are loss functions all the same?
\newblock \emph{Neural Computation}, 16\penalty0 (5):\penalty0 1063--1076, May
  2004.
\newblock ISSN 0899-7667.
\newblock \doi{10.1162/089976604773135104}.
\newblock URL \url{https://doi.org/10.1162/089976604773135104}.

\end{thebibliography}

\newpage
\appendix
\allowdisplaybreaks
{\bf \Large Appendix}

\section{Proofs}
\label{sec:proofs}

\subsection{Proof of Lemma \ref{lemma:costmat}}
\lemmacostmat*
\begin{proof}
\label{proof:costmat}
Let $a = \costmat_{1 0} + \costmat_{0 1} - \costmat_{0 0} - \costmat_{1 1}$ and $b = \PP(\Lbl = 0) \costmat_{0 0} + \PP(\Lbl = 1) \costmat_{1 1}$. Then
\begin{align*}
    \risk_\costmat(\hyp) \\
    \quad = \;& \PP(\hyp(\Obs) = 0 \wedge \Lbl = 0) \costmat_{0 0} + \PP(\hyp(\Obs) = 1 \wedge \Lbl = 0) \costmat_{1 0} \\
    & + \PP(\hyp(\Obs) = 0 \wedge \Lbl = 1) \costmat_{0 1} + \PP(\hyp(\Obs) = 1 \wedge \Lbl = 1) \costmat_{1 1} \\
    = \;& \PP(\hyp(\Obs) = 1 \wedge \Lbl = 0) (\costmat_{1 0} - \costmat_{0 0}) + \PP(\Lbl = 0) \costmat_{0 0} \\
    & + \PP(\hyp(\Obs) = 0 \wedge \Lbl = 1) (\costmat_{0 1} - \costmat_{1 1}) + \PP(\Lbl = 1) \costmat_{1 1} \\
    = \;& \PP(\hyp(\Obs) = 1 \wedge \Lbl = 0) (\costmat_{1 0} - \costmat_{0 0}) + \PP(\hyp(\Obs) = 0 \wedge \Lbl = 1) (\costmat_{0 1} - \costmat_{1 1}) + b \\
    = \;& (\costmat_{1 0} + \costmat_{0 1} - \costmat_{0 0} - \costmat_{1 1}) \left( \PP(\hyp(\Obs) = 1 \wedge \Lbl = 0) \frac{\costmat_{1 0} - \costmat_{0 0}}{\costmat_{1 0} + \costmat_{0 1} - \costmat_{0 0} - \costmat_{1 1}} \right. \\
    & \left. + \PP(\hyp(\Obs) = 0 \wedge \Lbl = 1) \frac{\costmat_{0 1} - \costmat_{1 1}}{\costmat_{1 0} + \costmat_{0 1} - \costmat_{0 0} - \costmat_{1 1}} \right) + b \\
    = \;& a (\PP(\hyp(\Obs) = 1 \wedge \Lbl = 0) \cost + \PP(\hyp(\Obs) = 0 \wedge \Lbl = 1) (1 - \cost)) + b \\
    = \;& a \risk_{\costmat'}(\hyp) + b.
\end{align*}
\end{proof}

\subsection{Proof of Lemma \ref{lemma:bayesopt}}
\lemmabayesopt*

This result is well-known \citep{devroye_probabilistic_2013} but we include a proof here for completeness.

\begin{proof}
\label{proof:bayesopt}
Let $\hyp(\obs) = \II\{\post(\obs) \geq \cost\}$ and let $\tilde{\hyp}: \obsspace \to \{0, 1\}$ be any other decision rule. We will show that not only is $\hyp$ an optimal decision rule, but in fact that if $\PP(\hyp(\Obs) \neq \tilde{\hyp}(\Obs) \wedge \post(\Obs) \neq c) > 0$, then $\risk_c(\tilde{\hyp}) > \risk_c(\hyp)$; that is, $\tilde{\hyp}$ is strictly suboptimal. Thus, any optimal decision rule $\hyp^*$ must satisfy $\hyp(\obs) = \hyp^*(\obs)$ almost surely except where $\post(\obs) = c$.

First, let's define the \emph{conditional risk of $\hyp$ at $\obs$}, denoted by $\risk_\cost(\hyp \mid \Obs = \obs)$:
\begin{equation*}
    \risk_\cost(\hyp \mid \Obs = \obs)
    = \cost \, \PP(\hyp(\Obs) = 1 \wedge \Lbl = 0 \mid \Obs = \obs) + (1 - \cost) \, \PP(\hyp(\Obs) = 0 \wedge \Lbl = 1 \mid \Obs = \obs).
\end{equation*}
Note that one of the two terms is always zero, depending on whether $\hyp(\Obs)$ is 0 or 1, since $\hyp(\Obs)$ is deterministic given $\Obs$. The risk of $\hyp$ is the expectation of the conditional risk:
\begin{equation*}
    \risk_\cost(\hyp) = \EE_{\Obs \sim \dist_\obs}[\risk_\cost(\hyp \mid \Obs = \obs)].
\end{equation*}
We can bound the conditional risk for the optimal decision rule $\hyp$:
\begin{align}
    \risk_\cost(\hyp \mid \obs) & = \begin{cases}
        \cost \PP(\Lbl = 0 \mid \Obs = \obs) \quad & \post(\obs) \geq \cost \\
        (1 - \cost) \PP(\Lbl = 1 \mid \Obs = \obs) \quad & \post(\obs) < \cost
    \end{cases} \nonumber \\
    & = \begin{cases}
        \cost (1 - \post(\obs)) \quad & \post(\obs) \geq \cost \\
        (1 - \cost) \post(\obs) \quad & \post(\obs) < \cost
    \end{cases} \nonumber \\
    & \leq \cost (1 - \cost). \label{eq:conditional_risk_optimal}
\end{align}

Now, consider the conditional risk for the other decision rule $\tilde{\hyp}$ at $\obs$. First, suppose $\tilde{\hyp}(\obs) = \hyp(\obs)$; that is, the decision rule agrees with the optimal one. Then clearly $\risk_\cost(\tilde{\hyp} \mid \Obs = \obs) = \risk_\cost(\hyp \mid \Obs = \obs) \leq \cost (1 - \cost)$. Next, suppose $\post(\obs) \neq \cost$ and $\tilde{\hyp}(\obs) \neq \hyp(\obs)$. Then
\begin{align}
    \risk_\cost(\tilde{\hyp} \mid \Obs = \obs) & = \begin{cases}
        \cost \PP(\Lbl = 0 \mid \Obs = \obs) \quad & \post(\obs) < \cost \\
        (1 - \cost) \PP(\Lbl = 1 \mid \Obs = \obs) \quad & \post(\obs) > \cost
    \end{cases} \nonumber \\
    & = \begin{cases}
        \cost (1 - \post(\obs)) \quad & \post(\obs) < \cost \\
        (1 - \cost) \post(\obs) \quad & \post(\obs) > \cost
    \end{cases} \nonumber \\
    & > \cost (1 - \cost). \label{eq:conditional_risk_not_optimal}
\end{align}
Finally, suppose $\post(\obs) = \cost$; in this case, it is clear that $\risk_\cost(\tilde{\hyp} \mid \Obs = \obs) = \cost (1 - \cost)$ regardless of what $\tilde{\hyp}(\obs)$ is. Putting this together, we can break down the risk of $\tilde{\hyp}$ by conditioning on whether $\tilde{\hyp}(\obs) = \hyp(\obs)$ or $\post(\obs) = \cost$:
\begin{align*}
    \risk_\cost(\tilde{\hyp}) & = \EE[\risk_\cost(\tilde{\hyp} \mid \Obs = \obs)] \\
    & = \EE[\risk_\cost(\tilde{\hyp} \mid \Obs = \obs) \mid \tilde{\hyp}(\Obs) = \hyp(\Obs) \vee \post(\Obs) = \cost] \; \PP(\tilde{\hyp}(\Obs) = \hyp(\Obs) \vee \post(\Obs) = \cost) \\
    & \quad + \EE[\risk_\cost(\tilde{\hyp} \mid \Obs = \obs) \mid \tilde{\hyp}(\Obs) \neq \hyp(\Obs) \wedge \post(\Obs) \neq \cost] \; \PP(\tilde{\hyp}(\Obs) \neq \hyp(\Obs) \wedge \post(\Obs) \neq \cost) \\
    & \overset{\text{(i)}}{>\!\!/\!\!\geq} \EE[\risk_\cost(\tilde{\hyp} \mid \Obs = \obs) \mid \tilde{\hyp}(\Obs) = \hyp(\Obs) \vee \post(\Obs) = \cost] \; \PP(\tilde{\hyp}(\Obs) = \hyp(\Obs) \vee \post(\Obs) = \cost) \\
    & \quad + \EE[c (1 - c) \mid \tilde{\hyp}(\Obs) \neq \hyp(\Obs) \wedge \post(\Obs) \neq \cost] \; \PP(\tilde{\hyp}(\Obs) \neq \hyp(\Obs) \wedge \post(\Obs) \neq \cost) \\
    & \overset{\text{(ii)}}{\geq} \EE[\risk_\cost(\hyp \mid \Obs = \obs) \mid \tilde{\hyp}(\Obs) = \hyp(\Obs) \vee \post(\Obs) = \cost] \; \PP(\tilde{\hyp}(\Obs) = \hyp(\Obs) \vee \post(\Obs) = \cost) \\
    & \quad + \EE[\risk_\cost(\hyp \mid \Obs = \obs) \mid \tilde{\hyp}(\Obs) \neq \hyp(\Obs) \wedge \post(\Obs) \neq \cost] \; \PP(\tilde{\hyp}(\Obs) \neq \hyp(\Obs) \wedge \post(\Obs) \neq \cost) \\
    & = \EE[\risk_\cost(\hyp \mid \Obs = \obs)] \\
    & = \risk_\cost(\hyp).
\end{align*}
(i) uses (\ref{eq:conditional_risk_not_optimal}) and (ii) uses (\ref{eq:conditional_risk_optimal}). The above shows that $\risk_\cost(\tilde{\hyp}) \geq \risk_\cost(\hyp)$ for any decision rule $\tilde{\hyp}$, demonstrating that $\hyp$ must have the lowest risk achievable. Note that (i) is strictly greater as long as $\PP(\tilde{\hyp}(\Obs) \neq \hyp(\Obs) \wedge \post(\Obs) \neq \cost) > 0$, validating the claim above that any optimal decision rule must agree with $\hyp$ almost surely except when $\post(\Obs) = \cost$.

\end{proof}

\subsection{Proof of Theorem \ref{thm:opt_dec}}
\thmoptdec*
\begin{proof}
\label{proof:opt_dec}
Let $\hyp$ denote the decision maker's decision rule. From the proof of Lemma \ref{lemma:bayesopt}, we know that the optimality of $\hyp$ means that $\hyp(\Obs) = \II\{\post(\Obs) \geq \cost\}$ almost surely as long as $\post(\Obs) \neq \cost$.

Let $E$ denote the event that we observe $\obs_i$ and $\obs_j$ in the sample such that $\post(\obs_i) \in (\cost, \cost + \epsilon]$ and $\post(\obs_j) \in [\cost - \epsilon, \cost)$:
\begin{equation*}
    E \quad = \quad
    \underbrace{\exists \obs_i \; \post(\obs_i) \in (\cost, \cost + \epsilon]}_{E_1}
    \quad \wedge \quad
   \underbrace{ \exists \obs_j \; \post(\obs_i) \in [\cost - \epsilon, \cost)}_{E_2}.
\end{equation*}

First, we will lower bound the probability of $E_1$:
\begin{align*}
    \PP(E_1)
    & = 1 - \PP(\forall \obs_i \; \post(\obs_i) \notin (\cost, \cost + \epsilon]) \\
    & = 1 - \left(\PP(\post(\Obs) \notin (\cost, \cost + \epsilon])\right)^m \\
    & = 1 - \left(1 - \PP(\post(\Obs) \in (\cost, \cost + \epsilon])\right)^m \\
    & \geq 1 - (1 - \epsilon p_\cost)^m \\
    & \geq 1 - e^{-m \epsilon p_\cost} \\
    & \geq 1 - e^{-\log(2 / \delta)} \\
    & = 1 - \delta / 2.
\end{align*}
Second, we will lower bound the probability of $E_2$:
\begin{align*}
    \PP(E_2)
    & = 1 - \PP(\forall j \; \post(\obs_j) \notin [\cost - \epsilon, \cost)) \\
    & = 1 - \left(\PP(\post(\Obs) \notin [\cost - \epsilon, \cost))\right)^m \\
    & = 1 - \left(1 - \PP(\post(\Obs) \in [\cost - \epsilon, \cost))\right)^m \\
    & \geq 1 - (1 - \epsilon p_\cost)^{m} \\
    & \geq 1 - e^{-m \epsilon p_\cost} \\
    & \geq 1 - e^{-\log(2 / \delta)} \\
    & = 1 - \delta / 2.
\end{align*}
Putting the above together, we can lower bound the probability of $E$:
\begin{align*}
    \PP(E) & = \PP(E_1 \wedge E_2) \\
    & = 1 - \PP(\neg E_1 \vee \neg E_2) \\
    & \geq 1 - \PP(\neg E_1) - \PP(\neg E_2) \\
    & \geq 1 - \delta.
\end{align*}
Finally, we will show that $E$ implies $|\hat\cost - \cost| \leq \epsilon$. Suppose $E$ occurs. Then $\post(\obs_i) > c$, so $\hyp(\obs_i) = \dec_i = 1$. This means that $\hat\cost \leq \post(\obs_i) \leq \cost + \epsilon$. Also, $\post(\obs_j) < c$, so $\hyp(\obs_j) = \dec_j = 0$. This means that $\hat\cost > \post(\obs_j) \geq \cost - \epsilon$. Thus
\begin{equation*}
    \cost - \epsilon < \hat\cost \leq \cost + \epsilon
\end{equation*}
\begin{equation*}
    |\hat\cost - \cost| \leq \epsilon.
\end{equation*}
So with probability at least $1 - \delta$, $|\hat\cost - \cost| \leq \epsilon$.
\end{proof}

\subsection{Proof of Lemma \ref{lemma:induced_postprob}}
The proof of Lemma \ref{lemma:induced_postprob} depends on another lemma, which will also be useful in the unknown hypothesis class setting. This lemma bounds the conditional probability that the correct decision $\Lbl = 1$ for observations $\obs$ between the decision boundaries of two optimal decision rules.
\begin{lemma}
\label{lemma:condprob}
Suppose $\opt_\dist(\hypcls)$ is monotone and let $\hyp_{\cost}, \hyp_{\cost'} \in \hypcls$ be optimal decision rules for loss parameters $\cost$ and $\cost'$, respectively, where $\cost < \cost'$. Then for every $\obs \in \obsspace$, $\hyp_{\cost'}(\obs) \leq \hyp_\cost(\obs)$. Furthermore, assuming $\PP(\hyp_{\cost}(\Obs) \neq \hyp_{\cost'}(\Obs)) = \PP(\hyp_{\cost}(\Obs) = 1 \wedge \hyp_{\cost'}(\Obs) = 0) > 0$,
\begin{equation*}
    \cost \leq \PP(\Lbl = 1 \mid \hyp_{\cost}(\Obs) = 1 \wedge \hyp_{\cost'}(\Obs) = 0) \leq \cost'.
\end{equation*}
\end{lemma}
\begin{proof}We can write the risk of a decision rule $\hyp$ for cost $\cost$ as
\begin{align}
    \risk_\cost(\hyp) & = \cost\, \PP(\hyp(\Obs) = 1 \wedge \Lbl = 0)
    + (1 - \cost)\, \PP(\hyp(\Obs) = 0 \wedge \Lbl = 1) \nonumber \\
    & = \cost \big[ \PP(\Lbl = 0) - \PP(\hyp(\Obs) = 0 \wedge \Lbl = 0) \big]
    + (1 - \cost)\, \PP(\hyp(\Obs) = 0 \wedge \Lbl = 1) \nonumber \\
    & = \cost \Big( \PP(\Lbl = 0) - \big[\PP(\hyp(\Obs) = 0) - \PP(\hyp(\Obs) = 0 \wedge \Lbl = 1) \big] \Big) + (1 - \cost)\, \PP(\hyp(\Obs) = 0 \wedge \Lbl = 1) \nonumber \\
    & = \cost \, \PP(\Lbl = 0) - \cost \, \PP(\hyp(\Obs) = 0) + \cost \, \PP(\hyp(\Obs) = 0 \wedge \Lbl = 1) \nonumber \\
    & \quad + \PP(\hyp(\Obs) = 0 \wedge \Lbl = 1) - \cost \, \PP(\hyp(\Obs) = 0 \wedge y = 1) \nonumber \\
    & = \cost \, \PP(\Lbl = 0) - \cost \, \PP(\hyp(\Obs) = 0) + \PP(\hyp(\Obs) = 0 \wedge \Lbl = 1). \label{eq:risk_rewrite}
\end{align}
Since $\hyp_\cost$ is optimal for $\cost$, we have
\begin{equation}
\label{eq:hyp_feasibility}
    \risk_\cost(\hyp_{\cost'}) - \risk_\cost(\hyp_\cost) \geq 0.
\end{equation}
Applying (\ref{eq:risk_rewrite}) to (\ref{eq:hyp_feasibility}) gives
\begin{equation}
\label{eq:hyp_feasibility_rewrite}
    \PP(\hyp_{\cost'}(\Obs) = 0 \wedge \Lbl = 1)
    - \PP(\hyp_\cost(\Obs) = 0 \wedge \Lbl = 1)
    - \cost\, \Big[ \PP(\hyp_{\cost'}(\Obs) = 0)
    - \PP(\hyp_\cost(\Obs) = 0) \Big]
    \geq 0.
\end{equation}

Now, suppose the lemma does not hold; that is, there is some $\obs \in \obsspace$ such that $\hyp_{\cost'}(\obs) > \hyp_\cost(\obs)$. Since $\opt_\dist(\hypcls)$ is monotone, this implies
\begin{equation}
\label{eq:monotone_claim} \tag{$\star$}
    \forall \obs \in \obsspace \quad \hyp_\cost(\obs) \leq \hyp_{\cost'}(\obs).
\end{equation}

Assuming (\ref{eq:monotone_claim}) we have the following two identities:
\begin{align*}
    \PP(\hyp_\cost(\Obs) = 0) - \PP(\hyp_{\cost'}(\Obs) = 0)
    & = \PP(\hyp_\cost(\Obs) = 0) \wedge \hyp_{\cost'}(\Obs) = 1) \\
    \PP(\hyp_\cost(\Obs) = 0 \wedge \Lbl = 1) - \PP(\hyp_{\cost'}(\Obs) = 0 \wedge \Lbl = 1)
    & = \PP(\hyp_\cost(\Obs) = 0) \wedge \hyp_{\cost'}(\Obs) = 1 \wedge \Lbl = 1).
\end{align*}
Plugging these in to (\ref{eq:risk_rewrite}) gives
\begin{equation*}
    \cost \, \PP(\hyp_\cost(\Obs) = 0) \wedge \hyp_{\cost'}(\Obs) = 1) - \PP(\hyp_\cost(\Obs) = 0) \wedge \hyp_{\cost'}(\Obs) = 1 \wedge \Lbl = 1) \geq 0
\end{equation*}
\begin{equation*}
    \PP(\hyp_\cost(\Obs) = 0) \wedge \hyp_{\cost'}(\Obs) = 1 \wedge \Lbl = 1)
    \leq \cost \, \PP(\hyp_\cost(\Obs) = 0) \wedge \hyp_{\cost'}(\Obs) = 1)
\end{equation*}
\begin{equation*}
    \frac{\PP(\hyp_\cost(\Obs) = 0) \wedge \hyp_{\cost'}(\Obs) = 1 \wedge \Lbl = 1)}{\PP(\hyp_\cost(\Obs) = 0) \wedge \hyp_{\cost'}(\Obs) = 1)} \leq \cost
\end{equation*}
\begin{equation*}
    \PP(\Lbl = 1 \mid \hyp_\cost(\Obs) = 0 \wedge \hyp_{\cost'}(\Obs) = 1) \leq \cost.
\end{equation*}
This is the first claim of the lemma. Now, we can apply the same set of steps to $\risk_{\cost'}(\hyp_\cost) - \risk_{\cost'}(\hyp_{\cost'}) \geq 0$ (i.e., using (\ref{eq:risk_rewrite}) and the above identities) to obtain
\begin{equation*}
    \cost' \leq \PP(\Lbl = 1 \mid \hyp_\cost(\Obs) = 0 \wedge \hyp_{\cost'}(\Obs) = 1).
\end{equation*}
Combining these two equations implies $\cost' \leq \cost$, but we assumed that $\cost < \cost'$, so this is a contradiction. Thus, (\ref{eq:monotone_claim}) must be false!

Since $\opt_\dist(\hypcls)$ is monotone, the falsity of (\ref{eq:monotone_claim}) implies that actually,
\begin{equation}
    \label{eq:monotone_claim_corrected}
    \forall \obs \in \obsspace \quad \hyp_{\cost'}(\obs) \leq \hyp_\cost(\obs).
\end{equation}

Now, we can complete the proof by repeating the above steps using (\ref{eq:monotone_claim_corrected}) instead of (\ref{eq:monotone_claim}) to obtain
\begin{equation*}
    \cost \leq \PP(\Lbl = 1 \mid \hyp_{\cost}(\Obs) = 1 \wedge \hyp_{\cost'}(\Obs) = 0) \leq \cost'.
\end{equation*}

\end{proof}

\lemmainducedpostprob*
\begin{proof}
Fix $\obs \in \obsspace$. Using Lemma \ref{lemma:condprob}, we have that
\begin{equation*}
    \cost < \cost' \quad \Rightarrow \quad \hyp_\cost(\obs) \geq \hyp_{\cost'}(\obs).
\end{equation*}
That is, $\hyp_\cost(\obs)$ is monotone non-increasing in $\cost$. This is enough to show that $\post_\hypcls(\obs)$ is well-defined. Consider three cases:
\begin{enumerate}
    \item $\forall \cost, \hyp_\cost(\obs) = 1$. In this case, $\overline{\post}_{\hypcls}(\obs) = \sup \{ \cost \in [0, 1] \mid \hyp_\cost(\obs) = 1 \} \cup \{0\} = 1$ and $\underline{\post}_\hypcls(\obs) = \inf \{ \cost \in [0, 1] \mid \hyp_\cost(\obs) = 0 \} \cup \{1\} = \inf \, \emptyset \cup \{1\} = 1$ so $\post_\hypcls(\obs) = 1$.
    \item $\forall \cost, \hyp_\cost(\obs) = 0$. In this case, $\overline{\post}_{\hypcls}(\obs) = \sup \{ \cost \in [0, 1] \mid \hyp_\cost(\obs) = 1 \} \cup \{0\} = \sup \, \emptyset \cup \{0\} = 0$ and $\underline{\post}_\hypcls(\obs) = \inf \{ \cost \in [0, 1] \mid \hyp_\cost(\obs) = 0 \} \cup \{1\} = 0$ so $\post_\hypcls(\obs) = 0$.
    \item $\exists \, \cost_0, \cost_1$ such that $\hyp_{\cost_0}(\obs) = 0$ and $\hyp_{\cost_1}(\obs) = 1$. In this case, neither $\{ \cost \in [0, 1] \mid \hyp_\cost(\obs) = 1 \}$ nor $\{ \cost \in [0, 1] \mid \hyp_\cost(\obs) = 0 \}$ is empty so we have
    \begin{align*}
        \overline{\post}_{\hypcls}(\obs) & = \sup \, \{ \cost \in [0, 1] \mid \hyp_\cost(\obs) = 1 \} \\
        \underline{\post}_\hypcls(\obs) & = \inf \, \{ \cost \in [0, 1] \mid \hyp_\cost(\obs) = 0 \}.
    \end{align*}
    Say $\post_\hypcls(\obs)$ is not well-defined; that is,
    \begin{equation*}
        \sup \, \{ \cost \in [0, 1] \mid \hyp_\cost(\obs) = 1 \}
        \neq
        \inf \, \{ \cost \in [0, 1] \mid \hyp_\cost(\obs) = 0 \}.
    \end{equation*}
    First, suppose $\sup \, \{ \cost \in [0, 1] \mid \hyp_\cost(\obs) = 1 \} < \inf \, \{ \cost \in [0, 1] \mid \hyp_\cost(\obs) = 0 \}$. Then there exists some $\cost$ for which $\hyp_\cost(\obs) \notin \{0, 1\}$, which is impossible. So $\sup \, \{ \cost \in [0, 1] \mid \hyp_\cost(\obs) = 1 \} > \inf \, \{ \cost \in [0, 1] \mid \hyp_\cost(\obs) = 0 \}$. However, this implies that $\exists \, \cost_1 \geq \cost_0$ such that $\hyp_{\cost_1}(\obs) = 1$ but $\hyp_{\cost_0}(\obs) = 0$. Since $\hyp_\cost(\obs)$ is nonincreasing in $c$, this is a contradiction. Thus $\post_\hypcls(\obs) = \overline{\post}_{\hypcls}(\obs) = \underline{\post}_\hypcls(\obs)$ is well-defined.
\end{enumerate}
\end{proof}
\corollarybayesoptsuboptknown*
\begin{proof}
Let
\begin{equation*}
    \hyp_\cost \in \argmin_{\hyp \in \hypcls} \risk_\cost(\hyp)
\end{equation*}
be an optimal decision rule in $\hypcls$ for loss parameter $\cost$.

Fix any $\obs \in \obsspace$. If $\post_\hypcls(\obs) = \cost$, we don't need to prove anything. If $\post_\hypcls(\obs) > \cost$, then suppose $\hyp_\cost(\obs) \neq 1$, i.e. $\hyp_\cost(\obs) = 0$. Then
\begin{equation*}
    \underline{\post}_\hypcls(\obs) = \inf \, \{ \cost' \in [0, 1] \mid \hyp_{\cost'}(\obs) = 0 \} \leq c
\end{equation*}
since $\hyp_\cost(\obs) = 0$. However, this is a contradiction since we assumed $\post_\hypcls(\obs) > \cost$. Thus $\hyp_\cost(\obs) = 1$.

Now, if $\post_\hypcls(\obs) < \cost$, suppose $\hyp_\cost(\obs) \neq 0$, i.e. $\hyp_\cost(\obs) = 1$. Then
\begin{equation*}
    \overline{\post}_\hypcls(\obs) = \sup \, \{ \cost' \in [0, 1] \mid \hyp_{\cost'}(\obs) = 1 \} \geq c.
\end{equation*}
This is also a contradiction since we assumed $\post_\hypcls(\obs) < \cost$, so $\hyp_\cost(\obs) = 0$.
\end{proof}

\subsection{Proof of Theorem \ref{thm:subopt_known}}
\theoremsuboptknown*

\begin{proof}
Let $\hyp \in \hypcls$ denote the decision maker's decision rule. From Corollary \ref{corollary:bayesopt_subopt_known}, we know that $\hyp(\obs) = \II\{\post_\hypcls(\obs) \geq \cost\}$ as long as $\post_\hypcls(\obs) \neq \cost$.

Let $E$ denote the event that we observe $\obs_i$ and $\obs_j$ in the sample such that $\post_\hypcls(\obs_i) \in (\cost, \cost + \epsilon]$ and $\post_\hypcls(\obs_j) \in [\cost - \epsilon, \cost)$. An analogous computation to the proof of Theorem \ref{thm:opt_dec} (Section \ref{proof:opt_dec}) shows that if $m \geq \frac{\log(2/\delta)}{p_\cost \epsilon}$, then $\PP(E) \geq 1 - \delta$.

If $E$ occurs, then $\hyp(\obs_i) = 1$ and so $\hat\cost \leq \cost + \epsilon$. Also, $\hyp(\obs_j) = 0$ so $\hat\cost \geq \cost - \epsilon$. Thus, we have
\begin{equation*}
    \PP(|\hat\cost - \cost| \leq \epsilon) \geq \PP(E) \geq 1 - \delta.
\end{equation*}
\end{proof}

\subsection{Proof of Theorem \ref{thm:subopt_unknown}}

\thmsuboptunknown*

\begin{proof}
Specifically, we will prove that $\PP(|\hat\cost - \cost| \leq \epsilon) \geq 1 - \delta$ as long as
\begin{equation}
    \label{eq:subopt_unknown_samples}
    m \geq O\left[ \left(\frac{\alpha}{\epsilon} + \frac{1}{\epsilon^2}\right)\left( \frac{d \log (\alpha / (p_\cost \epsilon)) + \log(1 / \delta)}{p_\cost}\right) \right].
\end{equation}

Throughout the proof, let $\hyp_\cost \in \argmin_{\hyp \in \hypcls} \risk_\cost(\hyp)$ be the true decision rule and let $\hat\hyp_{\hat\cost} \in \argmin_{\hat\hyp \in \hat\hypcls} \risk_{\hat\cost}(\hat\hypcls)$ be the estimated decision rule, i.e. one that agrees with the decisions in the sample of observations $\sample$.

First, we use a standard result from PAC learning theory to upper bound the disagreement between the estimated decision rule $\hat\hyp_{\hat\cost}$ and the true decision rule $\hyp_\cost$. In particular, since this is a case of \emph{realizable} PAC learning, i.e. the true decision rule $\hyp_\cost$ is in one of the hypothesis classes $\hypcls \in \hypset$, we have that
\begin{equation*}
    \PP(\hyp_\cost(\Obs) \neq \hat\hyp_{\hat\cost}(\Obs))
    \leq O \left(\frac{1}{\frac{\alpha}{p_\cost \epsilon} + \frac{1}{p_\cost \epsilon^2}}\right)
    = O \left( \min\left( \frac{p_\cost \epsilon}{\alpha}, p_\cost \epsilon^2 \right) \right)
\end{equation*}
with probability at least $1 - \delta$ over the drawn sample.
This bound follows from \citet{vapnik_estimation_2006} and \citet{blumer_learnability_1989} since the set of all possible hypotheses $\cup_{\hypcls \in \hypset} \hypcls$ has VC-dimension at most $d$, and we observe a sample of $m$ observations $\obs_i$ and decisions $\dec_i = \hyp_\cost(\obs_i)$ where $m$ satisfies (\ref{eq:subopt_unknown_samples}). In particular, denote
\begin{equation}
    \label{eq:disagree_bound}
    r = 
    \PP(\hyp_\cost(\Obs) \neq \hat\hyp_{\hat\cost}(\Obs))
    \leq \min\left( \frac{p_\cost \epsilon}{6 \alpha}, \frac{p_\cost \epsilon^2}{36} \right).
\end{equation}

Next, we show that (\ref{eq:disagree_bound}) implies that $| \hat\cost - \cost | \leq \epsilon$; since (\ref{eq:disagree_bound}) holds with probability at least $1 - \delta$, this is enough to complete the proof of Theorem \ref{thm:subopt_unknown}. We will prove that $\hat\cost - \cost \leq \epsilon$ given (\ref{eq:disagree_bound}). The proof that $\cost - \hat\cost \leq \epsilon$ is analogous. We require a technical lemma on probability theory:

\begin{lemma}
\label{lemma:overlap_bound}
Let $A$, $B$, and $C$ be events in a probability space with $\PP(A) > 0$ and $\PP(B) > 0$. Then
\begin{equation*}
    | \PP(C \mid A) - \PP(C \mid B) | \leq
    \frac{\PP(A \wedge \neg B) + \PP(\neg A \wedge B)}{\min(\PP(A), \PP(B))}.
\end{equation*}
\end{lemma}
\begin{subproof}[Proof of Lemma \ref{lemma:overlap_bound}]
To simply the proof of this lemma, we adopt the boolean algebra notation that $AB$ is equivalent to $A \wedge B$ and $\bar{A}$ is equivalent to $\neg A$. Then we have
\begin{align*}
    & \left| \PP(C \mid A) - \PP(C \mid B) \right| \\
    & \quad = \left| \frac{\PP(AC)}{\PP(A)} - \frac{\PP(BC)}{\PP(B)} \right| \\
    & \quad = \left| \frac{\PP(ABC) + \PP(A\bar{B}C)}{\PP(AB) + \PP(A\bar{B})} - \frac{\PP(ABC) + \PP(\bar{A}BC)}{\PP(AB) + \PP(\bar{A}B)} \right| \\
    & \quad = \frac{|\PP(ABC)\PP(\bar{A}B) + \PP(A\bar{B}C)\PP(B) - \PP(ABC)\PP(A\bar{B}) - \PP(\bar{A}BC)\PP(A)|}{\PP(A)\PP(B)} \\
    & \quad \overset{\text{(i)}}{\leq} \frac{\max\Big(\PP(ABC)\PP(\bar{A}B) + \PP(A\bar{B}C)\PP(B), \, \PP(ABC)\PP(A\bar{B}) + \PP(\bar{A}BC)\PP(A)\Big)}{\PP(A)\PP(B)} \\
    & \quad = \max\left( \frac{\PP(ABC)\PP(\bar{A}B) + \PP(A\bar{B}C)\PP(B)}{\PP(A)\PP(B)}, \, \frac{\PP(ABC)\PP(A\bar{B}) + \PP(\bar{A}BC)\PP(A)}{\PP(A)\PP(B)} \right) \\
    & \quad \overset{\text{(ii)}}{\leq} \max\left(
        \frac{\PP(B)\PP(\bar{A}B) + \PP(A\bar{B})\PP(B)}{\PP(A)\PP(B)}, \,
        \frac{\PP(A)\PP(A\bar{B}) + \PP(\bar{A}B)\PP(A)}{\PP(A)\PP(B)}
    \right) \\
    & \quad = \max\left( \frac{\PP(\bar{A}B) + \PP(A\bar{B})}{\PP(A)}, \, \frac{\PP(A\bar{B}) + \PP(\bar{A}B)}{\PP(B)} \right) \\
    & \quad = \frac{\PP(A\bar{B}) + \PP(\bar{A}B)}{\min(\PP(A), \PP(B))}.
\end{align*}
(i) uses the fact that for positive $u$ and $v$, $|u - v| \leq \max(u, v)$. (ii) uses the fact that $\PP(E_1 E_2) \leq \PP(E_1)$ for any events $E_1$ and $E_2$.
\end{subproof}

Essentially, Lemma \ref{lemma:overlap_bound} says that if events $A$ and $B$ have high ``overlap,'' then the conditional probabilities of another event $C$ given $A$ and $B$ should be close. We next carefully construct two such events with high overlap.

First, let $\cost' = \cost + \epsilon / 2$ and let $\hyp_{\cost'} \in \argmin_{\hyp \in \hypcls} \risk_{\cost'}(\hyp)$. Since $\hyp_\cost$ and $\hypset$ are $\alpha$-MD-smooth, we have that
\begin{align}
    \MD(\hyp_{\cost'}, \opt_\dist(\hat\hypcls))
    & \leq (1 + \alpha |\cost' - \cost|) \MD(\hyp_\cost, \opt_\dist(\hat\hypcls)) \\
    & \leq (1 + \alpha \epsilon / 2) \PP(\hyp_\cost(\Obs) \neq \hat\hyp_{\hat\cost}(\Obs)) \nonumber \\
    & \leq (1 + \alpha \epsilon / 2) r. \label{eq:cost_prime_bound}
\end{align}
Since $\MD(\hyp, \opt_\dist(\hat\hypcls)) = \inf_{\hat\hyp \in \opt_\dist(\hat\hypcls)} \PP(\hyp(\Obs) \neq \hat\hyp(\Obs))$, there must be some hypothesis $\hat\hyp_{\hat\cost'} \in \argmin_{\hat\hyp \in \hat\hypcls} \risk_{\hat\cost'}(\hat\hyp)$ that matches the minimum disagreement with $\hyp_{\cost'}$ plus a small positive number (in case the infimum is not achieved):
\begin{align*}
    \PP(\hat\hyp_{\hat\cost'}(\Obs) \neq \hyp_{\cost'}(\Obs))
    & \leq \MD(\hyp_{\cost'}, \opt_\dist(\hat{\cost})) + r \\
    & \leq (2 + \alpha \epsilon / 2) r.
\end{align*}
Now, let the events $A$, $B$, and $C$ be defined as follows:
\begin{align*}
    A: & \quad
    \hyp_\cost(\Obs) = 1 \wedge \hyp_{\cost'}(\Obs) = 0, \\
    B: & \quad
    \hat\hyp_{\hat\cost}(\Obs) = 1 \wedge \hat\hyp_{\hat\cost'}(\Obs) = 0, \\
    C: & \quad \Lbl = 1.
\end{align*}
Using Lemma \ref{lemma:overlap_bound}, we can write the bound
\begin{equation}
    \label{eq:overlap_result}
    \PP(\Lbl = 1 \mid B) \leq \PP(\Lbl = 1 \mid A) +  \frac{\PP(A \wedge \neg B \vee \neg A \wedge B)}{\min(\PP(A), \PP(B))}.
\end{equation}
We will establish bounds on each term in (\ref{eq:overlap_result}).

\smallparagraph{Upper bound on $\PP(A \wedge \neg B \vee \neg A \wedge B)$} It is easy to see that
\begin{align*}
    A \wedge \neg B \vee \neg A \wedge B
    \quad \Rightarrow \quad
    \hyp_\cost(\Obs) \neq \hat\hyp_{\hat\cost}(\Obs)
    \vee \hyp_{\cost'}(\Obs) \neq \hat\hyp_{\hat\cost'}(\Obs).
\end{align*}
Given this implication, it must be that
\begin{align*}
    \PP(A \wedge \neg B \vee \neg A \wedge B)
    & \leq \PP(\hyp_\cost(\Obs) \neq \hat\hyp_{\hat\cost}(\Obs) \vee \hyp_{\cost'}(\Obs) \neq \hat\hyp_{\hat\cost'}(\Obs)) \\
    & \leq (3 + \alpha \epsilon/2) r \\
    & \leq p_\cost \epsilon^2 / 12 + p_\cost \epsilon^2 / 12
    = p_\cost \epsilon^2 / 6
\end{align*}
where the inequalities follow from (\ref{eq:disagree_bound}) and (\ref{eq:cost_prime_bound}).

\smallparagraph{Lower bound on $\min(\PP(A), \PP(B))$}
Since $\hyp_\cost$ is optimal within $\hypcls$ for loss parameter $\cost$, Corollary \ref{corollary:bayesopt_subopt_known} gives that $\hyp_\cost(\obs) = 1$ if $\post_\hypcls(\obs) > \cost$. Similarly, $\hyp_{\hat\cost}(\obs) = 0$ if $\post_\hypcls(\obs) < \cost$. Therefore,
\begin{equation*}
    \post_\hypcls(\Obs) \in (\cost, \cost') \quad \Rightarrow \quad
    \hyp_\cost(\Obs) = 1 \wedge \hyp_{\hat\cost}(\Obs) = 0
    \quad \Leftrightarrow \quad
    A.
\end{equation*}
This implication allows us to lower bound $\PP(A)$:
\begin{equation*}
    \PP(A)
    \geq \PP(\post_\hypcls(\Obs) \in (\cost, \cost'))
    = \PP(\post_\hypcls(\Obs) \in (\cost, \cost + \epsilon / 2))
    \geq p_\cost \epsilon / 2
\end{equation*}
where the final inequality is by assumption. We also need to lower bound $\PP(B)$ in order to lower bound $\min(\PP(A), \PP(B))$:
\begin{align*}
    \PP(B) & = \PP(A \wedge B) + \PP(\neg A \wedge B) \\
    & = \PP(A) - \PP(A \wedge \neg B) + \PP(\neg A \wedge B) \\
    & \geq \PP(A) - \Big(\PP(A \wedge \neg B) + \PP(\neg A \wedge B)\Big) \\
    & \geq p_\cost \epsilon / 2 - p_\cost \epsilon^2 / 6 \\
    & \geq p_\cost \epsilon / 3.
\end{align*}
We assume that $\epsilon \leq 1$ to lower bound $\epsilon \geq \epsilon^2$, but this is fine since if $\epsilon > 1$ then Theorem \ref{thm:subopt_unknown} holds trivially. Thus we have $\min(\PP(A), \PP(B)) \geq p_\cost \epsilon / 3$.

\smallparagraph{Lower bound on $\PP(\Lbl = 1 \mid B)$}
By Lemma \ref{lemma:condprob}, we have that, since $\PP(B) > 0$,
\begin{equation*}
    \PP(\Lbl = 1 \mid B) = \PP(\Lbl = 1 \mid \hat\hyp_{\hat\cost}(\Obs) = 1 \wedge \hat\hyp_{\hat\cost'}(\Obs) = 0) \geq \hat\cost.
\end{equation*}

\smallparagraph{Upper bound on $\PP(\Lbl = 1 \mid A)$}
Similarly, by Lemma \ref{lemma:condprob}, we have that, since $\cost' > \cost$ and $\PP(A) > 0$,
\begin{equation*}
    \PP(\Lbl = 1 \mid A) = \PP(\Lbl = 1 \mid \hyp_{\cost}(\Obs) = 1 \wedge \hyp_{\cost'}(\Obs) = 0) \leq \cost' = \cost + \epsilon / 2.
\end{equation*}

\smallparagraph{Concluding the proof}
Given all these bounds, we can rewrite (\ref{eq:cost_prime_bound}) as
\begin{align*}
    \hat\cost \leq \PP(\Lbl = 1 \mid B) & \leq \PP(\Lbl = 1 \mid A) +  \frac{\PP(A \wedge \neg B \vee \neg A \wedge B)}{\min(\PP(A), \PP(B))} \\
    & \leq \cost + \epsilon / 2 + \frac{p_\cost \epsilon^2 / 6}{p_\cost \epsilon / 3} \\
    & \leq \cost + \epsilon / 2 + \epsilon / 2 = \cost + \epsilon \\
    \hat\cost - \cost & \leq \epsilon.
\end{align*}
This completes the proof that $\hat\cost - \cost \leq \epsilon$ with probability at least $1 - \delta$; the proof that $\cost - \hat\cost \leq \epsilon$ is analogous.
\end{proof}

\subsection{Proof of Theorem \ref{thm:optimal_lower}}

\theoremoptimallower*

\begin{proof}
Consider a distribution over $\Obs \in \obsspace = [0, 1]$ where
\begin{equation*}
    \post(\obs) = \PP(\Lbl = 1 \mid \Obs = \obs) = \obs.
\end{equation*}
Let the distribution $\dist_\Obs$ over $\Obs$ have density $p_\cost$ on the interval $(\nicefrac{1}{2} - 2 \epsilon, \nicefrac{1}{2} + 2 \epsilon)$ and let $\PP(\Obs = 0) = \PP(\Obs = 1) = \nicefrac{1}{2} - 2 p_\cost \epsilon$.

Let $\cost_1 = \nicefrac{1}{2} - \epsilon$ and $\cost_2 = \nicefrac{1}{2} + \epsilon$. Then clearly, for $\cost \in \{\cost_1, \cost_2\}$, the conditions of Theorem \ref{thm:opt_dec} are satisfied:
\begin{equation*}
    \PP(\post(\Obs) \in [\cost - \epsilon, \cost)) = \PP(\post(\Obs) \in (\cost, \cost + \epsilon]) = p_\cost \epsilon.
\end{equation*}
By Lemma \ref{lemma:bayesopt}, the optimal decision rule for loss parameter $\cost_1$ is $\hyp_{\cost_1}(\obs) = \II\{\obs \geq \cost_1\}$ and for $\cost_2$ it is $\hyp_{\cost_2}(\obs) = \II\{\obs \geq \cost_2\}$.

Now suppose
\begin{equation*}
    m < \frac{\log(\nicefrac{1}{2 \delta})}{8 p_\cost \epsilon}
\end{equation*}
as stated in the theorem. We can bound the probability of the following event $E$:
\begin{align*}
    \PP(\underbrace{\forall \obs_i \in \sample \quad \post(\obs_i) \in \{0, 1\}}_E)
    & = \left[ \PP(\Obs \in \{0, 1\}) \right]^m \\
    & = (1 - 4 p_\cost \epsilon)^m \\
    & \overset{\text{(i)}}{\geq} \Big( e^{-8 p_\cost \epsilon} \Big)^m \\
    & = e^{-\log(\nicefrac{1}{2 \delta})} = 2 \delta.
\end{align*}
(i) uses the fact that $1 - u \geq e^{-2u}$ for $u \in [0, \nicefrac{1}{2}]$. Now, suppose $E$ occurs. In this case, $\hyp_{\cost_1}(\obs_i) = \hyp_{\cost_2}(\obs_i)$ for all $\obs_i \in \sample$. That is, regardless of which loss parameter $\cost \in \{\cost_1, \cost_2\}$ is used, the distribution of samples will be the same. Let $\sample_1$ denote the random variable for a sample taken from a decision maker using $\hyp_{\cost_1}$ and $\sample_2$ a sample taken from $\hyp_{\cost_2}$. Since these have the same distribution under $E$, they must induce the same probabilities when the IDT algorithm $\hat\cost$ is applied to them:
\begin{align*}
    p_1 = \PP(\hat\cost(\sample_1) \leq \nicefrac{1}{2} \mid E) & = \PP(\hat\cost(\sample_2) \leq \nicefrac{1}{2} \mid E), \\
    p_2 = \PP(\hat\cost(\sample_1) > \nicefrac{1}{2} \mid E) & = \PP(\hat\cost(\sample_2) > \nicefrac{1}{2} \mid E).
\end{align*}
Since $p_1 + p_2 = 1$, at least one of $p_1, p_2 \geq \nicefrac{1}{2}$. Suppose WLOG that $p_1 \geq \nicefrac{1}{2}$. Then
\begin{align*}
    \PP(|\hat\cost(\sample_2) - \cost_2| \geq \epsilon)
    & \geq \PP(\hat\cost(\sample_2) \leq \nicefrac{1}{2}) \\
    & = \PP(\hat\cost(\sample_2) \leq \nicefrac{1}{2} \mid E) \, \PP(E) \\
    & \geq \nicefrac{1}{2} (2 \delta) = \delta.
\end{align*}
Thus there is a decision problem $(\dist, \cost_2)$ for which the IDT algorithm $\hat\cost$ must make an error of at least size $\epsilon$ with at least probability $\delta$. This concludes the proof.
\end{proof}

\corollarycertainlower*

\begin{proof}
Let the loss parameters $\cost_1 = \nicefrac{1}{2} - \epsilon$ and $\cost_2 = \nicefrac{1}{2} + \epsilon$ be defined as in the proof of Theorem \ref{thm:optimal_lower} above. By Lemma \ref{lemma:bayesopt}, the optimal decision rule for loss parameter $\cost_1$ is $\hyp_{\cost_1}(\obs) = \II\{\post(\obs) \geq \cost_1\}$ and for $\cost_2$ it is $\hyp_{\cost_2}(\obs) = \II\{\post(\obs) \geq \cost_2\}$. Since $\PP(\post(\obs) \in \{0, 1\}) = 1$, it is clear that the decision rules make the same decision rules almost surely, i.e. $\PP(\hyp_{\cost_1}(\Obs) = \hyp_{\cost_2}(\Obs)) = 1$. Thus, letting $\sample_1$ and $\sample_2$ denote samples drawn from decision rules $\hyp_{\cost_1}$ and $\hyp_{\cost_2}$, respectively, as above, we have that the distributions of $\sample_1$ and $\sample_2$ are indistinguishable. Thus by the same argument as above we can show that (WLOG)
\begin{equation*}
    \PP(|\hat\cost(\sample_2) - \cost_2| \geq \epsilon) \geq \PP(\hat\cost(\sample_2) \leq \nicefrac{1}{2})
    \geq \nicefrac{1}{2}.
\end{equation*}
\end{proof}

\subsection{Proof of Lemma \ref{lemma:fairness}}

\defnsufficiency*

\lemmafairness*

\begin{proof}[Proof that equal $\cost_\att$ imply group calibration]
Assume $\cost_\att = \cost_{\att'} = \cost$ for every $\att, \att' \in \attspace$. Then define
\begin{equation}
    \label{eq:score}
    r(\obs) = \post(\obs) + \hyp(\obs).
\end{equation}
That is, $r(\obs)$ is the posterior probability $\post(\obs) = \PP(\Lbl = 1 \mid \Obs = \obs)$ plus one if the decision rule outputs the decision $\hyp(\obs) = 1$.
From the proof of Lemma \ref{lemma:bayesopt}, we know that $\hyp(\obs) = 1$ if $\post(\obs) > c$ and $\hyp(\obs) = 0$ if $\post(\obs) < c$. From this and (\ref{eq:score}) we can write
\begin{equation*}
    \hyp(\obs) = \II\{ r(\obs) \geq \cost + 1\}.
\end{equation*}
Now we need to show that $\Lbl \indep \Att \mid r(\Obs)$. Note that $r(\Obs) \in [0, \cost] \cup [\cost + 1, 2]$. First, we consider $r(\Obs) \in [0, \cost]$. In this case, for any $\att \in \attspace$, we have
\begin{align*}
    \PP(\Lbl = 1 \mid \Att = \att, r(\Obs) = r)
    & = \PP(\Lbl = 1 \mid \Att = \att, \post(\Obs) = r) \\
    & = r \\
    & = \PP(\Lbl = 1 \mid r(\Obs) = r).
\end{align*}
Next, say $r(\Obs) \in [\cost + 1, 2]$. Then
\begin{align*}
    \PP(\Lbl = 1 \mid \Att = \att, r(\Obs) = r)
    & = \PP(\Lbl = 1 \mid \Att = \att, \post(\Obs) = r - 1) \\
    & = r - 1 \\
    & = \PP(\Lbl = 1 \mid r(\Obs) = r).
\end{align*}
So in either case, $\PP(\Lbl = 1 \mid \Att = \att, r(\Obs) = r) = \PP(\Lbl = 1 \mid r(\Obs) = r)$. Thus $\Lbl \indep \Att \mid r(\Obs)$.
\end{proof}

\begin{proof}[Proof of inverse]
Now, assume $\exists \, \att, \att' \in \attspace$ such that $\cost_\att \neq \cost_{\att'}$. WLOG, suppose that $\cost_{\att} < \cost_{\att'}$. Let $r: \obsspace \to \mathbb{R}$ be any function satisfying $\hyp(\obs) = \II\{r(\obs) \geq t\}$. WLOG we can also assume $t = 0$. 
From Lemma \ref{lemma:bayesopt}, we know that if $\att(\obs) = \att$, then $\post(\obs) < \cost_\att$ implies $\hyp(\obs) = 0$ and $\post(\obs) > \cost_\att$ implies $\hyp(\obs) = 1$. Also, if $\att(\obs) = \att'$, then $\post(\obs) < \cost_{\att'}$ implies $\hyp(\obs) = 0$ and $\post(\obs) > \cost_{\att'}$ implies $\hyp(\obs) = 1$. Therefore,
\begin{align*}
    & \PP(\Lbl = 1 \mid \Att = \att, r(\Obs) > 0) \\
    & \quad = \PP(\Lbl = 1 \mid \Att = \att, \post(\Obs) > \cost_\att) \\
    & \quad = \PP(\Lbl = 1 \mid \post(\Obs) > \cost_\att) \\
    & \quad = \PP(\Lbl = 1 \mid \post(\Obs) \in (\cost_\att, \cost_{\att'})) \frac{\PP(\post(\Obs) \in (\cost_\att, \cost_{\att'}))}{\PP(\post(\Obs) > \cost_\att)} \\
    & \qquad + \PP(\Lbl = 1 \mid \post(\Obs) \geq \cost_{\att'}) \frac{\PP(\post(\Obs) \geq \cost_{\att'})}{\PP(\post(\Obs) > \cost_\att)} \\
    & \quad \overset{\text{(i)}}{<} \cost_{\att'} \frac{\PP(\post(\Obs) \in (\cost_\att, \cost_{\att'}))}{\PP(\post(\Obs) > \cost_\att)}
    + \PP(\Lbl = 1 \mid \post(\Obs) \geq \cost_{\att'}) \frac{\PP(\post(\Obs) \geq \cost_{\att'})}{\PP(\post(\Obs) > \cost_\att)} \\
    & \quad \overset{\text{(ii)}}{\leq} \PP(\Lbl = 1 \mid \post(\Obs) \geq \cost_{\att'}) \\
    & \quad \leq \PP(\Lbl = 1 \mid \post(\Obs) > \cost_{\att'}) \\
    & \quad = \PP(\Lbl = 1 \mid \Att = \att', \post(\Obs) > \cost_{\att'}) \\
    & \quad = \PP(\Lbl = 1 \mid \Att = \att', r(\Obs) > 0).
\end{align*}

(i) and (ii) make use of the fact that
\begin{align*}
    & \PP(\Lbl = 1 \mid \post(\Obs) \in (\cost_\att, \cost_{\att'}))
    = \EE[\post(\Obs) \mid \post(\Obs) \in (\cost_\att, \cost_{\att'})] \\
    & \quad < \cost_{\att'}
    \leq \EE[\post(\Obs) \mid \post(\Obs) \geq \cost_{\att'}]
    = \PP(\Lbl = 1 \mid \post(\Obs) \geq \cost_{\att'}).
\end{align*}
(i) also uses the assumption that $\PP(\post(\Obs) \in (\cost_\att, \cost_{\att'})) > 0$.

Thus, we have that $\PP(\Lbl = 1 \mid \Att = \att, r(\Obs) > 0) \neq \PP(\Lbl = 1 \mid \Att = \att', r(\Obs) > 0)$; therefore, $\Lbl$ and $\Att$ are \emph{not} independent given $r(\Obs)$, so group calibration is not satisfied.
\end{proof}

\section{POMDP Formulation of IDT}
\label{sec:pomdp}

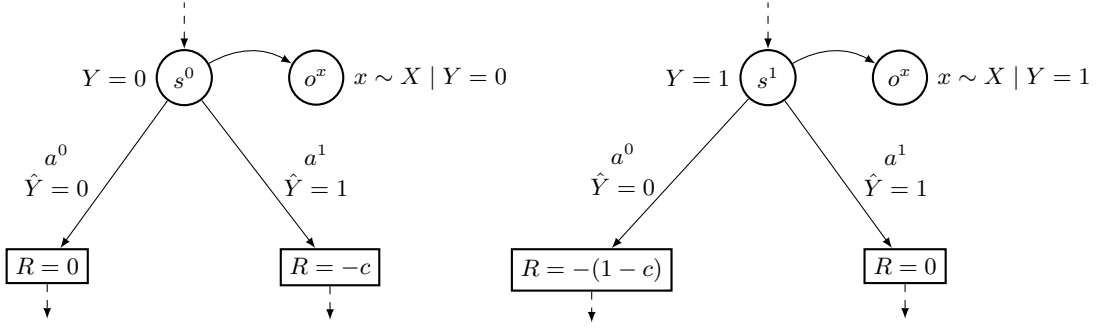
\begin{figure}
    \centering
    \begin{tikzpicture}[auto,node distance=8mm,>=latex,font=\small]

        \tikzstyle{state}=[thick,draw=black,circle]
        \tikzstyle{reward}=[thick,draw=black,rectangle]
    
        \node[state,label=left:${\Lbl=0}$] (s0) {$s^0$};
        \node[state,right=10mm of s0,label=right:${\obs \sim \Obs \mid \Lbl = 0}$] (o0) {$o^\obs$};
        \draw[->] (s0) to[bend left] (o0) ;
    
        \node[reward,below left=20mm and 10mm of s0] (s0a0) {$R = 0$};
        \node[reward,below right=20mm and 10mm of s0] (s0a1) {$R = -\cost$};
        \draw[->] (s0) -- node[left] {\begin{tabular}{c} $a^0$ \\ $\Dec = 0$\end{tabular}} (s0a0);
        \draw[->] (s0) -- node[right] {\begin{tabular}{c} $a^1$ \\ $\Dec = 1$\end{tabular}} (s0a1);
        
        \node[state,right=70mm of s0,label=left:${\Lbl=1}$] (s1) {$s^1$};
        \node[state,right=10mm of s1,label=right:${\obs \sim \Obs \mid \Lbl = 1}$] (o1) {$o^\obs$};
        \draw[->] (s1) to[bend left] (o1) ;
        
        \node[reward,below left=20mm and 10mm of s1] (s1a0) {$R = -(1 - \cost)$};
        \node[reward,below right=20mm and 10mm of s1] (s1a1) {$R = 0$};
        \draw[->] (s1) -- node[left] {\begin{tabular}{c} $a^0$ \\ $\Dec = 0$\end{tabular}} (s1a0);
        \draw[->] (s1) -- node[right] {\begin{tabular}{c} $a^1$ \\ $\Dec = 1$\end{tabular}} (s1a1);
    
        \draw[dashed,->] (s0a0) -- ++ (0, -0.7);
        \draw[dashed,->] (s0a1) -- ++ (0, -0.7);
        \draw[dashed,->] (s1a0) -- ++ (0, -0.7);
        \draw[dashed,->] (s1a1) -- ++ (0, -0.7);
        \draw[dashed,<-] (s0) -- ++ (0, 1);
        \draw[dashed,<-] (s1) -- ++ (0, 1);
    
    \end{tikzpicture}
    \caption{A graphical depiction of the POMDP formulation of IDT described in Appendix \ref{sec:pomdp}. A state in $\{s^0, s^1\}$ is randomly selected at each timestep and an observation is generated according to the conditional distribution of $\Obs \mid \Lbl$. An action (decision) is taken and the agent receives reward equal to the negative of the loss.}
    \label{fig:pomdp}
\end{figure}

As mentioned in the main text, IDT can be seen as a special case of inverse reinforcement learning (IRL) in a partially observable Markov decision process (POMDP) (or equivalently, belief state MDP). Here, we present the equivalent POMDP and discuss connections to to our results.

A POMDP is a tuple consisting of seven elements. For an IDT decision problem $(\dist, \cost)$ they are:
\begin{itemize}
    \item The state space consists of two states, each corresponding to a value of $\Lbl$, the ground truth/correct decision. We call them $s^0$ for $\Lbl = 0$ and $s^1$ for $\Lbl = 1$.
    \item The action space consists of two actions, each corresponding to one of the decisions $\Dec$. We equivalently call them $a^0$ for $\Dec = 0$ and $a^1$ for $\Dec = 1$.
    \item The transition probabilities do not depend on the previous state or action; rather, $s^0$ or $s^1$ is randomly selected based on their probabilities under the distribution $\dist$:
    \begin{align*}
        p(s_{t+1} = s^0 \mid s_t, a_t) = \PP_{\Obs, \Lbl \sim \dist}(\Lbl = 0), \\
        p(s_{t+1} = s^1 \mid s_t, a_t) = \PP_{\Obs, \Lbl \sim \dist}(\Lbl = 1). \\
    \end{align*}
    \item The reward function is the negative of the loss function described in Section \ref{sec:problem}:
    \begin{align*}
        R(s^0, a^0) & = 0  \quad &  R(s^1, a^0) & = -(1 - c), \\
        R(s^0, a^1) & = -c \quad &  R(s^1, a^1) & = 0. \\
    \end{align*}
    \item The observation space includes elements for each $\Obs \in \obsspace$. We denote by $o^\obs$ the POMDP observation for $\obs \in \obsspace$.
    \item The observation probabilities are
    \begin{equation*}
        p(o_t = o^\obs \mid s_t = s^\lbl) = \PP(\Obs = \obs \mid \Lbl = \lbl).
    \end{equation*}
    \item The discount factor $\gamma$ is basically irrelevant to IDT, since the decisions are non-sequential. Thus any $\gamma$ will produce the same behavior.
\end{itemize}

A graphical depiction of this POMDP is shown in Figure \ref{fig:pomdp}. Any decision rule $\hyp: \obsspace \to \{0, 1\}$ corresponds to a policy $\pi$ in this POMDP:
\begin{equation*}
    \pi(a_t = a^{\dec} \mid o_t = o^\obs) = \II\{\hyp(\obs) = \dec\}.
\end{equation*}

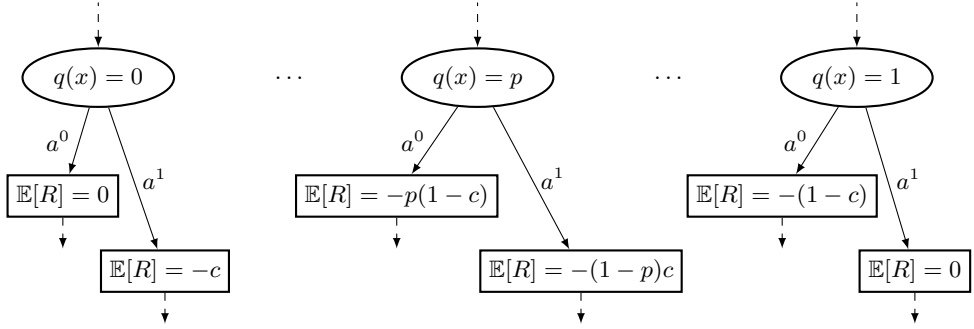
\begin{figure}
    \centering
    \begin{tikzpicture}[auto,node distance=8mm,>=latex,font=\small]
        \usetikzlibrary{shapes}

        \tikzstyle{beliefstate}=[thick,draw=black,ellipse]
        \tikzstyle{reward}=[thick,draw=black,rectangle]

        \node[beliefstate,label={[label distance=12mm]right:$\hdots$}] (s0) {$q(\obs) = 0$};
        \draw[dashed,<-] (s0) -- ++ (0, 1);
        \node[reward,below left=10mm and -10mm of s0] (s0a0) {$\EE [R] = 0$};
        \node[reward,below right=20mm and -7mm of s0] (s0a1) {$\EE [R] = -\cost$};
        \draw[->] (s0) -- node[left] {$a^0$} (s0a0);
        \draw[->] (s0) -- node[right] {$a^1$} (s0a1);
        
        \node[beliefstate,right=30mm of s0,label={[label distance=12mm]right:$\hdots$}] (s05) {$\post(\obs) = p$};
        \draw[dashed,<-] (s05) -- ++ (0, 1);
        \node[reward,below left=10mm and -10mm of s05] (s05a0) {$\EE [R] = -p (1 - \cost)$};
        \node[reward,below right=20mm and -7mm of s05] (s05a1) {$\EE [R] = -(1 - p)\cost$};
        \draw[->] (s05) -- node[left] {$a^0$} (s05a0);
        \draw[->] (s05) -- node[right] {$a^1$} (s05a1);
        
        \node[beliefstate,right=30mm of s05] (s1) {$q(\obs) = 1$};
        \draw[dashed,<-] (s1) -- ++ (0, 1);
        \node[reward,below left=10mm and -10mm of s1] (s1a0) {$\EE [R] = -(1 - \cost)$};
        \node[reward,below right=20mm and -7mm of s1] (s1a1) {$\EE [R] = 0$};
        \draw[->] (s1) -- node[left] {$a^0$} (s1a0);
        \draw[->] (s1) -- node[right] {$a^1$} (s1a1);
    
        \draw[dashed,->] (s0a0) -- ++ (0, -0.7);
        \draw[dashed,->] (s0a1) -- ++ (0, -0.7);
        \draw[dashed,->] (s05a0) -- ++ (0, -0.7);
        \draw[dashed,->] (s05a1) -- ++ (0, -0.7);
        \draw[dashed,->] (s1a0) -- ++ (0, -0.7);
        \draw[dashed,->] (s1a1) -- ++ (0, -0.7);
    
    \end{tikzpicture}
    \caption{A graphical depiction of the belief state MDP formulation of IDT. There is a belief state for each posterior probability $\post(\obs) = \PP(\Lbl = 1 \mid \Obs = \obs) \in [0, 1]$. Observing the agent at a belief state gives a constraint on their reward function \citep{ng_algorithms_2000}. Thus, if $\post(\Obs)$ has support on $[0, 1]$, i.e. if there is a significant range of uncertainty in the decision problem, then there can be arbitrarily many such constraints, allowing the loss parameter $\cost$ to be learned to arbitrary precision.}
    \label{fig:belief_state_mdp}
\end{figure}

\smallparagraph{Belief state MDP} The above POMDP can be equivalently formulated as a belief state MDP. The belief states correspond to values of the posterior probability
\begin{equation*}
    \PP(s = s^1 \mid o = o^\obs) = \PP(\Lbl = 1 \mid \Obs = \obs) = \post(\obs).
\end{equation*}
A graphical depiction of this belief state reduction is shown in Figure \ref{fig:belief_state_mdp}.

Since the POMDP is non-sequential, these beliefs only depend on the most recent observation $o^\obs$. The expected reward for action $a^{\dec}$ at belief state with posterior probability $\post(\obs)$ is
\begin{align*}
    R(\post(\obs), a^0)
    & = \PP(s = s^0 \mid \post(\obs)) R(s^0, a^0) + \PP(s = s^1 \mid \post(\obs)) R(s^1, a^0)
    = -\post(\obs) (1 - \cost), \\
    R(\post(\obs), a^1)
    & = \PP(s = s^0 \mid \post(\obs)) R(s^0, a^1) + \PP(s = s^1 \mid \post(\obs)) R(s^1, a^1)
    = -(1 - \post(\obs)) \cost. \\
\end{align*}
Thus, observing decision $a^0$ at a belief state $\post(\obs)$ indicates that
\begin{align*}
    R(\post(\obs), a^0) & \geq R(\post(\obs), a^1) \\
    -\post(\obs) (1 - \cost) & \geq -(1 - \post(\obs)) \cost \\
    c & \geq \post(\obs).
\end{align*}
Similarly, observing decision $a^1$ at a belief state $\post(\obs)$ indicates that
\begin{align*}
    R(\post(\obs), a^0) & \leq R(\post(\obs), a^1) \\
    -\post(\obs) (1 - \cost) & \leq -(1 - \post(\obs)) \cost \\
    c & \leq \post(\obs).
\end{align*}
Thus, as described in Section \ref{sec:optimal}, IDT in this (optimal) case consists of determining the threshold on $\post(\obs)$ where the action switches from $a^0$ to $a^1$ for observations $o^\obs$.

This formulation gives some additional insight into why uncertainty is helpful for IDT. If $\post(\obs) \in \{0, 1\}$ always, then there are only two belief states corresponding to $\post(\obs) = 0$ and $\post(\obs) = 1$. Thus, we only obtain two constraints on the value of $\cost$, i.e. $0 \leq \cost \leq 1$. However, if $\post(\Obs)$ has support on all of $[0, 1]$, then we there belief states corresponding to every $\post(\obs) \in [0, 1]$. Thus we can obtain infinite constraints on the value of $\cost$, allowing learning it to arbitrary precision as shown in Section \ref{sec:optimal}.

\section{Alternative Suboptimality Model}
\label{sec:alternative_subopt}

As mentioned in Section \ref{sec:suboptimal_known}, there are many ways to model suboptimal decision making. One possibility is to only require that the decision rule $\hyp$ is \emph{close} to optimal, i.e.
\begin{equation}
    \label{eq:alternative_subopt}
    \risk_\cost(\hyp) \leq \risk_\cost^\text{opt} + \Delta
    \qquad \text{where} \qquad
    \risk_\cost^\text{opt} = \inf_{\hyp^*} \risk_\cost(\hyp^*).
\end{equation}
However, as we show in the following lemma, this assumption can preclude identifiablity of $\cost$. The models of suboptimality we present in Sections \ref{sec:suboptimal_known} and \ref{sec:suboptimal_unknown}, in contrast, still allow exact identifiability of the loss parameter.

\begin{lemma}[Loss cannot always be identified for close-to-optimal decision rules]
\label{lemma:lower_alternative_subopt}
Fix $0 < \Delta \leq 1$ and $0 < \epsilon < 1/4$. Then for any IDT algorithm $\hat\cost(\cdot)$, there is a decision problem $(\dist, \cost)$ and a decision rule $\hyp$ which is $\Delta$-close to optimal as in (\ref{eq:alternative_subopt}) such that
\begin{equation*}
    \PP(|\hat\cost(\sample) - \cost| \geq \epsilon) \geq 1/2,
\end{equation*}
where the sample $\sample$ of any size $m$ is observed from the decision rule $\hyp$. Furthermore, the distribution $\dist$ and loss parameter $\cost$ satisfy the requirements of Theorem \ref{thm:opt_dec} for when the decision maker is optimal.
\end{lemma}
\begin{proof}
Consider a distribution over $\Obs \in \obsspace = [0, 1]$ where
\begin{equation*}
    \post(\obs) = \PP(\Lbl = 1 \mid \Obs = \obs) = \obs.
\end{equation*}
Let the distribution $\dist_\Obs$ have density $\Delta$ on the interval $(\nicefrac{1}{2} - 2 \epsilon, \nicefrac{1}{2} + 2 \epsilon)$ and let $\PP(\Obs = 0) = \PP(\Obs = 1) = \nicefrac{1}{2} - 2 \Delta \epsilon$.

Let $\cost_1 = \nicefrac{1}{2} - \epsilon$ and $\cost_2 = \nicefrac{1}{2} + \epsilon$. Then clearly $\PP(\post(\Obs) \in [\cost - \epsilon, \cost)) = \PP(\post(\Obs) \in (\cost, \cost + \epsilon]) = \epsilon \Delta$ for $\cost \in \{\cost_1, \cost_2\}$. Thus either $\cost_1$ or $\cost_2$ satisfies the conditions of Theorem \ref{thm:opt_dec}.

Now define identical decision rules
\begin{equation*}
    \hyp_1(\obs) = \hyp_2(\obs) = \II\{\obs \geq \nicefrac{1}{2} - \epsilon\}.
\end{equation*}
From Lemma \ref{lemma:bayesopt}, we know that $\hyp_1$ is optimal for $\cost_1$, so it is certainly $\Delta$-close to optimal. We can show that $\hyp_2$ is $\Delta$-close to optimal for $\cost_2$ as well:
\begin{align*}
    & \risk_{\cost_2}(\hyp_2) - \risk_{\cost_2}(\obs \mapsto \II\{\obs \geq \nicefrac{1}{2} + \epsilon\}) \\
    & \quad = \EE\Big[\ell(\II\{\Obs \geq \nicefrac{1}{2} - \epsilon\}, \Lbl) -  \ell(\II\{\Obs \geq \nicefrac{1}{2} + \epsilon\}, \Lbl)\Big] \\
    & \quad = \EE\Big[\ell(\II\{\Obs \geq \nicefrac{1}{2} - \epsilon\}, \Lbl) -  \ell(\II\{\Obs \geq \nicefrac{1}{2} + \epsilon\}, \Lbl) \mid \Obs \in [\nicefrac{1}{2} - \epsilon, \nicefrac{1}{2} + \epsilon] \Big] \PP(\Obs \in [\nicefrac{1}{2} - \epsilon, \nicefrac{1}{2} + \epsilon]) \\
    & \quad \leq 2 \PP(\Obs \in [\nicefrac{1}{2} - \epsilon, \nicefrac{1}{2} + \epsilon]) \\
    & \quad = 4 \epsilon \Delta \leq \Delta.
\end{align*}
Since $\hyp_1$ and $\hyp_2$ are identical, we must have that for a sample $\sample$ chosen according to either, at least one of $\PP(\hat\cost(\sample) \geq \nicefrac{1}{2}) \geq \nicefrac{1}{2}$ or $\PP(\hat\cost(\sample) < \nicefrac{1}{2}) \geq \nicefrac{1}{2}$. Thus for some $\cost \in \{\cost_1, \cost_2\}$,
\begin{equation*}
    \PP(|\hat\cost(\sample) - c| \geq \epsilon) \geq 1/2.
\end{equation*}
\end{proof}

\section{Additional Results for IDT with Suboptimal Decision Maker}
\label{sec:subopt_unknown_additional}

\subsection{Lower bound for unknown hypothesis class}
\label{sec:subopt_unknown_lower}

We give two lower bounds for the sample complexity in the unknown hypothesis class case from Section \ref{sec:suboptimal_unknown}. First, in Theorem \ref{thm:suboptimal_lower_nodim}, we show that there is an IDT problem such that $m = \Omega(\frac{\log(1/\delta)}{p_\cost \epsilon^2})$ samples are required to estimate $\cost$. Second, in Theorem \ref{thm:suboptimal_lower_dim}, we show that there is an IDT problem such that $m = \Omega(\frac{\sqrt{d}}{p_\cost \epsilon})$ samples are required. These lower bounds do not precisely match our upper bound of $m = O(\frac{d}{p_\cost \epsilon^2} + \frac{\log(1/\delta)}{p_\cost \epsilon^2})$ from Theorem \ref{thm:subopt_unknown}, and we leave as an open problem the exact minimax sample complexity of IDT in the unknown hypothesis class case. However, they do show that IDT does become harder as the VC-dimension $d$ increases, and that in some suboptimal cases a number of samples proportional to $\nicefrac{1}{\epsilon^2}$ is needed to estimate $\cost$ to precision $\epsilon$---more than the $\nicefrac{1}{\epsilon}$ needed for an optimal decision maker.

\begin{theorem}[First lower bound for suboptimal decision maker]
\label{thm:suboptimal_lower_nodim}
Fix $0 < \epsilon \leq 1/8$, $0 < \delta \leq 1/2$, and $p_\cost \leq 1/10$. Then there is a decision problem $(\dist, \cost)$, hypothesis class family $\hypset$, and hypothesis class $\hypcls \in \hypset$ satisfying the conditions of Theorem \ref{thm:subopt_unknown} with the above parameters such that
\begin{equation*}
    m < \Omega\left(\frac{\log(1/\delta)}{p_\cost \epsilon^2}\right)
    \quad \text{implies that} \quad
    \PP(| \hat\cost(\sample) - \cost | \geq \epsilon) \geq \delta.
\end{equation*}
\end{theorem}

\begin{proof}
Specifically, let the sample size
\begin{equation*}
    m = \frac{\log(1/(2 \delta))}{40 p_\cost \epsilon^2}.
\end{equation*}

\smallparagraph{Defining the distribution} First, we define a joint distribution $\dist$ over $\Obs = (\Obs_1, \Obs_2) \in \obsspace = \mathbb{R}^2$ and $\Lbl \in \{0, 1\}$. The distribution of $\Obs$ has support on 2 line segments in $\mathbb{R}^2$ and at a point. It can be summarized as follows:
\begin{enumerate}
    \item $\dist_\Obs$ has density $\frac{5 p_\cost}{2}$ on the line segment from $(-1, 0)$ to $(1, 0)$. \\ $\PP(\Lbl = 1 \mid \Obs = (\obs_1, 0)) = \frac{1 + \obs_1}{2}$.
    \item $\dist_\Obs$ has density $10 p_\cost \obs_1$ at points $(\obs_1, 1)$ on the line segment from $(0, 1)$ to $(1, 1)$. \\
    $\PP(\Lbl = 1 \mid \Obs = (\obs_1, 1)) = 1$.
    \item $\dist_\Obs$ has point mass $\PP(\Obs = (-1, 0)) = 1 - 10 p_\cost$. \\
    $\PP(\Lbl = 1 \mid \Obs = (-1, 0)) = 0$.
\end{enumerate}

\smallparagraph{Defining the family of hypothesis classes} Now, we define a family of two hypothesis classes:
\begin{align*}
    \hypcls^1 & \triangleq \{ \hyp(\obs) = \II\{\obs_1 \geq b\} \mid b \in [3/8, 5/8] \} \\
    \hypcls^2 & \triangleq \{ \hyp(\obs) = \II\{\obs_1 \geq b + 2 \epsilon \obs_2\} \mid b \in [1/2, 3/4] \} \\
    \hypset & \triangleq \{\hypcls^1, \hypcls^2\}.
\end{align*}
Let's analyze $\hypcls^1$ first. The posterior probability that $\Lbl = 1$ given that $\Obs_1 = \obs_1$ is
\begin{equation}
    \label{eq:post_h1}
    \PP(\Lbl = 1 \mid \Obs_1 = \obs_1) = \begin{cases}
        \frac{1 + \obs_1}{2} & \quad \obs_1 < 0 \\
        \frac{1 + 9 \obs_1}{2 + 8 \obs_1} & \quad \obs_1 \geq 0.
    \end{cases}
\end{equation}
It is simple to show that this is increasing in $\obs_1$; thus, the Bayes optimal decision rule based on $\Obs_1$ for $\cost$ is
\begin{equation}
    \label{eq:bayesopt_h1}
    \hyp^1_\cost(\obs) = \begin{cases}
        \II\{ \obs_1 \geq 2 \cost - 1 \} & \quad \cost \leq 1/2 \\
        \II\{ \obs_1 \geq \frac{2 \cost - 1}{9 - 8 \cost} \} & \quad \cost > 1/2.
    \end{cases}
\end{equation}
Now, let's analyze $\hypcls^2$. The posterior probability that $\Lbl = 1$ given that $\Obs_1 - 2 \epsilon \Obs_2 = b$ for $b \geq -2 \epsilon$ is
\begin{equation}
    \label{eq:post_h2}
    \PP(\Lbl = 1 \mid \Obs_1 - 2 \epsilon \Obs_2 = b) = \frac{1 + 9b + 16 \epsilon}{2 + 8b + 16 \epsilon}.
\end{equation}
This can also be shown to be increasing in $b$, so the Bayes optimal decision rule based on $\Obs_1 - 2 \epsilon \Obs_2$ for $\cost >= 1/2$ is
\begin{equation}
    \label{eq:bayesopt_h2}
    \hyp^2_\cost(\obs) = \II\left\{ \obs_1 - 2 \epsilon \obs_2 \geq \frac{2\cost - 1 - 16 \epsilon + 16 \cost \epsilon}{9 - 8 \cost} \right\}.
\end{equation}

For this proof, we consider two hypothesis class and loss parameter pairs: $\cost_1 = 1/2$ for $\hypcls^1$ and $\cost_2 = \frac{1 + 16 \epsilon}{2 + 16 \epsilon}$ for $\hypcls^2$. These correspond to the decision rules
\begin{align*}
    \hyp^1(\obs) & = \II\{\obs_1 \geq 0\}, \\
    \hyp^2(\obs) & = \II\{\obs_1 - 2 \epsilon \obs_2 \geq 0\} = \begin{cases}
        \obs_1 \geq 0 & \quad \obs_2 = 0 \\
        \obs_1 \geq 2 \epsilon & \quad \obs_2 = 1.
    \end{cases}
\end{align*}
It should be clear that these decision rules agree except when $\obs_2 = 1$ and $\obs_1 \in [0, 2 \epsilon)$.

Another important fact is that
\begin{equation}
    \label{eq:cost2_lower}
    \cost_2 = \frac{1 + 16 \epsilon}{2 + 16 \epsilon} = \frac{1}{2} + \frac{4 \epsilon}{1 + 8 \epsilon} \geq \frac{1}{2} + 2 \epsilon
\end{equation}
since $\epsilon \leq 1/8$.

We defer to the end of the proof to show that these hypotheses and distribution satisfy the conditions of Theorem \ref{thm:subopt_unknown}. 

\smallparagraph{Deriving the lower bound} Similarly to the proof of Theorem \ref{thm:optimal_lower}, we can bound the probability of an event $E$:
\begin{align*}
    \PP(\underbrace{\not\exists \obs_i \in \sample \quad \obs_{i, 1} \in [0, 2 \epsilon) \wedge \obs_{i, 2} = 1}_E)
    & = \left[ 1 - \PP(\Obs_1 \in [0, 2 \epsilon) \wedge \Obs_2 = 1) \right]^m \\
    & = (1 - 20 p_\cost \epsilon^2)^m \\
    & \geq \Big( e^{-40 p_\cost \epsilon^2} \Big)^m \\
    & = e^{-\log(\nicefrac{1}{2 \delta})} = 2 \delta.
\end{align*}

Conditional on $E$, the distributions of samples $\sample_1$ and $\sample_2$ for decision rules $\hyp^1$ and $\hyp^2$ are identical:
\begin{align*}
    p_1 = \PP(\hat\cost(\sample_1) \leq \nicefrac{1}{2} + \epsilon \mid E) & = \PP(\hat\cost(\sample_2) \leq \nicefrac{1}{2} + \epsilon \mid E), \\
    p_2 = \PP(\hat\cost(\sample_1) > \nicefrac{1}{2} + \epsilon \mid E) & = \PP(\hat\cost(\sample_2) > \nicefrac{1}{2} + \epsilon \mid E).
\end{align*}
Since $p_1 + p_2 = 1$, at least one of $p_1, p_2 \geq \nicefrac{1}{2}$. Suppose WLOG that $p_1 \geq \nicefrac{1}{2}$. Then
\begin{align*}
    \PP(|\hat\cost(\sample_2) - \cost_2| \geq \epsilon)
    & \overset{\text{(i)}}{\geq} \PP(\hat\cost(\sample_2) \leq \nicefrac{1}{2} + \epsilon) \\
    & = \PP(\hat\cost(\sample_2) \leq \nicefrac{1}{2} + \epsilon \mid E) \, \PP(E) \\
    & \geq \nicefrac{1}{2} (2 \delta) = \delta.
\end{align*}
(i) uses the fact shown earlier in (\ref{eq:cost2_lower}). Thus, there is a decision problem $(\dist, \cost_2)$ for which the IDT algorithm $\hat\cost$ must make an error of at least size $\epsilon$ with at least probability $\delta$. This concludes the main proof.

\smallparagraph{Verifying the requirements of Theorem \ref{thm:subopt_unknown}} First, we need to show that $\post_{\hypcls^1}(\Obs)$ has density at least $p_\cost$ on $[\cost_1 - \epsilon, \cost_1 + \epsilon] = [1/2 - \epsilon, 1/2 + \epsilon]$. From (\ref{eq:post_h1}) and (\ref{eq:bayesopt_h1}), it is clear that
\begin{equation*}
    \post_{\hypcls^1}(\obs) = g_1(\obs_1) = \begin{cases}
        \frac{1 + \obs_1}{2} & \quad \obs_1 < 0 \\
        \frac{1 + 9 \obs_1}{2 + 8 \obs_1} & \quad \obs_1 \geq 0.
    \end{cases} 
\end{equation*}
We can write the density of $\post_{\hypcls^1}(\Obs)$ as the density of $\Obs_1$ multiplied by the derivative of the inverse of $g_1$:
\begin{align*}
    p(\obs_1) \frac{d}{d \cost} g_1^{-1}(\cost)
    & \geq \frac{5 p_\cost}{2} \frac{d}{d \cost} \begin{cases}
        2 \cost - 1 & \quad \cost \leq 1/2 \\
        \frac{2 \cost - 1}{9 - 8\cost} & \quad \cost > 1/2
    \end{cases} \\
    & = \frac{5 p_\cost}{2} \begin{cases}
        2 & \quad \cost \leq 1/2 \\
        \frac{10}{(9 - 8 \cost)^2} & \quad \cost > 1/2
    \end{cases} \\
    & \geq p_\cost.
\end{align*}

Next, we need to show that $\post_{\hypcls^2}(\Obs)$ has density at least $p_\cost$ on $[\cost_2 - \epsilon, \cost_2 + \epsilon] \subseteq [1/2, 1]]$. From (\ref{eq:post_h2}) and (\ref{eq:bayesopt_h2}), we know that
\begin{equation*}
    \post_{\hypcls^2}(\obs) = g_2(\obs_1 - 2 \epsilon \obs_2) = \frac{1 + 9 (\obs_1 - 2 \epsilon \obs_2) + 16 \epsilon}{2 + 8 (\obs_1 - 2 \epsilon \obs_2) + 16 \epsilon}.
\end{equation*}
Using the same method as for $\post_{\hypcls^1}(\Obs)$ and the fact that the density of $\Obs_1 - 2 \epsilon \Obs_2$ is at least the density of $\Obs_1$ (i.e., $\frac{5 p_\cost}{2}$), we have that the density of $\post_{\hypcls^2}(\Obs)$ is at least
\begin{align*}
    \frac{5 p_\cost}{2} \frac{d}{d \cost} g_2^{-1}(\cost)
    & = \frac{5 p_\cost}{2} \frac{d}{d \cost} \frac{2 \cost - 1 - 16 \epsilon + 16 \cost \epsilon}{9 - 8 \cost} \\
    & = \frac{5 p_\cost}{2} \frac{10 + 16 \epsilon}{(9 - 8 \cost)^2} \\
    & \geq \frac{5 p_\cost}{2} \frac{2}{5} = p_\cost.
\end{align*}
The only remaining condition of Theorem \ref{thm:subopt_unknown} to prove is MD-smoothness. Again, consider $\hypcls^1$ first:
\begin{align*}
    \MD(\hyp^1_{b_1}, \hypcls^2) & =
    \min_{b_2 \in [1/2, 3/4]} \PP\left( \hyp^1_{b_1}(\Obs) \neq \hyp^2_{b_2}(\Obs) \right) \\
    & = \min_{b_2 \in [1/2, 3/4]} \frac{5 p_\cost}{2} \left| b_1 - b_2 \right| + 5 p_\cost \left| b_1^2 - (b_2 + 2 \epsilon)^2 \right| \\
    & = \frac{5 p_\cost}{2} \left| b_1 - b_1 \right| + 5 p_\cost \left| b_1^2 - (b_1 + 2 \epsilon)^2 \right| \\
    & = 20 p_\cost | \epsilon (b_1 + \epsilon) |.
\end{align*}
From (\ref{eq:bayesopt_h1}), we know that $b_1 - b_1' \leq 10 (\cost_1 - \cost_1')$ where $b_1$ and $b_1'$ are the optimal thresholds for loss parameters $\cost_1$ and $\cost_1'$, respectively.
So we have that
\begin{align*}
    \MD(\hyp^1_{c_1'}, \hypcls^2) - \MD(\hyp^1_{c_1}, \hypcls^2) & = 20 p_\cost \epsilon (|b_1' + \epsilon| - |b_1 + \epsilon|) \\
    & \leq 20 p_\cost \epsilon | b_1' - b_1 | \\
    & \leq 200 p_\cost \epsilon | c_1' - c_1 |. \\
\end{align*}
Thus $\hyp^1$ and $\hypset$ are $\alpha$-MD-smooth with $\alpha = 200 p_\cost \epsilon$.

Similarly, for $\hypcls^2$,
\begin{align*}
    \MD(\hyp^2_{b_2}, \hypcls^1) & =
    \min_{b_1 \in [1/2, 3/4]} \PP\left( \hyp^1_{b_1}(\Obs) \neq \hyp^2_{b_2}(\Obs) \right) \\
    & = \min_{b_1 \in [1/2, 3/4]} \frac{5 p_\cost}{2} \left| b_1 - b_2 \right| + 5 p_\cost \left| b_1^2 - (b_2 + 2 \epsilon)^2 \right| \\
    & = \frac{5 p_\cost}{2} \left| b_2 - b_2 \right| + 5 p_\cost \left| b_2^2 - (b_2 + 2 \epsilon)^2 \right| \\
    & = 20 p_\cost | \epsilon (b_2 + \epsilon) |.
\end{align*}
So we have that
\begin{align*}
    \MD(\hyp^2_{c_2'}, \hypcls^1) - \MD(\hyp^1, \hypcls^1) & = 20 p_\cost \epsilon (|b_2' + \epsilon| - |b_2 + \epsilon|) \\
    & \leq 20 p_\cost \epsilon | b_2' - b_2 | \\
    & \leq 200 p_\cost \epsilon | c_2' - c_2 |, \\
\end{align*}
and thus $\hyp^2$ and $\hypset$ are also $200 p_\cost \epsilon$-MD-smooth.

\end{proof}

\begin{theorem}[Second lower bound for suboptimal decision maker]
\label{thm:suboptimal_lower_dim}
Let $d \geq 6$ such that $d \equiv 2 \pmod{4}$. Let $\epsilon \in (0, \frac{1}{64 \sqrt{d - 2}}]$ and $p_\cost \in (0, 1]$. Then for any IDT algorithm $\hat{\cost}(\cdot)$, there is a decision problem $(\dist, \cost)$, hypothesis class family $\hypset$, and hypothesis class $\hypcls \in \hypset$ satisfying the conditions of Theorem \ref{thm:subopt_unknown} with the above parameters such that
\begin{equation*}
    m < \Omega\left(\frac{\sqrt{d}}{p_\cost \epsilon}\right)
    \quad \text{implies that} \quad
    \PP(| \hat\cost(\sample) - \cost | \geq \epsilon) \geq \frac{1}{160}.
\end{equation*}
\end{theorem}

\begin{proof}
Specifically, let
\begin{equation*}
    m = \frac{\sqrt{d - 2}}{64 p_\cost \epsilon}.
\end{equation*}

\smallparagraph{Defining the distribution} Let $n = d - 2 \geq 1$; $n$ is divisible by four. First, we define a joint distribution $\dist$ over $\Obs \in \obsspace = \mathbb{R}^{n + 1}$ and $\Lbl \in \{0, 1\}$. Let $\Obs_j$ refer to the $j$th coordinate of the random vector $\Obs$ and let $\obs_{i j}$ refer to the $j$th coordinate of the $i$th sample $\obs_i$. Furthermore, let $\Obs_{1:n}$ refer to the first $n$ components of $\Obs$.

The distribution of $\Obs$ has support on $n$ line segments in $\mathbb{R}^{n + 1}$ and at the origin. In particular, it has density $p_\cost / n$ on each line segment from $(0, \hdots, \Obs_j = 1, \hdots, 0, 0)$ to $(0, \hdots, \Obs_j = 1, \hdots, 0, 1)$, where the density is with respect to the Lebesque measure on the line. There is additionally a point mass of probability $1 - p_\cost$ at the origin. Everywhere on the support of $\dist$,
\begin{equation*}
    \PP(\Lbl = 1 \mid \Obs_{1:n} = \obs_{1:n}, \Obs_{n+1} = \obs_{n+1}) = \obs_{n+1}.
\end{equation*}

\smallparagraph{Defining the family of hypothesis classes} Next, we define a family of hypothesis classes. Let $\sigma \in \{-1, 1\}^n$ and define
\begin{equation*}
    f^\sigma(\obs) = \obs_{n+1} - 8 \epsilon \sqrt{n} \sigma^\top \obs_{1:n}.
\end{equation*}
Then we define $2^n$ hypothesis classes, one for each value of $\sigma$:
\begin{align*}
    \hypcls^\sigma & \triangleq \left\{ \hyp(\obs) = \II \left\{ f^\sigma(\obs) \geq b \right\} \;\middle|\; b \in [1/4, 3/4] \right\}, \\
    \hypset & \triangleq \{ \hypcls^\sigma \mid \sigma \in \{0, 1\}^n \}.
\end{align*}
Now, we can derive the optimal decision rule in hypothesis class $\hypcls^\sigma$ for loss parameter $\cost$. Let $[f^\sigma(\Obs)]_{1/4}^{3/4} = \max(1/4, \min(3/4, f^\sigma(\Obs))$ denote the value $f^\sigma(\Obs)$ clamped to the interval $[1/4, 3/4]$. Then for $b \in (1/4, 3/4)$,
\begin{align*}
    \PP\left(\Lbl = 1 \mid [f^\sigma(\Obs)]_{1/4}^{3/4} = b\right)
    & = \PP\left(\Lbl = 1 \mid \Obs_{n+1} - 8 \epsilon \sqrt{n} \sigma^\top \Obs_{1:n} = b\right) \\
    & = \frac{1}{n} \sum_{j=1}^n \PP\left(\Lbl = 1 \mid \Obs_j = 1 \wedge \Obs_{n+1} = b + 8 \epsilon \sqrt{n} \sigma_j\right) \\
    & = b + 8 \epsilon \sqrt{n} \frac{\mathbf{1}^\top \sigma}{n}.
\end{align*}
where $\mathbf{1}$ is the all-ones vector. Thus, the Bayes optimal decision rule based on $[f^\sigma(\Obs)]_{1/4}^{3/4}$ is
\begin{align*}
    \hyp^\sigma_c(\obs) & = \II\left\{f^\sigma(\obs) + 8 \epsilon \sqrt{n} \frac{\mathbf{1}^\top \sigma}{n} \geq c\right\} \\
    & = \II\left\{f^\sigma(\obs) \geq c - 8 \epsilon \sqrt{n} \frac{\mathbf{1}^\top \sigma}{n}\right\}
\end{align*}
for $c - 8 \epsilon \sqrt{n} \frac{\mathbf{1}^\top \sigma}{n} \in (1/4, 3/4)$. The induced posterior probability for $\hypcls^\sigma$ is
\begin{equation*}
    \post_{\hypcls^\sigma}(\obs) = f^\sigma(\obs) + 8 \epsilon \sqrt{n} \frac{\mathbf{1}^\top \sigma}{n}.
\end{equation*}

We consider one hypothesis from each hypothesis class $\hypcls^\sigma \in \hypset$. Specifically, we consider the optimal decision rule for
\begin{equation*}
    \cost^\sigma = \frac{1}{2} + 8 \epsilon \sqrt{n} \frac{\mathbf{1}^\top \sigma}{n},
\end{equation*}
which, as shown above is,
\begin{equation}
    \label{eq:hyp_def}
    \hyp^\sigma(\obs) = \II\left\{f^\sigma(\obs) \geq \frac{1}{2} \right\}.
\end{equation}
We leave until the end of the proof to show that each of these decision rules $\hyp_\sigma$ for $\sigma \in \{-1, 1\}^n$ satisfies the requirements of Theorem \ref{thm:subopt_unknown}.

\smallparagraph{Deriving the lower bound}
Now, we are ready to derive the lower bound that there is some $\hyp^\sigma$ such that $\PP(|\hat\cost(\sample) - \cost| \geq \epsilon) \geq \frac{1}{80}$. First, we can rewrite $\hyp^\sigma$ from (\ref{eq:hyp_def}) as
\begin{align*}
    \hyp^\sigma((0, \obs_j = 1, 0, \obs_{n+1}))
    & = \II\{\obs_{n+1} - 8 \epsilon \sqrt{n} \sigma_j \geq 1/2\} \\
    & = \II\{\obs_{n+1} \geq 1/2 + 8 \epsilon \sqrt{n} \sigma_j\}.
\end{align*}
Thus, only decisions made on points where $\obs_{n+1} \in [1/2 - 8 \epsilon \sqrt{n}, 1/2 + 8 \epsilon \sqrt{n}]$ are dependent on $\sigma_j$. Denote by $E_j$ the event that there is an observed sample that depends on $\sigma_j$:
\begin{equation*}
    E_j \quad \triangleq \quad \exists \obs_i \in \sample \text{ such that } \obs_{i j} = 1 \wedge \obs_{i, n+1} \in [1/2 - 8 \epsilon \sqrt{n}, 1/2 + 8 \epsilon \sqrt{n}].
\end{equation*}

Suppose we let $\sigma_j$ be independently Rademacher distributed, i.e. we assign equal probability $1 / 2^n$ to each $\sigma \in \{-1, 1\}$. Then if $E_j$ does not occur, the sample of decisions $\sample$ is independent from $\sigma_j$, i.e.
\begin{equation*}
    \sample \indep \sigma_j \mid \neg E_j.
\end{equation*}

Now let $F$ denote the event that more than $n / 2$ of the $E_j$ events occur:
\begin{equation*}
    F \quad \triangleq \quad | \{ j \in 1, \hdots, n \mid E_j \} | > n / 2.
\end{equation*}
We will start by proving a lower bound on $\PP(| \hat\cost(\sample) - \cost^\sigma| \geq \epsilon \mid \neg F)$. If $F$ does not occur, then at least half of the $E_j$ do not occur. Thus at least half of the elements of $\sigma$ are independent from the sample $\sample$. Let $I$ be the set of indices $j$ for which $E_j$ does not occur; thus, $\sigma_I \indep \sample$, and given $\neg F$, $|I| \geq n / 2$.

We can decompose $\cost^\sigma$ into part that depends on $\sigma_I$ and part that depends on $\sigma_{I^C}$:
\begin{equation}
    \label{eq:c_decomp}
    \cost^\sigma = \frac{1}{2} + 8 \epsilon \sqrt{n} \frac{\mathbf{1}^\top \sigma_I}{n} + 8 \epsilon \sqrt{n} \frac{\mathbf{1}^\top \sigma_{I^C}}{n}.
\end{equation}
Note that for each $j \in I$, $\frac{\sigma_j + 1}{2}$ is $\nicefrac{1}{2}$-Bernoulli distributed. Thus
\begin{equation*}
    Z = \frac{\mathbf{1}^\top \sigma_I + |I|}{2}
    = \sum_{j \in I} \frac{\sigma_j + 1}{2}
    \sim \text{Binom}\left(|I|, \frac{1}{2}\right).
\end{equation*}
We can establish lower bounds on the tails of this given that $F$ occurs:
\begin{equation*}
    \PP\left(Z - \frac{|I|}{2} \geq t \mid \neg F\right) = \left(Z - \frac{|I|}{2} \leq -t \mid \neg F\right)
    \geq \frac{1}{15} e^{-32 t^2 / n}.
\end{equation*}
This lower bound is from \citet{matousek_probablistic_2008}. Plugging in $t = \frac{1}{8}\sqrt{n}$, we obtain
\begin{align}
    \PP\left(Z - \frac{|I|}{2} \geq \frac{1}{8}\sqrt{n} \mid \neg F\right) = \left(Z - \frac{|I|}{2} \leq -\frac{1}{8}\sqrt{n} \mid \neg F\right)
    & \geq \frac{1}{20} \nonumber \\
    \PP\left(\mathbf{1}^\top \sigma_I \geq \frac{1}{4}\sqrt{n} \mid \neg F\right) = \left(\mathbf{1}^\top \sigma_I \leq -\frac{1}{4}\sqrt{n} \mid \neg F\right)
    & \geq \frac{1}{20}. \label{eq:sigma_tail}
\end{align}
Given $\sample$, $\sigma_{I^C}$ is completely known (since $E_j$ occurs for each $j \in I^C$, revealing $\sigma_j$). So plugging (\ref{eq:sigma_tail}) into (\ref{eq:c_decomp}) gives
\begin{align*}
    \PP\left(\cost^\sigma - \frac{1}{2} - 8 \epsilon \sqrt{n} \frac{\mathbf{1}^\top \sigma_{I^C}}{n} \geq 2 \epsilon \mid \neg F, \sample \right)
    = \PP\left(\cost^\sigma - \frac{1}{2} - 8 \epsilon \sqrt{n} \frac{\mathbf{1}^\top \sigma_{I^C}}{n} \leq -2 \epsilon \mid \neg F, \sample \right)
    & \geq \frac{1}{20} \\
    \PP\left(\cost^\sigma - \cost^{\sigma_{I^C}} \geq 2 \epsilon \mid \neg F, \sample \right)
    = \PP\left(\cost^\sigma - \cost^{\sigma_{I^C}} \leq -2 \epsilon \mid \neg F, \sample \right)
    & \geq \frac{1}{20}. \\
\end{align*}
That is, there is at least probability $\nicefrac{1}{20}$ that $\cost^\sigma$ is more than $2 \epsilon$ above and below $\cost^{\sigma_{I^C}}$, given $\neg F$ and the observed sample $\sample$.

This is enough to show that $\PP(|\hat\cost(\sample) - \cost^\sigma| \geq \epsilon \mid \neg F, \sample) \geq \frac{1}{40}$. First, observe that
\begin{equation*}
    \PP(\hat\cost(\sample) \geq \cost^{\sigma_{I^C}} \mid \neg F, \sample)
    + \PP(\hat\cost(\sample) < \cost^{\sigma_{I^C}} \mid \neg F, \sample) = 1,
\end{equation*}
so one of these probabilities must be at least $\nicefrac{1}{2}$. Say WLOG that it is the first. Then
\begin{align*}
    & \PP(|\hat\cost(\sample) - \cost^\sigma| \geq \epsilon \mid \neg F, \sample) \\
    & \quad \geq \PP(\cost^\sigma - \cost^{\sigma_{I^C}} \leq -2 \epsilon \wedge \hat\cost(\sample) \geq \cost^{\sigma_{I^C}} \mid \neg F, \sample) \\
    & \quad \overset{(i)}{=} \PP(\cost^\sigma - \cost^{\sigma_{I^C}} \leq -2 \epsilon \mid \neg F, \sample) \, \PP( \hat\cost(\sample) \geq \cost^{\sigma_{I^C}} \mid \neg F, \sample) \\
    & \quad \geq \left(\frac{1}{20}\right)\left(\frac{1}{2}\right) = \frac{1}{40}.
\end{align*}
Here, (i) makes use of the fact that $\sample \indep \sigma_I \mid \neg F$.
Given this, we can finally derive the lower bound on the unconditional probability that $\PP(|\hat\cost(\sample) - \cost^\sigma| \geq \epsilon)$:
\begin{align}
    & \PP(|\hat\cost(\sample) - \cost^\sigma| \geq \epsilon) \nonumber \\
    & \quad = \PP(|\hat\cost(\sample) - \cost^\sigma| \geq \epsilon \mid F) \PP(F) + \PP(|\hat\cost(\sample) - \cost^\sigma| \geq \epsilon \mid \neg F) \PP(\neg F) \nonumber \\
    & \quad \geq \PP(|\hat\cost(\sample) - \cost^\sigma| \geq \epsilon \mid \neg F) \PP(\neg F) \nonumber \\
    & \quad \geq \frac{\PP(\neg F)}{40}. \label{eq:f_cond_lower_bound}
\end{align}
So we need to derive a lower bound on $\PP(\neg F)$. We can do so by noting that in order for $F$ to occur, there must be at least $n / 2$ samples $x_i$ with $x_{i,n+1} \in [1/2 - 8 \epsilon \sqrt{n}, 1/2 + 8 \epsilon \sqrt{n}]$. The probability of this event for a particular sample is
\begin{equation*}
    \PP\Big(X_{n+1} \in [1/2 - 8 \epsilon \sqrt{n}, 1/2 + 8 \epsilon \sqrt{n}]\Big) = 16 p_\cost \epsilon \sqrt{n}.
\end{equation*}
So at least $n / 2$ of the $m$ samples must have the event with probability $16 p_\cost \epsilon \sqrt{n}$ occur for $F$ to occur. Let $\text{GE}(p, m, r)$ denote the probability of at least $r$ successes of probability $p$ in $m$ independent trials. Then there is the following fact from probability theory \citep{angluin_fast_1979}:
\begin{equation*}
    \text{GE}(p, m, (1 + \gamma) m p) \leq e^{- \gamma^2 m p / 3}.
\end{equation*}
Then
\begin{align*}
    \PP(F) & \leq \text{GE}(16 p_\cost \epsilon \sqrt{n}, m, n/2) \\
    & = \text{GE}\left(16 p_\cost \epsilon \sqrt{n}, \frac{\sqrt{n}}{64 p_\cost \epsilon}, 2 \left(\frac{\sqrt{n}}{64 p_\cost \epsilon}\right)\left(16 p_\cost \epsilon \sqrt{n}\right)\right) \\
    & \leq e^{-n/12} \leq 3/4
\end{align*}
as long as $n \geq 4$ as assumed. Thus $\PP(\neg F) > 1/4$. So putting this together with (\ref{eq:f_cond_lower_bound}), we have
\begin{equation*}
    \PP(|\hat\cost(\sample) - \cost^\sigma| \geq \epsilon) \geq \frac{1}{160}.
\end{equation*}
This equation is given with respect to the uniform distribution over $\sigma$. But there also must be a particular $\sigma$ and thus corresponding $\hyp^\sigma \in \hypcls^\sigma$ which has the same tails on $\hat\cost(\sample) - \cost$. Thus we conclude the proof.

\smallparagraph{Verifying the requirements of Theorem \ref{thm:subopt_unknown}} Now we show that the distribution and hypothesis class family satisfy the conditions of Theorem \ref{thm:subopt_unknown}. First, note that all $\hyp \in \hypcls \in \hypset$ are thresholds on linear functions of the observation $\obs$. Thus, $\cup_{\hypcls \in \hypset} \hypcls$ is a subset of the halfspaces in $\mathbb{R}^{n+1}$ and so it has VC-dimension at most $n + 2 = d$.

Next, it is clear that for $\rho \leq \epsilon$,
\begin{align*}
    \PP(\post_{\hypcls^\sigma}(\Obs) \in (c, c + \rho])
    & = \PP\left(f^\sigma(\Obs) + 8 \epsilon \sqrt{n} \frac{\mathbf{1}^\top \sigma}{n} \in (c, c + \rho] \right) \\
    & = \sum_{j=1}^n \PP\left(X_j = 1 \wedge X_{n+1} - 8 \epsilon \sqrt{n} \sigma_j + 8 \epsilon \sqrt{n} \frac{\mathbf{1}^\top \sigma}{n} \in (c, c + \rho] \right) \\
    & = \sum_{j=1}^n \frac{p_\cost \rho}{n} = p_\cost \rho.
\end{align*}
A similar result can be shown for $\PP(\post_{\hypcls^\sigma}(\Obs) \in [c - \rho, c))$.

Finally, we need to show that MD-smoothness holds. Take any $\hyp^\sigma$ and any $\hypcls^{\tilde\sigma}$. Then the disagreement between $\hyp^\sigma$ and a hypothesis in $\hypcls^{\tilde\sigma}$ with threshold $b$ is
\begin{align*}
    \PP(\hyp^\sigma(\Obs) \neq \hyp^{\tilde\sigma}_b(\Obs)
    & = \frac{p_\cost}{n} \sum_{j=1}^n \left| \frac{1}{2} + 8 \epsilon \sqrt{n} \sigma_j - b - 8 \epsilon \sqrt{n} \tilde\sigma_j \right| \\
    & = \frac{p_\cost}{n} \sum_{j=1}^n \left| \left(\frac{1}{2} + 8 \epsilon \sqrt{n} ( \sigma_j - \tilde\sigma_j)\right) - b  \right|.
\end{align*}
This is minimized when $b$ is the median of $\left(\frac{1}{2} + 8 \epsilon \sqrt{n} ( \sigma_j - \tilde\sigma_j)\right)$ for $j = 1, \hdots, n$. Thus $b \in [\frac{1}{2} - 8 \epsilon \sqrt{n}, \frac{1}{2} + 8 \epsilon \sqrt{n}]$; since $\epsilon \leq \frac{1}{64 \sqrt{n}}$, this implies $b \in [3/8, 5/8]$. Suppose now we let $c' \in [c^\sigma - 1/8, c^\sigma + 1/8]$. Then we can let $b' = b + (c' - c^\sigma)$ and
\begin{align*}
    \MD(\hyp^\sigma_{\cost'}, \hypcls^{\tilde\sigma}) \leq \PP\Big(\hyp^\sigma_{\cost'}(\Obs) \neq \hyp^{\tilde\sigma}_{b'}(\Obs)\Big)
    = \PP\Big(\hyp^\sigma(\Obs) \neq \hyp^{\tilde\sigma}_{b}(\Obs)\Big) = \MD(\hyp^\sigma, \hypcls^{\tilde\sigma}).
\end{align*}
Thus for $|c' - c^\sigma| \leq 1/8$, $h^\sigma$ and $\hypset$ are 0-MD-smooth. If $|c' - c^\sigma| > 1/8$, then we have
\begin{align*}
    \MD(\hyp^\sigma_{\cost'}, \hyp^{\tilde\sigma}) \leq 1 < \frac{8}{\MD(\hyp^\sigma, \hypcls^{\tilde\sigma})} |\cost' - \cost^\sigma| \MD(\hyp^\sigma, \hypcls^{\tilde\sigma}).
\end{align*}
Thus overall $\hyp^\sigma$ and $\hypset$ are $\alpha$-MD-smooth with
\begin{equation*}
    \alpha = \max_{\tilde\sigma \neq \sigma} \frac{8}{\MD(\hyp^\sigma, \hypcls^{\tilde\sigma})}.
\end{equation*}

\end{proof}

\emph{Bibliographic note:} we establish dependence on the VC dimension $d$ in Theorem \ref{thm:suboptimal_lower_dim} using a technique similar to that used by \citet{ehrenfeucht_general_1989}.

\subsection{Necessity of MD-smoothness}
\label{sec:md_smooth_counterexample}

The lower bounds given in Section \ref{sec:subopt_unknown_lower} do not depend on the $\alpha$ parameter from the MD-smoothness assumption made in Theorem \ref{sec:suboptimal_unknown}; thus, one may wonder if this assumption is necessary. In the following lemma, we show that it is necessary in some cases by giving an example of an IDT problem where a lack of MD-smoothness precludes identifiability of the loss parameter.

\begin{lemma}[No MD-smoothness can prevent identifiablity]
\label{lemma:no_md_smooth}
Let $\epsilon \in (0, 1/10)$. Then for any IDT algorithm $\hat{\cost}(\cdot)$, there is a decision problem $(\dist, \cost)$, hypothesis class family $\hypset$, and hypothesis class $\hypcls \in \hypset$ satisfying the conditions of Theorem \ref{thm:subopt_unknown} \emph{except} for MD-smoothness such that
\begin{equation*}
    \PP(| \hat\cost(\sample) - \cost | \geq \epsilon) \geq \frac{1}{2}
\end{equation*}
for a sample $\sample$ of any size $m$.
\end{lemma}

\begin{figure}
    \centering
    \input{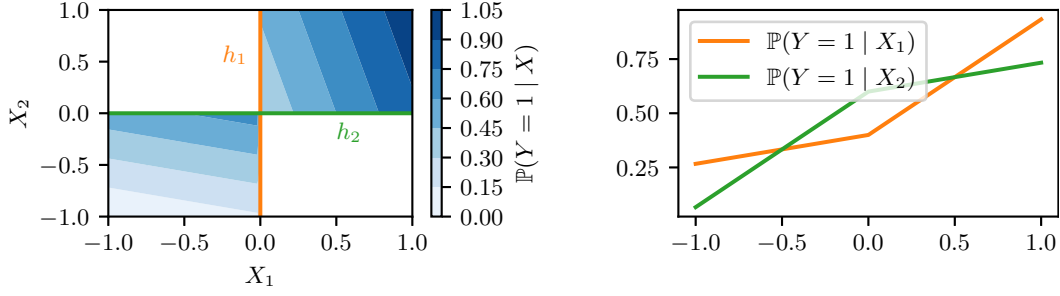}
    \caption{A visualization of the distribution and decision rules used in Lemma \ref{lemma:no_md_smooth} to show that a lack of MD-smoothness can prevent identifiability of the loss parameter $\cost$. On the left, the distribution over $\Obs = (\Obs_1, \Obs_2)$ and $\Lbl$ is shown; $\Obs$ has constant density on unit squares in the first and third quadrants, and $\PP(\Lbl = 1 \mid \Obs)$ varies as shown with the heatmap. We consider two decision rules $\hyp_1$ and $\hyp_2$ which are optimal thresholds of $\Obs_1$ and $\Obs_2$, respectively, for loss parameters $\cost_1 = 2/5$ and $\cost_2 = 3/5$, respectively. Since $\cost_1 \neq \cost_2$ but $\PP(\hyp_1(\Obs) = \hyp_2(\Obs)) = 1$, it is impossible to identify $\cost$ reliably. This is because the distribution and decision rules are not MD-smooth, since shifting either decision rule slightly causes a jump in minimum disagreement with the other hypothesis class from $0$ to a positive value.}
    \label{fig:no_md_smooth}
\end{figure}

\begin{proof}
\smallparagraph{Defining the distribution} First, we define a distribution $\dist$ over $\Obs \in \obsspace = \mathbb{R}^2$ and $\Lbl \in \{0, 1\}$. $\dist_\Obs$ has density $\nicefrac{1}{2}$ on two squares $[-1, 0] \times [-1, 0]$ and $[0, 1] \times [0, 1]$, and the distribution of $\Lbl \mid \Obs$ is defined as follows:
\begin{equation*}
    \PP(\Lbl = 1 \mid \Obs = \obs = \begin{cases}
        \frac{2}{3} + \frac{2}{15} \obs_1 + \frac{8}{15} \obs_2 & \quad \obs \in [-1, 0] \times [-1, 0] \\
        \frac{1}{3} + \frac{8}{15} \obs_1 + \frac{2}{15} \obs_2 & \quad \obs \in [0, 1] \times [0, 1].
    \end{cases}
\end{equation*}

\smallparagraph{Defining the family of hypothesis classes} We consider the two hypothesis classes which are thresholds on one component of the observation $\obs$:
\begin{align*}
    \hypcls_1 & = \{\hyp(\obs) = \II\{\obs_1 \geq b\} \mid b \in [-1, 1]\}, \\
    \hypcls_2 & = \{\hyp(\obs) = \II\{\obs_2 \geq b\} \mid b \in [-1, 1] \}.
\end{align*}
That is, $\hypset = \{\hypcls_1, \hypcls_2\}$. The conditional probabilities for $\Lbl = 1$ given just one of the observation components are
\begin{equation}
    \label{eq:no_md_smooth_post}
    \begin{aligned}
        \post_{\hypcls_1}(\obs) = \PP(\Lbl = 1 \mid \Obs_1 = \obs_1) & = \frac{2}{5} + \frac{2}{15} \obs_1 + \frac{2}{5} \obs_1 \II\{\obs_1 \geq 0\}, \\
        \post_{\hypcls_2}(\obs) = \PP(\Lbl = 1 \mid \Obs_2 = \obs_2) & = \frac{3}{5} + \frac{2}{15} \obs_2 + \frac{2}{5} \obs_2 \II\{\obs_2 \leq 0\}. \\
    \end{aligned}
\end{equation}
We consider the optimal decision rules for $\cost_1 = 2/5$ and $\cost_2 = 3/5$ in $\hypcls_1$ and $\hypcls_2$, respectively, which from the above can be calculated as
\begin{align*}
    \hyp_1(\obs) & = \II\{\obs_1 \geq 0\}, \\
    \hyp_2(\obs) & = \II\{\obs_2 \geq 0\}. \\
\end{align*}
The distribution and decision rules are visualized in Figure \ref{fig:no_md_smooth}.

\smallparagraph{Lack of identifiability} Note that since $\Obs$ only has support where $\sign(\Obs_1) = \sign(\Obs_2)$, the above decision rules are indistinguishable. Thus, we use the same techniques from Corollary \ref{corollary:certain_lower} and Lemma \ref{lemma:lower_alternative_subopt} to show that for at least one of $\cost \in \{\cost_1, \cost_2\}$
\begin{equation*}
    \PP(|\hat\cost(\sample) - \cost| \geq \nicefrac{1}{2}(\cost_2 - \cost_1) = \nicefrac{1}{10} \geq \epsilon) \geq \nicefrac{1}{2}.
\end{equation*}

\smallparagraph{Hypothesis classes are not MD-smooth} Although this is not required for the proof of the lemma, we will demonstrate that the defined hypothesis classes are not $\alpha$-MD-smooth for any $\alpha$. By way of contradiction, assume that there is some $\alpha$ such that $\hyp_1$ and $\hypset$ are MD-smooth. Then for any $c_1' \in [0, 1]$,
\begin{equation*}
    \MD(\hyp_{c_1'}, \hypcls_2) \leq (1 + \alpha |c_1' - c_1|) \MD(\hyp_1, \hypcls_2) = 0.
\end{equation*}
Here, $\MD(\hyp_1, \hypcls_2)$ since $\PP(\hyp_1(\Obs) \neq \hyp_2(\Obs)) = 0$, i.e. $\hyp_1$ and $\hyp_2$ do not disagree at all. However, there are clearly values of $\cost_1'$ such that $\MD(\hyp_{\cost_1'}, \hypcls_2) > 0$, so we have a contradiction.

\smallparagraph{Verifying the other requirements of Theorem \ref{thm:subopt_unknown}} Clearly, the family of hypothesis classes defined above have finite VC-dimension.

The densities of $\post_{\hypcls_1}(\Obs)$ and $\post_{\hypcls_2}(\Obs)$ can be calculated as the density of $\Obs_1$ or $\Obs_2$ multiplied by the derivative of the inverse of the posterior probability functions. The densities of $\Obs_1$ and $\Obs_2$ are both $\nicefrac{1}{2}$ on the interval $[-1, 1]$, and the derivative of the inverse of the equations in (\ref{eq:no_md_smooth_post}) is at least $\nicefrac{15}{8}$. So the distribution satisfies the requirements of Theorem \ref{thm:subopt_unknown} other than MD-smoothness with $p_\cost \geq \nicefrac{15}{16}$.

\end{proof}

\section{Feature Subset Hypothesis Class Family}
\label{sec:feat_hypset}

In this section, we work through the application of Theorem \ref{thm:subopt_unknown} to a practical example. Theorem \ref{thm:subopt_unknown} concerns the case of IDT when the decision maker could be restricting themselves to any suboptimal hypothesis class $\hypcls \in \hypset$ for some family of hypothesis classes $\hypset$. In this example, we consider $\hypset_\text{feat}$ as defined in (\ref{eq:feat_subset_fam}) and repeated here:
\begin{equation}
    \hypset_\text{feat} \triangleq \left\{ \hypcls_S \mid S \subseteq \{1, \hdots, n\} \right\} \quad \text{where} \quad \hypcls_S \triangleq \left\{ \hyp(\obs) = f(\obs_{S}) \mid f: \mathbb{R}^{|S|} \to \{0, 1\} \right\}.
    \tag{\ref{eq:feat_subset_fam}}
\end{equation}

This family can model decision makers that have bounded computational capacity and may only be able to reason based on a few features of the data. An application of structural risk minimization \citep{vapnik_principles_1991} from learning theory shows that the sample complexity of IDT in this case may scale only linearly in the number of features considered and logarithmically in the total feature count:

\begin{lemma}
\label{lemma:subset_sample_complexity}
Let a decision maker use a hypothesis class $\hypcls_S \in \hypset_\text{feat}$ as defined in (\ref{eq:feat_subset_fam}) which consists of decision rules depending only on the subset of the features in $S$. Let $s = |S|$ be the number of such features; neither $s$ nor $S$ is known. Suppose $\obsspace = \mathbb{R}^d$, i.e. $d$ is the total number of features. Let assumptions on $\epsilon$, $\delta$, $\alpha$, and $p_\cost$ be as in Theorem \ref{thm:subopt_unknown}.

Let $\hat\hyp_{\hat\cost} \in \argmin_{\hat\hyp \in \hypcls_{\hat{S}}} \risk_{\hat\cost}(\hat\hyp)$ be chosen to be consistent with the observed decisions, i.e. $\hat\hyp_{\hat\cost}(x_i) = \dec_i$, and such that $| \hat{S} |$ is as small as possible. Then $| \hat\cost - \cost | \leq \epsilon$ with probability at least $1 - \delta$ as long as the number of samples $m$ satisfies
\begin{equation*}
    m \geq O\left[ \left(\frac{\alpha}{\epsilon} + \frac{1}{\epsilon^2} \right) \left(\frac{s \log d + \log(1 / \delta)}{p_\cost}\right) \right].
\end{equation*}
\end{lemma}
\begin{proof}
We prove Lemma \ref{lemma:subset_sample_complexity} by bounding the VC-dimension of the union of all optimal decision rules in all $\hypcls_S \in \hypset_\text{feat}$ where $|S| \leq s$. An optimal decision rule for loss parameter $\cost$ in $\hypcls_S$ is given by the Bayes optimal classifier:
\begin{equation*}
    \hyp^S_\cost(\obs) = \II \{ \PP(\Lbl = 1 \mid \Obs_S = \obs_s) \geq c \}.
\end{equation*}
Now consider a set of observations $\obs_1, \hdots, \obs_d \in \obsspace$. We will show that for $d > 1 + 2 s \log_2 (n + 1)$, this set cannot be shattered by $d$. To see why, note that decision rules in any particular class $\hypcls_S$ threshold the posterior probability $\PP(\Lbl = 1 \mid \Obs_S = \obs_s)$. Thus, each hypothesis class can only produce $d + 1$ distinct labelings of the set of observations. The number of hypothesis classes $\hypcls_S$ with $|S| \leq s$ is
\begin{equation*}
    \sum_{k = 0}^s \begin{pmatrix} n \\ s \end{pmatrix} \leq \sum_{k=0}^s n^k \leq (n + 1)^s.
\end{equation*}
So the number of distinct labelings assigned by hypotheses in $\hypset$ to the observations must be at most $(d+1) (n + 1)^s < 2^d$ if $d > 1 + 2 s \log_2 (n+1)$. Thus this set cannot be shattered, so
\begin{equation*}
    \text{VCdim}\left(\cup_{|S| \leq s} \hypcls_S\right) \leq 1 + 2 s \log_2 (n + 1) = O(s \log n).
\end{equation*}
Applying Theorem \ref{thm:subopt_unknown} with $d = O(s \log n)$ completes the proof.
\end{proof}

The following lemma states conditions under which $\alpha$-MD-smoothness holds for $\hypset_\text{feat}$.

\begin{lemma}
\label{lemma:subset_md_smooth}
Let $\hypset_\text{feat}$ and $\hypcls_S$ be defined as in (\ref{eq:feat_subset_fam}). Let $h \in \hypcls_S$. Suppose that there is a $\zeta > 0$ such that for any $\hat{S} \subseteq \{1, \hdots, n\}$, one of the following holds: either (a) $\PP(\Lbl = 1 \mid \Obs = \obs_S) = \PP(\Lbl = 1 \mid \Obs = \obs_{\hat{S}})$ for all $\obs \in \mathbb{R}^d$, or (b) $\MD(h, \hypcls_{\hat{S}}) \geq \zeta$. Furthermore, suppose that the distribution of $\post_{\hypcls_S}(\Obs)$ is absolutely continuous with respect to the Lebesque measure and that its density is bounded above by $M < \infty$. Then $h$ and $\hypset_\text{feat}$ are $\alpha$-MD-smooth with $\alpha = M / \zeta$.
\end{lemma}

Since $\alpha$-MD-smoothness is a sufficient condition for identification of the loss function parameter $\cost$, Lemma \ref{lemma:subset_md_smooth} gives conditions under which IDT can be performed. The main requirement is that considering different subsets of the features either gives identical decision rules (case (a)) or decision rules which disagree by some minimum amount (case (b)). If decision rules using a different subset of the features can be arbitrarily close to the true one, it may not be possible to apply IDT.

\begin{proof}
Consider any $\hat{S} \subseteq \{1, \hdots, n\}$. If (a) holds for $\hat{S}$, then $\hyp^S_\cost(\obs) = \hyp^{\hat{S}}_\cost(\obs)$ for any $\cost \in [0, 1]$ and $\obs \in \obsspace$. Thus
\begin{align*}
    \MD(\hyp^S_{\cost'}, \hypcls_{\hat{S}}) = 0 \leq (1 + \alpha |\cost' - \cost|) \MD(\hyp^S_\cost, \hypcls_{\hat{S}}) = 0
\end{align*}
so $\alpha$-MD-smoothness holds in this case for any $\alpha$.

If (b) holds, then let $\hat\hyp \in \argmin_{\hat\hyp \in \hypcls_{\hat{S}}} \PP(\hyp(\Obs) \neq \hat\hyp(\Obs))$. Let $\cost' \in [0, 1]$; without loss of generality, we may assume that $\cost' > \cost$. Denote $\post_S(\obs) = \PP(\Lbl = 1 \mid \Obs_S = \obs_s)$. Then
\begin{align*}
    & \MD(\hyp^S_{\cost'}, \hypcls_{\hat{S}}) \\
    & \quad \leq \PP(\hyp^S_{\cost'}(\Obs) \neq \hat\hyp(\Obs)) \\
    & \quad = \PP\Big(\post_S(\Obs) < \cost' \wedge \hat\hyp(\Obs) = 1)\Big) + \PP\Big(\post_S(\Obs) > \cost' \wedge \hat\hyp(\Obs) = 0)\Big) \\
    & \quad \leq \PP\Big(\post_S(\Obs) \in [\cost, \cost') \wedge \hat\hyp(\Obs) = 1\Big) 
    + \PP\Big(\post_S(\Obs) < \cost \wedge \hat\hyp(\Obs) = 1\Big)
    + \PP\Big(\post_S(\Obs) > \cost \wedge \hat\hyp(\Obs) = 0)\Big) \\
    & \quad = \PP\Big(\post_S(\Obs) \in [\cost, \cost') \wedge \hat\hyp(\Obs) = 1\Big) + \MD(\hyp, \hypcls_{\hat{S}}) \\
    & \quad \leq M (\cost' - \cost) + \MD(\hyp, \hypcls_{\hat{S}}) \\
    & \quad \leq \left[ 1 + \frac{M}{\zeta} (\cost' - \cost) \right] \MD(\hyp, \hypcls_{\hat{S}}).
\end{align*}
So $\hyp$ and $\hypset$ satisfy $\alpha$-MD-smoothness with $\alpha = M/\zeta$.
\end{proof}

\section{Surrogate Loss Functions}
\label{sec:surrogate}

Here, we explore using IDT when the decision maker minimizes a surrogate loss instead of the true loss. So far, as formulated in Section \ref{sec:problem}, we have assumed that the decision maker chooses a decision rule $\hyp$ which minimizes the expected loss $\EE[\loss_\cost(\hyp(\Obs), \Lbl)]$, where the loss function is defined as
\begin{align}
    \loss_\cost(\dec, \lbl) & = \begin{cases}
        0 & \quad \dec = \lbl \\
        \cost & \quad \dec = 1 \wedge \lbl = 0 \\
        1 - \cost & \quad \dec = 0 \wedge \lbl = 1
    \end{cases} \nonumber \\
    & = \begin{cases}
        \cost \, \II\{ \dec = 1 \} & \quad \lbl = 0 \\
        (1 - \cost) \, \II\{ \dec = 0 \} & \quad \lbl = 1.
    \end{cases} \label{eq:indicator_loss}
\end{align}
However, this loss function is not convex or continuous, so it is difficult to optimize. Thus, we might expect the decision maker to choose their decision rule using a \emph{surrogate loss} which is convex. In particular, suppose that the decision rule $\hyp(\cdot)$ is calculated by thresholding a function $f: \obsspace \to \mathbb{R}$:
\begin{equation*}
    \hyp(\obs) = \II \{ f(\obs) \geq 0 \}.
\end{equation*}
Then, we can replace the indicator functions in (\ref{eq:indicator_loss}) with a surrogate loss $V: \mathbb{R} \to \mathbb{R}$:
\begin{equation}
    \label{eq:surrogate_loss}
    \tilde{\loss}_\cost(w, \lbl) = \begin{cases}
        \cost \, V(w) & \quad \lbl = 0 \\
        (1 - \cost) \, V(-w) & \quad \lbl = 1.
    \end{cases}
\end{equation}
Say that the decision maker minimizes this loss $\tilde{\loss}_\cost$ instead of the true loss $\loss$:
\begin{equation}
    \label{eq:surrogate_loss_hyp}
    f^* \in \argmin_f \; \EE[ \tilde{\loss}_\cost(f(\Obs), \Lbl) ].
\end{equation}
The following lemma shows that, for reasonable surrogate losses, if the decision maker is optimal then minimizing the surrogate loss is equivalent to minimizing the true loss. The proof is adapted from Section 4.2 of \citet{rosasco_are_2004}; they show that the hinge loss, squared loss, and logistic loss all satisfy the necessary conditions.
\begin{lemma}
\label{lemma:surrogate_equivalent}
Suppose $V: \mathbb{R} \to \mathbb{R}$ is convex and that it is strictly increasing in a neighborhood of 0. Let $f^*$ be chosen as in (\ref{eq:surrogate_loss_hyp}), and let $\hyp(\obs) = \II \{ f^*(\obs) \geq 0 \}$. Then $\hyp \in \argmin_\hyp \EE[\ell_\cost(\hyp(\Obs), \Lbl)]$; that is, the threshold of $f^*$ is an optimal decision rule for the true cost function.
\end{lemma}
\begin{proof}
We prove the lemma by contradiction; assume that $\hyp$ is \emph{not} an optimal decision rule for the true loss function. Then by Lemma \ref{lemma:bayesopt},
\begin{equation*}
    \PP(\hyp(\Obs) \neq \II\{\post(\Obs) \geq \cost\} \wedge \post(\Obs) \neq \cost) > 0.
\end{equation*}
This implies that either
\begin{align*}
    \PP(\hyp(\Obs) = 0 \wedge \post(\Obs) > \cost) > 0 & \qquad \text{or} \qquad \PP(\hyp(\Obs) = 1 \wedge \post(\Obs) < \cost) > 0,
\end{align*}
or equivalently,
\begin{align}
    \label{eq:nonoptimal_implication}
    \PP(f^*(\Obs) < 0 \wedge \post(\Obs) > \cost) > 0 & \qquad \text{or} \qquad \PP(f^*(\Obs) \geq 0 \wedge \post(\Obs) < \cost) > 0.
\end{align}
Without loss of generality, assume the former. Define
\begin{equation*}
    \tilde{f}(\obs) = \begin{cases}
        0 & \quad f^*(\obs) < 0 \wedge \post(\obs) > \cost \\
        f^*(\obs) & \quad \text{otherwise}.
    \end{cases}
\end{equation*}
Consider any $\obs$ which satisfies $f^*(\obs) < 0$ and $\post(\obs) > \cost$. We can write
\begin{align*}
    & \EE\left[\tilde{\ell}_\cost(f^*(\Obs), \Lbl) - \tilde{\ell}_\cost(\tilde{f}(\Obs), \Lbl) \mid \Obs = \obs\right] \\
    & \quad = \PP(\Lbl = 0 \mid \Obs = \obs) \, \cost \, \left(V(f^*(\obs)) - V(\tilde{f}(\obs))\right) 
    + \PP(\Lbl = 1 \mid \Obs = \obs) \, (1 - \cost) \, \left(V(-f^*(\obs)) - V(-\tilde{f}(\obs))\right)  \\
    & \quad = (1 - \post(\obs)) \, \cost \, \left(V(f^*(\obs)) - V(0)\right)
    + \post(\obs) \, (1 - \cost) \, \left(V(-f^*(\obs)) - V(0)\right) \\
    & \quad = \tilde{\ell}_\cost(f^*(\obs) \mid \obs) - \tilde{\ell}_\cost(0 \mid \obs),
\end{align*}
where we define
\begin{equation*}
    \tilde{\ell}_\cost(w \mid \obs) = (1 - \post(\obs)) \, \cost \, V(w) + \post(\obs) \, (1 - c) \, V(-w).
\end{equation*}
$\tilde{\ell}_\cost(w \mid \obs)$ satisfies two properties:
\begin{enumerate}
    \item It is convex in $w$, since it is a sum of two convex functions.
    \item It is strictly decreasing in $w$ in a neighborhood of 0. To see why, note that we assumed $\post(\obs) > c$, so
    \begin{equation*}
        (1 - \post(\obs)) \, \cost < (1 - \cost) \, \cost < \post(\obs) \, (1 - c).
    \end{equation*}
    Thus, since the weight on $V(-w)$ is greater than the weight on $V(w)$, and $V(w)$ is strictly increasing about 0, $\tilde{\ell}_\cost(w \mid \obs)$ must be strictly decreasing about 0.
\end{enumerate}
Together, these properties imply that 
\begin{equation*}
    \tilde{\ell}_\cost(f^*(\obs) \mid \obs) - \tilde{\ell}_\cost(0 \mid \obs) > 0
\end{equation*}
since we assumed that $f^*(\obs) < 0$. Thus we have that
\begin{equation}
    \label{eq:cond_loss_diff_bound}
    \EE\left[\tilde{\ell}_\cost(f^*(\Obs), \Lbl) - \tilde{\ell}_\cost(\tilde{f}(\Obs), \Lbl) \mid \Obs = \obs\right] > 0
\end{equation}
for any $\obs$ where $f^*(\obs) < 0$ and $\post(\obs) > \cost$.

Now, we analyze the difference in expect loss for $f^*$ and $\tilde{f}$. Since these agree on all points except when $f^*(\obs) < 0$ and $\post(\obs) > \cost$, we have that
\begin{align}
    & \EE[\tilde{\loss}(f^*(\Obs), \Lbl)] - \EE[\tilde{\loss}(\tilde{f}(\Obs), \Lbl)] \nonumber \\
    & \quad = \EE\Big[\tilde{\loss}(f^*(\Obs), \Lbl) - \tilde{\loss}(\tilde{f}(\Obs), \Lbl) \;\Big|\; f^*(\Obs) < 0 \wedge \post(\Obs) > \cost\Big] \;
    \PP\Big(f^*(\Obs) < 0 \wedge \post(\Obs) > \cost\Big) \nonumber \\
    & \quad \overset{(i)}> 0.
    \label{eq:nonoptimal_implies_nonoptimal}
\end{align}
Here, (i) is due to the combination of (\ref{eq:cond_loss_diff_bound}), which implies the first term is positive, and the first case of (\ref{eq:nonoptimal_implication}), which implies the second term is positive.

(\ref{eq:nonoptimal_implies_nonoptimal}) implies that $\tilde{f}$ has lower expected surrogate loss than $f^*$. However, we assumed that $f^*$ minimized the expected surrogate loss; thus we have a contradiction.
\end{proof}

Lemma \ref{lemma:surrogate_equivalent} means that all the results for an optimal decision maker (e.g., Theorem \ref{thm:opt_dec}) apply immediately to a decision maker minimizing a reasonable surrogate loss. In the case of decision problems without uncertainty, the decision rule will encounter zero loss and thus must be optimal, so Lemma \ref{lemma:surrogate_equivalent} also applies in this case for an optimal or suboptimal decision maker (e.g., Corollary \ref{corollary:certain_lower}). In the case of a suboptimal decision maker facing uncertainty, different loss functions may lead to different decision rules, so we cannot extend the results in that case to surrogate losses. Table \ref{tab:surrogate_summary} summarizes which results hold equivalently for decision makers minimizing an expected surrogate loss.

\begin{table}[]
    \centering
    \begin{tabular}{l|rr}
        \toprule
        \bf Setting & \bf True loss & \bf Surrogate loss \\
        \midrule
        IDT for optimal decision maker (Theorem \ref{thm:opt_dec}) & \cmark & \cmark \\
        IDT for suboptimal decision maker (Theorems \ref{thm:subopt_known} and \ref{thm:subopt_unknown}) & \cmark & \xmark \\
        No identifiability for decisions without uncertainty (Corollary \ref{corollary:certain_lower}) & \cmark & \cmark \\
        \bottomrule
    \end{tabular}
    \vspace{6pt}
    \caption{An overview of which of our results apply in the setting when the decision maker is minimizing a surrogate loss rather than the true loss.}
    \label{tab:surrogate_summary}
\end{table}

\section{Further Comparison to Prior Work}
In this section, we compare two prior papers on preference learning to our results. \citet{mindermann_active_2019} and \citet{biyik_asking_2019} both propose methods for active preference learning, i.e. querying a person to learn their preferences. In each method, queries are prioritized which minimize the uncertainty of the person. The authors argue that such queries are easier to answer and thus lead to more effective preference learning. At first, these results may seem to contradict our findings that uncertain decisions make preference learning easier. However, we argue that their results are not in conflict with ours. Decisions with more uncertainty are probably more difficult for people to make, and those close to the decision boundary are probably the most difficult. However, our results show that it is \emph{necessary} to observe such decisions in order to recover the person's preferences. If we cannot observe decisions made arbitrarily close to the person's decision boundary, we cannot exactly characterize the loss function they are optimizing. Thus, combining the results of \citet{mindermann_active_2019} and \citet{biyik_asking_2019} with ours suggests that there is a tradeoff between the ease of the decision problem for the human and the identifiability of their preferences. That is, uncertainty may make the human’s decision problem more difficult but our problem of identifying preferences easier.

\end{document}